\documentclass[final]{cvpr}

\usepackage{times}
\usepackage{epsfig}
\usepackage{graphicx}
\usepackage{amsmath}
\usepackage{amssymb}
\usepackage[pagebackref=true,breaklinks=true,colorlinks,bookmarks=false]{hyperref}
\hypersetup{linkcolor=[rgb]{0.8,0.1,0.1}}
\hypersetup{citecolor=[rgb]{0.4,0.15,0.95}}
\usepackage{url}            
\usepackage{booktabs}       
\usepackage{amsfonts}       
\usepackage{nicefrac}       
\usepackage{microtype}      

\usepackage{subfigure}
\usepackage[subfigure]{tocloft}
\usepackage{caption}
\usepackage{verbatim}
\usepackage{caption}
\usepackage{amsthm}
\usepackage{bm}
\usepackage{enumitem}
\usepackage{makecell}
\usepackage{wrapfig}
\usepackage{multirow}
\usepackage{comment}
\usepackage{pifont}
\usepackage{algorithm}
\usepackage{setspace}
\usepackage[algo2e,algoruled,boxed,lined]{algorithm2e}

\newlength\savewidth

\newcommand{\norm}[1]{\left\lVert#1\right\rVert}

\usepackage{xcolor}
\definecolor{new}{rgb}{0.36, 0.15, 0}

\newtheorem{theorem}{Theorem}
\newtheorem{lemma}{Lemma}
\renewcommand{\captionlabelfont}{\footnotesize}

\usepackage[pagebackref=true,breaklinks=true,colorlinks,bookmarks=false]{hyperref}

\def\cvprPaperID{11399} 
\def\confYear{CVPR 2021}
\begin{document}

\title{Orthogonal Over-Parameterized Training}

\author{\fontsize{10pt}{\baselineskip}\selectfont Weiyang Liu\textsuperscript{1,2,*}\ \ \ Rongmei Lin\textsuperscript{3,*}\ \ \ Zhen Liu\textsuperscript{4}\ \ \ James M. Rehg\textsuperscript{5}\ \ \ Liam Paull\textsuperscript{4}\ \ \ Li Xiong\textsuperscript{3}\ \ \ Le Song\textsuperscript{5}\ \ \ Adrian Weller\textsuperscript{1,6}\\[1.5mm]
\fontsize{10pt}{\baselineskip}\selectfont \textsuperscript{1}University of Cambridge\ \ \ \ \textsuperscript{2}Max Planck Institute for Intelligent Systems\ \ \ \ \textsuperscript{3}Emory University\\
\fontsize{10pt}{\baselineskip}\selectfont \textsuperscript{4}Mila, Université de Montréal\ \ \ \ \textsuperscript{5}Georgia Institute of Technology\ \ \ \ \textsuperscript{6}Alan Turing Institute\ \ \ \ \textsuperscript{*}Equal Contribution\\
}

\maketitle
\thispagestyle{empty}

\begin{abstract}
   The inductive bias of a neural network is largely determined by the architecture and the training algorithm. To achieve good generalization, how to effectively train a neural network is of great importance. We propose a novel orthogonal over-parameterized training (OPT) framework that can provably minimize the hyperspherical energy which characterizes the diversity of neurons on a hypersphere. By maintaining the minimum hyperspherical energy during training, OPT can greatly improve the empirical generalization. Specifically, OPT fixes the randomly initialized weights of the neurons and learns an orthogonal transformation that applies to these neurons. We consider multiple ways to learn such an orthogonal transformation, including unrolling orthogonalization algorithms, applying orthogonal parameterization, and designing orthogonality-preserving gradient descent. For better scalability, we propose the stochastic OPT which performs orthogonal transformation stochastically for partial dimensions of neurons. Interestingly, OPT reveals that learning a proper coordinate system for neurons is crucial to generalization. We provide some insights on why OPT yields better generalization. Extensive experiments validate the superiority of OPT over the standard training.
\end{abstract}

\vspace{-2.8mm}
\section{Introduction}
\vspace{-0.1mm}
The inductive bias encoded in a neural network is generally determined by two major aspects: how the neural network is structured (\ie, network architecture) and how the neural network is optimized (\ie, training algorithm). For the same network architecture, using different training algorithms could lead to a dramatic difference in generalization performance~\cite{keskar2017improving,reddi2019convergence} even if the training loss is close to zero, implying that different training procedures lead to different inductive biases. Therefore, how to effectively train a neural network that generalize well remains an open challenge.
\par

\begin{figure}[t]
  \renewcommand{\captionlabelfont}{\footnotesize}
  \setlength{\abovecaptionskip}{4pt}
  \setlength{\belowcaptionskip}{-12pt}
  \centering
  \vspace{-2mm}
  \includegraphics[width=3.2in]{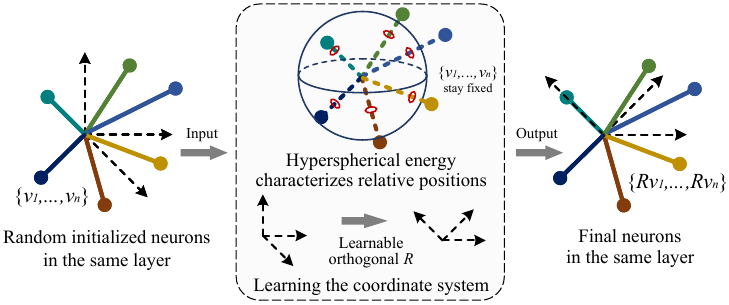}
  \caption{\footnotesize Overview of the orthogonal over-parameterized training framework. OPT learns an orthogonal transformation for each layer in the neural network, while keeping the randomly initialized neuron weights fixed.}\label{overview}
\end{figure}

Recent theories~\cite{gunasekar2017implicit,gunasekar2018implicit,kawaguchi2016deep,li2018algorithmic} suggest the importance of over-parameterization in linear neural networks. For example, \cite{gunasekar2017implicit} shows that optimizing an underdetermined quadratic objective over a matrix $\bm{M}$ with gradient descent on a factorization of $\bm{M}$ leads to an implicit regularization that may improve generalization. There is also strong empirical evidence~\cite{ding2019acnet,liu2019neural} that over-parameterzing the convolutional filters under some regularity is beneficial to generalization. Our paper aims to leverage the power of over-parameterization and explore more intrinsic structural priors in order to train a well-performing neural network.

\par

Motivated by this goal, we propose a generic orthogonal over-parameterized training~(OPT) framework for neural networks. Different from conventional neural training, OPT over-parameterizes a neuron $\thickmuskip=2mu \medmuskip=2mu\bm{w}\in\mathbb{R}^{d}$ with the multiplication of a learnable layer-shared orthogonal matrix $\thickmuskip=2mu \medmuskip=2mu\bm{R}\in\mathbb{R}^{d\times d}$ and a fixed randomly-initialized weight vector $\thickmuskip=2mu \medmuskip=2mu \bm{v}\in\mathbb{R}^{d}$, and it follows that the equivalent weight for the neuron is $\thickmuskip=2mu \medmuskip=2mu \bm{w}=\bm{R}\bm{v}$. Once each element of the neuron weight $\bm{v}$ has been randomly initialized by a zero-mean Gaussian distribution~\cite{he2015delving,glorot2010understanding}, we fix them throughout the entire training process. Then OPT learns a layer-shared orthogonal transformation $\bm{R}$ that is applied to all the neurons (in the same layer). An illustration of OPT is given in Fig.~\ref{overview}. In contrast to standard neural training, OPT decomposes the neuron into an orthogonal transformation $\bm{R}$ that learns a proper coordinate system, and a weight vector $\bm{v}$ that controls the specific position of the neuron. Essentially, the weights $\thickmuskip=2mu \medmuskip=2mu \{\bm{v}_1,\cdots,\bm{v}_n\in\mathbb{R}^d\}$ of different neurons determine the relative positions, while the layer-shared orthogonal matrix $\bm{R}$ specifies the coordinate system. Such a decoupled parameterization enables strong modeling flexibility.

\par

Another motivation of OPT comes from an empirical observation that neural networks with lower \emph{hyperspherical energy} generalize better~\cite{liu2018learning}. Hyperspherical energy quantifies the diversity of neurons on a hypersphere, and essentially characterizes the relative positions among neurons via this form of diversity. \cite{liu2018learning} introduces hyperspherical energy as a regularization in the network but do not guarantee that the hyperspherical energy can be effectively minimized (due to the existence of data fitting loss). To address this issue, we leverage the property of hyperspherical energy that it is independent of the coordinate system in which the neurons live and only depends on their relative positions. Specifically, we prove that, if we randomly initialize the neuron weight $\bm{v}$ with certain distributions, these neurons are guaranteed to attain minimum hyperspherical energy in expectation. It follows that OPT maintains the minimum energy during training by learning a coordinate system (\ie, layer-shared orthogonal matrix) for the neurons. Therefore, OPT is able to provably minimize the hyperspherical energy.
\par

We consider several ways to learn the orthogonal transformation. First, we unroll different orthogonalization algorithms such as Gram-Schmidt process, Householder reflection and Löwdin's symmetric orthogonalization. Different unrolled algorithms yield different implicit regularizations to construct the neuron weights. For example, symmetric orthogonalization guarantees that the new orthogonal basis has the least distance in the Hilbert space from the original non-orthogonal basis. Second, we consider to use a special parameterization (\eg, Cayley parameterization) to construct the orthogonal matrix, which is more efficient in training. Third, we consider an orthogonality-preserving gradient descent to ensure that the matrix $\bm{R}$ stays orthogonal after each gradient update. Last, we relax the original optimization problem by making the orthogonality constraint a regularization for the matrix $\bm{R}$. Different ways of learning the orthogonal transformation may encode different inductive biases. We note that OPT aims to utilize orthogonalization as a tool to learn neurons that maintain small hyperspherical energy, rather than to study a specific orthogonalization method. Furthermore, we propose a refinement strategy to reduce the hyperspherical energy for the randomly initialized neuron weights $\thickmuskip=2mu \medmuskip=2mu \{\bm{v}_1,\cdots,\bm{v}_n\}$. In specific, we directly minimize the hyperspherical energy of these random weights as a preprocessing step before training them on actual data.

To improve scalability, we further propose the stochastic OPT that randomly samples neuron dimensions to perform orthogonal transformation. The random sampling process is repeated many times such that each dimension of the neuron is sufficiently learned. Finally, we provide some theoretical insights and discussions to justify the effectiveness of OPT. The advantages of OPT are summarized as follows:
\vspace{-0.5mm}
\begin{itemize}[leftmargin=*,nosep,nolistsep]
    \item OPT is a generic neural network training framework with strong flexibility. There are many different ways to learn the orthogonal transformations and each one imposes a unique inductive bias. Our paper compares how different orthogonalizations may affect generalization in OPT.
    \item OPT is the first training framework where the hyperspherical energy is provably minimized (in contrast to \cite{liu2018learning}), leading to better empirical generalization. OPT reveals that learning a proper coordinate system is crucial to generalization, and the hyperspherical energy is sufficiently expressive to characterize relative neuron positions.
    \item There is no extra computational cost for the OPT-trained neural network in inference. In the testing stage, it has the same inference speed and model size as the normally trained network. Our experiments also show that OPT performs well on a diverse class of neural networks and therefore is agnostic to different neural architectures.
    \item Stochastic OPT can greatly improve the scalability of OPT while enjoying the same guarantee to minimize hyperspherical energy and having comparable performance.
\end{itemize}

\vspace{-1mm}
\section{Related Work}
\vspace{-1mm}

\textbf{Orthogonality in Neural Networks}. Orthogonality is widely adopted to improve neural networks. \cite{bansal2018can,liu2017deep,brock2016neural,huang2020controllable,xie2017all} use orthogonality as a regularization for neurons. \cite{huang2018orthogonal,lezcano2019cheap,arjovsky2016unitary,wisdom2016full,qi2020deep,jia2019orthogonal} use principled orthogonalization methods to guarantee the neurons are orthogonal to each other. In contrast to these works, OPT does not encourage orthogonality among neurons. Instead, OPT utilizes principled orthogonalization for learning orthogonal transformations for (not necessarily orthogonal) neurons to minimize hyperspherical energy.

\par

\textbf{Parameterization of Neurons}. There are various ways to parameterize a neuron for different applications. \cite{ding2019acnet} over-parameterizes a 2D convolution kernel by combining a 2D kernel of the same size and two additional 1D asymmetric kernels. The resulting convolution kernel has the same effective parameters during testing but more parameters during training. \cite{liu2019neural} constructs a neuron with a bilinear parameterization and regularizes the bilinear similarity matrix. \cite{yang2015deep} reparameterizes the neuron matrix with an adaptive fastfood transform to compress model parameters. \cite{jaderberg2014speeding,liu2015sparse,wang2017factorized} employ sparse and low-rank structures to construct convolution kernels for a efficient neural network. 

\par

\textbf{Hyperspherical Learning}. \cite{liu2017deep,Liu2017CVPR,wang2018cosface,deng2019arcface,wang2018additive,LinCoMHE20,Liu2021SphereUni} propose to learn representations on a hypersphere and show that the angular information, in contrast to magnitude information, preserves the most semantic meaning. \cite{liu2018learning} define the hyperspherical energy that quantifies the diversity of neurons on a hypersphere and shows that the small hyperspherical energy generally improves empirical generalization.

\vspace{-1mm}
\section{Orthogonal Over-Parameterized Training}
\vspace{-0.5mm}
\subsection{General Framework}
\vspace{-0.7mm}

OPT parameterizes the neuron as the multiplication of an orthogonal matrix $\thickmuskip=2mu \medmuskip=2mu \bm{R}\in\mathbb{R}^{d\times d}$ and a neuron weight vector $\thickmuskip=2mu \medmuskip=2mu \bm{v}\in\mathbb{R}^d$, and the equivalent neuron weight becomes $\thickmuskip=2mu \medmuskip=2mu \bm{w} =\bm{R}\bm{v}$. The output $\hat{y}$ of this neuron can be represented by $\thickmuskip=2mu \medmuskip=2mu \hat{y} = (\bm{R}\bm{v})^\top\bm{x}$ where $\thickmuskip=2mu \medmuskip=2mu \bm{x}\in\mathbb{R}^d$ is the input vector. In OPT, we typically fix the randomly initialized neuron weight $\bm{v}$ and only learn the orthogonal matrix $\bm{R}$. In contrast, the standard neuron is directly formulated as $\thickmuskip=2mu \medmuskip=2mu \hat{y}=\bm{v}^\top\bm{x}$, where the weight vector $\bm{v}$ is learned via back-propagation in training.

As an illustrative example, we consider a linear MLP with a loss function $\mathcal{L}$ (\eg, the least squares loss: $\thickmuskip=2mu \medmuskip=2mu \mathcal{L}(e_1,e_2)=(e_1-e_2)^2$). Specifically, the learning objective of the standard training is $\min_{\{\bm{v}_i,u_i,\forall i\}} \sum_{j=1}^{m}\mathcal{L}\big(y,\sum_{i=1}^{n} u_i\bm{v}_i^\top\bm{x}_j\big)$, while differently, our OPT is formulated as

\vspace{-3.4mm}
\begin{equation}\label{opt}
\footnotesize
\begin{aligned}
    \min_{\{\bm{R},u_i,\forall i\}} \sum_{j=1}^{m}\mathcal{L}\big(y,\sum_{i=1}^{n} u_i(\bm{R}\bm{v}_i)^\top\bm{x}_j\big)\ \ \ \textnormal{s.t.}\ \bm{R}^\top\bm{R}=\bm{R}\bm{R}^\top=\bm{I}
\end{aligned}
\end{equation}
\vspace{-1.8mm}

\noindent where $\thickmuskip=2mu \medmuskip=2mu \bm{v}_i\in\mathbb{R}^d$ is the $i$-th neuron in the first layer, and $\thickmuskip=2mu \medmuskip=2mu \bm{u}=\{u_1,\cdots,u_n\}\in\mathbb{R}^n$ is the output neuron in the second layer. In OPT, each element of $\bm{v}_i$ is usually sampled from a zero-mean Gaussian distribution (\eg, both Xavier~\cite{glorot2010understanding} and Kaiming~\cite{he2015delving} initializations belong to this class), and is fixed throughout the entire training process. In general, OPT learns an orthogonal matrix that is applied to all the neurons instead of learning the individual neuron weight. Note that, we usually do not apply OPT to neurons in the output layer (\eg, $\bm{u}$ in this MLP example, and the final linear classifiers in CNNs), since it makes little sense to fix a set of random linear classifiers. Therefore, the central problem is how to learn these layer-shared orthogonal matrices.

\vspace{-0.3mm}
\subsection{Hyperspherical Energy Perspective}
\vspace{-0.45mm}

One of the most important properties of OPT is its invariance to hyperspherical energy. Based on \cite{liu2018learning}, the hyperspherical energy of $n$ neurons is defined as $\thickmuskip=2mu \medmuskip=2mu \bm{E}(\hat{\bm{v}}_i|_{i=1}^n) = \sum_{i=1}^{n}\sum_{j=1,j\neq i}^{n} \norm{\hat{\bm{v}}_i-\hat{\bm{v}}_j}^{-1}$ in which $\thickmuskip=2mu \medmuskip=2mu \hat{\bm{v}}_i=\frac{\bm{v}_i}{\|\bm{v}_i\|}$ is the $i$-th neuron weight projected onto the unit hypersphere $\thickmuskip=2mu \medmuskip=2mu \mathbb{S}^{d-1}=\{\bm{v}\in\mathbb{R}^{d}|\norm{\bm{v}}=1\}$. Hyperspherical energy is used to characterize the diversity of $n$ neurons on a unit hypersphere. Assume that we have $n$ neurons in one layer, and we have learned an orthogonal matrix $\bm{R}$ for these neurons. The hyperspherical energy of these $n$ OPT-trained neurons is

\vspace{-3.9mm}
\begin{equation}
\footnotesize
\begin{aligned}
    \bm{E}(\hat{\bm{R}}\hat{\bm{v}}_i|_{i=1}^n) &= \sum_{i=1}^{n}\sum_{j=1,j\neq i}^{n} \norm{\bm{R}\hat{\bm{v}}_i-\bm{R}\hat{\bm{v}}_j}^{-1}\\[-1mm]
    \big(\textnormal{since} \norm{\bm{R}}^{-1}=1\big)\ \  &= \sum_{i=1}^{n}\sum_{j=1,j\neq i}^{n} \norm{\hat{\bm{v}}_i-\hat{\bm{v}}_j}^{-1} =  \bm{E}(\hat{\bm{v}}_i|_{i=1}^n)
\end{aligned}
\end{equation}
\vspace{-1.9mm}

\noindent which verifies that the hyperspherical energy does not change in OPT. Moreover, \cite{liu2018learning} proves that minimum hyperspherical energy corresponds to the uniform distribution over the hypersphere. As a result, if the initialization of the neurons in the same layer follows the uniform distribution over the hypersphere, then we can guarantee that the hyperspherical energy is minimal in a probabilistic sense.

\vspace{-0.3mm}
\begin{theorem}\label{unifrom}
For the neuron $\thickmuskip=0mu \medmuskip=0mu \bm{h}=\{h_1,\cdots,h_d\}$ where $h_i,\forall i$ are initialized i.i.d. following a zero-mean Gaussian distribution (i.e., $\thickmuskip=2mu \medmuskip=2mu h_i \sim N(0,\sigma^2)$), the projections onto a unit hypersphere $\thickmuskip=0mu \medmuskip=0mu \hat{\bm{h}}=\bm{h}/\|\bm{h}\|$ where $\thickmuskip=0mu \medmuskip=0mu \|\bm{h}\|=(\sum_{i=1}^d h_i^2)^{1/2}$ are uniformly distributed on the unit hypersphere $\mathbb{S}^{d-1}$. The neurons with minimum hyperspherical energy attained asymptotically approach the uniform distribution on $\mathbb{S}^{d-1}$.
\end{theorem}
\vspace{-1.3mm}

Theorem~\ref{unifrom} proves that, as long as we initialize the neurons in the same layer with zero-mean Gaussian distribution, the resulting hyperspherical energy is guaranteed to be small (\ie, the expected energy is minimal). It is because the neurons are uniformly distributed on the unit hypersphere and hyperspherical energy quantifies the uniformity on the hypersphere in some sense. More importantly, prevailing neuron initializations such as \cite{glorot2010understanding} and \cite{he2015delving} are zero-mean Gaussian distribution. Therefore, our neurons naturally have low hyperspherical energy from the beginning. Appendix~\ref{rand_property} gives geometric properties of the random initialized neurons.

\vspace{-0.6mm}
\subsection{Unrolling Orthogonalization Algorithms}
\vspace{-0.7mm}
\setlength{\columnsep}{5pt}
\begin{wrapfigure}{r}{0.293\textwidth}
  \begin{center}
  \advance\leftskip+1mm
  \renewcommand{\captionlabelfont}{\footnotesize}
    \vspace{-0.3in}  
    \includegraphics[width=0.292\textwidth]{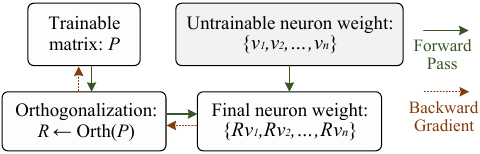}
    \vspace{-0.26in} 
    \caption{\footnotesize Unrolled orthogonalization.}\label{ortho}
    \vspace{-0.2in} 
  \end{center}
\end{wrapfigure}
In order to learn the orthogonal transformation, we unroll classic orthogonalization algorithms and embed them into the neural network such that the training can be performed in an end-to-end fashion. We need to make every step of the orthogonalization algorithm differentiable, as shown in Fig.~\ref{ortho}.

\vspace{0.1mm}

\textbf{Gram-Schmidt Process}. This method takes a linearly independent set and eventually produces an orthogonal set based on it. The Gram-Schmidt Process (GS) usually takes the following steps to orthogonalize a set of vectors $\thickmuskip=2mu \medmuskip=2mu \{\bm{u}_1,\cdots,\bm{u}_n\}\in\mathbb{R}^{n\times n}$ and obtain an orthonormal set $\thickmuskip=2mu \medmuskip=2mu \{\bm{e}_1,\cdots,\bm{e}_i,\cdots,\bm{e}_n\}\in\mathbb{R}^{n\times n}$. First, when $\thickmuskip=2mu \medmuskip=2mu i=1$, we have $\thickmuskip=2mu \medmuskip=2mu\bm{e}_1 = \frac{\tilde{\bm{e}}_1}{\|{\tilde{\bm{e}}_1}\|}$ where $\thickmuskip=2mu \medmuskip=2mu\tilde{\bm{e}}_1 = \bm{u}_1$. Then, when $\thickmuskip=2mu \medmuskip=2mu n\geq i\geq 2$, we have $\thickmuskip=2mu \medmuskip=2mu\bm{e}_i = \frac{\tilde{\bm{e}}_i}{\|{\tilde{\bm{e}}_i}\|}$ where $\thickmuskip=1mu \medmuskip=1mu \tilde{\bm{e}}_i = \bm{u}_i-\sum_{j=1}^{i-1}\textnormal{Proj}_{\bm{e}_{j}}(\bm{u}_i)$. Note that, $\thickmuskip=2mu \medmuskip=2mu \textnormal{Proj}_{\bm{b}}(\bm{a})=\frac{\langle \bm{a},\bm{b}\rangle}{\langle\bm{b},\bm{b}\rangle}\bm{b}$ is defined as the projection operator.

\vspace{0.4mm}

\textbf{Householder Reflection}. A Householder reflector is defined as $\thickmuskip=2mu \medmuskip=2mu \bm{H}=\bm{I}-2\frac{\bm{u}\bm{u}^\top}{\|\bm{u}\|^2}$ where $\bm{u}$ is perpendicular to the reflection hyperplane. In QR factorization, Householder reflection (HR) is used to transform a (non-singular) square matrix into an orthogonal matrix and an upper triangular matrix. Given a matrix $\thickmuskip=2mu \medmuskip=2mu \bm{U}= \{\bm{u}_1,\cdots,\bm{u}_n\}\in\mathbb{R}^{n\times n}$, we consider the first column vector $\bm{u}_1$. We use Householder reflector to transform $\bm{u}_1$ to $\thickmuskip=2mu \medmuskip=2mu \bm{e}_1=\{1,0,\cdots,0\}$. Specifically, we construct an orthogonal matrix $\bm{H}_1$ with $\thickmuskip=2mu \medmuskip=2mu \bm{H}_1=\bm{I}-2\frac{(\bm{u}_1-\norm{\bm{u}_1}\bm{e}_1)(\bm{u}_1-\norm{\bm{u}_1}\bm{e}_1)^\top}{\norm{\bm{u}_1-\norm{\bm{u}_1}\bm{e}_1}^2}$.
The first column of $\bm{H}_1\bm{U}$
becomes $\thickmuskip=2mu \medmuskip=2mu \{\|\bm{\bm{u}_1}\|,0,\cdots,0\}$. At the $k$-th step, we can view the sub-matrix $\bm{U}_{(k:n,k:n)}$ as a new $\bm{U}$, and use the same procedure to construct the Householder transformation $\thickmuskip=2mu \medmuskip=2mu \tilde{\bm{H}}_k\in\mathbb{R}^{(n-k)\times (n-k)}$. We construct the final Householder transformation as $\thickmuskip=2mu \medmuskip=2mu \bm{H}_k=\textnormal{Diag}(\bm{I}_{k},\tilde{\bm{H}}_k)$. Now we can gradually transform $\bm{U}$ to an upper triangular matrix with $n$ Householder reflections. Therefore, we have that $\thickmuskip=2mu \medmuskip=2mu \bm{H}_n\cdots\bm{H}_2\bm{H}_1\bm{U}=\bm{R}^{\textnormal{up}}$ where $\bm{R}^{\textnormal{up}}$ is an upper triangular matrix and the obtained orthogonal set is $\thickmuskip=2mu \medmuskip=2mu \bm{Q}^\top=\bm{H}_n\cdots\bm{H}_2\bm{H}_1$.

\textbf{Löwdin's Symmetric Orthogonalization}. Let the matrix $\thickmuskip=2mu \medmuskip=2mu \bm{U}= \{\bm{u}_1,\cdots,\bm{u}_n\}\in\mathbb{R}^{n\times n}$ be a given set of linearly independent vectors in an $n$-dimensional space. A non-singular linear transformation $\bm{A}$ can transform the basis $\bm{U}$ to an orthogonal basis $\bm{R}$: $\thickmuskip=2mu \medmuskip=2mu \bm{R}=\bm{U}\bm{A}$. The matrix $\bm{R}$ will be orthogonal if 
$\thickmuskip=2mu \medmuskip=2mu \bm{R}^\top\bm{R} = (\bm{U}\bm{A})^\top \bm{U}\bm{A}=\bm{A}^\top\bm{M}\bm{A}=\bm{I}$
where $\thickmuskip=2mu \medmuskip=2mu \bm{M}=\bm{U}^\top\bm{U}$ is the Gram matrix of the given set $\bm{U}$. We obtain a general solution to the orthogonalization problem via the substitution: $\thickmuskip=2mu \medmuskip=2mu \bm{A}=\bm{M}^{-\frac{1}{2}}\bm{B}$ where $\bm{B}$ is an arbitrary unitary matrix. The specific choice $\thickmuskip=2mu \medmuskip=2mu \bm{B}=\bm{I}$ gives the Löwdin's symmetric orthogonalization~(LS): $\thickmuskip=2mu \medmuskip=2mu \bm{R}=\bm{U}\bm{M}^{-\frac{1}{2}}$. We can analytically obtain the symmetric orthogonalization from the singular value decomposition: $\thickmuskip=2mu \medmuskip=2mu \bm{U}=\bm{W}\bm{\Sigma}\bm{V}^\top$. Then LS gives $\thickmuskip=2mu \medmuskip=2mu\bm{R}=\bm{W}\bm{V}^\top$ as the orthogonal set for $\bm{U}$. LS has a 
unique property which the other orthogonalizations do not have. The orthogonal set resembles the original set in a nearest-neighbour sense. More specifically, LS guarantees that $\sum_i\|\bm{R}_i-\bm{U}_i\|^2$ (where $\bm{R}_i$ and $\bm{U}_i$ are the $i$-th column of $\bm{R}$ and $\bm{U}$, respectively) is minimized. Intuitively, LS indicates the gentlest pushing of the directions of the vectors in order to get them orthogonal to each other.

\textbf{Discussion}. These orthogonalization algorithms are fully differentiable and end-to-end trainable. For accurate orthogonality, these algorithms can be used repeatedly and unrolled with multiple steps. Empirically, one-step unrolling already works well. Givens rotations can also  construct the orthogonal matrix, but it requires traversing all lower triangular elements in the original set $\bm{U}$, which takes $\mathcal{O}(n^2)$ complexity and is too costly. Interestingly, each orthogonalization encodes a unique inductive bias to the neurons by imposing implicit regularizations (\eg, least distance in Hilbert space for LS). Details about these orthogonalizations are in Appendix~\ref{supp_ortho}. Unrolling orthogonalization has been considered in different scenarios \cite{huang2018orthogonal,tang2019orthogonal,mhammedi2017efficient}. More orthogonalization methods~\cite{lezcano2019trivializations} can be applied in OPT, but exhaustively applying them to OPT is out of the scope of this paper.

\vspace{-0.7mm}
\subsection{Orthogonal Parameterization}
\vspace{-1.1mm}

A convenient way to ensure orthogonality while learning the matrix $\bm{R}$ is to use a special parameterization that inherently guarantees orthogonality. The exponential parameterization use $\thickmuskip=2mu \medmuskip=2mu \bm{R}=\exp(\bm{W})$ (where $\exp(\cdot)$ denotes the matrix exponential) to represent an orthogonal matrix from a skew-symmetric matrix $\bm{W}$. The Cayley parameterization~(CP) is a Padé approximation of the exponential parameterization, and is a more natural choice due to its simplicity. CP uses the following transform to construct an orthogonal matrix $\bm{R}$ from a skew-symmetric matrix $\bm{W}$: $\bm{R} = (\bm{I}+\bm{W})(\bm{I}-\bm{W})^{-1}$ where $\thickmuskip=2mu \medmuskip=2mu\bm{W}=-\bm{W}^\top$. We note that CP only produces the orthogonal matrices with determinant $1$, which belong to the special orthogonal group and thus $\thickmuskip=2mu \medmuskip=2mu \bm{R}\in SO(n)$. Specifically, it suffices to learn the upper or lower triangular of the matrix $\bm{W}$ with unconstrained optimization to obtain a desired orthogonal matrix $\bm{R}$. Cayley parameterization does not cover the entire orthogonal group and is less flexible in terms of representation power, which serves as an explicit regularization for the neurons.

\vspace{-0.8mm}
\subsection{Orthogonality-Preserving Gradient Descent}
\vspace{-0.8mm}

An alternative way to guarantee orthogonality is to modify the gradient update for the matrix $\bm{R}$. The idea is to initialize $\bm{R}$ with an arbitrary orthogonal matrix and then ensure each gradient update is to apply an orthogonal transformation to $\bm{R}$. It is essentially conducting gradient descent on the Stiefel manifold~\cite{Li2020Efficient,wen2013feasible,wisdom2016full,lezcano2019cheap,arjovsky2016unitary,henaff2016recurrent,jing2017tunable}. Given a matrix $\thickmuskip=2mu \medmuskip=2mu \bm{U}_{(0)}\in\mathbb{R}^{n\times n}$ that is initialized as an orthogonal matrix, we aim to construct an orthogonal transformation as the gradient update. We use the Cayley transform to compute a parametric curve on the Stiefel manifold $\thickmuskip=2mu \medmuskip=2mu\mathcal{M}_s=\{ \bm{U}\in\mathbb{R}^{n\times n}: \bm{U}^\top\bm{U}=\bm{I} \}$ with a specific metric via a skew-symmetric matrix $\bm{W}$ and use it as the update rule:

\vspace{-4.25mm}
\begin{equation}\label{ogd}
\footnotesize
\bm{Y}(\lambda)=(\bm{I}-\frac{\lambda}{2}\bm{W})^{-1}(\bm{I}+\frac{\lambda}{2}\bm{W})\bm{U}_{(i)},\ \bm{U}_{(i+1)}=\bm{Y}(\lambda)
\end{equation}
\vspace{-3.45mm}

\noindent where $\thickmuskip=2mu \medmuskip=2mu \hat{\bm{W}}=\nabla f(\bm{U}_{(i)})\bm{U}_{(i)}^\top-\frac{1}{2}\bm{U}_{(i)}(\bm{U}_{(i)}^\top\nabla f(\bm{U}_{(i)}\bm{U}_{(i)}^\top)$ and $\thickmuskip=2mu \medmuskip=2mu\bm{W}=\hat{\bm{W}}-\hat{\bm{W}}^\top$. $\bm{U}_{(i)}$ denotes the orthogonal matrix in the $i$-th iteration. $\nabla f(\bm{U}_{(i)})$ denotes the original gradient of the loss function \emph{w.r.t.} $\bm{U}_{(i)}$. We term this gradient update as orthogonal-preserving gradient descent (OGD). To reduce the computational cost of the matrix inverse in Eq.~\ref{ogd}, we use an iterative method~\cite{Li2020Efficient} to approximate the Cayley transform without matrix inverse. We arrive at the fixed-point iteration:

\vspace{-3mm}
\begin{equation}
\footnotesize
    \bm{Y}(\lambda)=\bm{U}_{(i)}+\frac{\lambda}{2}\bm{W}\big(\bm{U}_{(i)}+\bm{Y}(\lambda)\big)
\end{equation}
\vspace{-3mm}

\noindent which converges to the closed-form Cayley transform with a rate of $o(\lambda^{2+n})$ ($n$ is the iteration number). In practice, two iterations suffice for a reasonable approximation accuracy.

\vspace{-0.8mm}
\subsection{Relaxation to Orthogonal Regularization}
\vspace{-0.8mm}
Alternatively, we also consider relaxing the original optimization with an orthogonality constraint to an unconstrained optimization with orthogonality regularization~(OR). Specifically, we remove the orthogonality constraint, and adopt an orthogonality regularization for $\bm{R}$, \ie,  $\thickmuskip=9mu \medmuskip=2mu \|\bm{R}^\top\bm{R}-\bm{I}\|_F^2$. However, OR cannot guarantee the energy stays unchanged. Taking Eq.~\ref{opt} as an example, the objective becomes

\vspace{-4.2mm}
\begin{equation}
\footnotesize
    \min_{\bm{R},u_i,\forall i} \sum_{j=1}^{m}\mathcal{L}\big(y,\sum_{i=1}^{n} u_i(\bm{R}\bm{v}_i)^\top\bm{x}_j\big)+\beta\|\bm{R}^\top\bm{R}-\bm{I}\|_F^2
\end{equation}
\vspace{-2.5mm}

\noindent where $\beta$ is a hyperparameter. This serves as an relaxation of the original OPT objective. Note that, OR is imposed to $\bm{R}$ instead of neurons and is quite different from the existing orthogonality regularization on neurons~\cite{liu2017deep,bansal2018can,huang2018orthogonal,xie2017all,brock2016neural}.

\vspace{-0.7mm}
\subsection{Refining the Initialization as Preprocessing}\label{refine_neuron}
\vspace{-0.7mm}

\textbf{Minimizing the energy beforehand}. Because we randomly initialize the neurons $\thickmuskip=2mu \medmuskip=2mu \{\bm{v}_1,\cdots,\bm{v}_n\}$, there exists a variance that makes the hyperspherical energy deviate from the minima even if the hyperspherical energy is minimal in a probabilistic sense. To further reduce the hyperspherical energy, we propose to refine the random initialization by minimizing its hyperspherical energy as a preprocessing step before the OPT training. Specifically, before feeding these neurons to OPT, we first minimize the hyperspherical energy of the initialized neurons with gradient descent (without fitting the training data). Moreover, since the randomly initialized neurons cannot guarantee to get rid of the collinearity redundancy as shown in \cite{liu2018learning} (\ie, two neurons are on the same line but have opposite directions), we can perform the half-space hyperspherical energy minimization~\cite{liu2018learning}.

\textbf{Normalizing the neurons}. The norm of the randomly initialized neurons may have some influence on OPT, serving a role similar to weighting the importance of different neurons. Moreover, the norm makes the hyperspherical energy less expressive to characterize the diversity of neurons, as discussed in Section~\ref{main_discussion}. To make the coordinate frame (i.e. the rotation matrix $R$) truly independent of the relative positions of the neurons, we propose to normalize the neuron weights such that each neuron has unit norm. Because the weights of the neurons $\thickmuskip=2mu \medmuskip=2mu \{\bm{v}_1,\cdots,\bm{v}_n\}$ are fixed during training and orthogonal matrices will not change the norm of the neurons, we only need to normalize the randomly initialized neuron weights as a preprocessing before the OPT training.

We have comprehensively evaluated both refinement strategies in Section~\ref{CNN_exp} and verified their effectiveness. Note that the effectiveness of OPT is not dependent on these refinements. Our experiments do not use these refinements by default and the results show that OPT still performs well.

\vspace{-1.75mm}
\section{Towards Better Scalablity for OPT}
\vspace{-1.65mm}

\setlength{\columnsep}{6pt}
\begin{wrapfigure}{r}{0.137\textwidth}
  \begin{center}
  \advance\leftskip+1mm
  \renewcommand{\captionlabelfont}{\footnotesize}
    \vspace{-0.3in}  
    \includegraphics[width=0.135\textwidth]{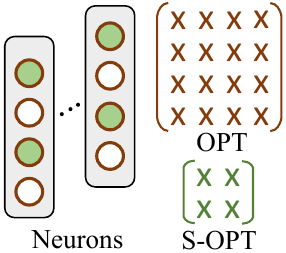}
    \vspace{-0.26in} 
    \caption{\footnotesize S-OPT.}\label{sopt}
    \vspace{-0.3in} 
  \end{center}
\end{wrapfigure}

If the dimension of neurons becomes extremely large, then the orthogonal matrix to transform the neurons will also be large. Therefore, it may take large GPU memory and time to train the neural networks with the original OPT. To address this, we propose a scalable variant -- stochastic OPT (S-OPT). The key idea of S-OPT is to randomly select some dimensions from the neurons in the same layer and construct a small orthogonal matrix to transform these dimensions together. The selection of dimensions is stochastic in each outer iteration, so a small orthogonal matrix is sufficient to cover all the neuron dimensions. S-OPT aims to approximate a large orthogonal transformation for all the neuron dimensions with many small orthogonal transformations for random subsets of these dimensions, which shares similar spirits with Givens rotation. The approximation will be more accurate when the procedure is randomized over many times. Fig.~\ref{sopt} compares the size of the orthogonal matrix in OPT and S-OPT. The orthogonal matrix in OPT is of size $d\times d$, while the orthogonal matrix in S-OPT is of size $p\times p$ where $p$ is usually much smaller than $d$. Most importantly, S-OPT can still preserve the low hyperspherical energy of neurons because of the following result.

\vspace{-0.45mm}
\begin{theorem}\label{sopt_thm}
For $n$ $d$-dimensional neurons, selecting any $p$ ($\thickmuskip=2mu \medmuskip=2mu p\leq d$) dimensions and applying an shared orthogonal transformation ($\thickmuskip=2mu \medmuskip=2mu p\times p$ orthogonal matrix) to these $p$ dimensions of all neurons will not change the hyperspherical energy.
\end{theorem}
\vspace{-0.25mm}

\setlength{\columnsep}{9pt}
\begin{wrapfigure}{r}{0.25\textwidth}
\begin{center}
\begin{minipage}{1\linewidth}
\vspace{-0.475in}
\captionsetup[algorithm]{font=footnotesize}
\begin{algorithm}[H]
\footnotesize
\caption{\footnotesize Stochastic OPT}\label{sopt_alg}
\For{$i=1,2,\cdots,N_{\textnormal{out}}$}{
    \For{$j=1,2,\cdots,N_{\textnormal{in}}$}{
    \underline{\textbf{1}}. Randomly select $p$ dimensions from $d$-dimensional neurons in the same layer.\;
    
    \underline{\textbf{2}}. Construct an orthogonal matrix $\bm{R}_p\in\mathbb{R}^{p\times p}$ and initialize it as identity matrix.\;
    
    \underline{\textbf{3}}. Update $\bm{R}_p$ by applying OPT with one iteration.\;
    \vspace{-0.3mm}
    
    }
    
    \underline{\textbf{4}}. Multiply $\bm{R}_p$ back to the $p$-dim sub-vectors from the $d$-dim neurons to transform these neurons.\;
    \vspace{-0.25mm}
    
}
\vspace{-0.8mm}
\end{algorithm}
\vspace{-0.46in}
\end{minipage}
\end{center}
\end{wrapfigure}

A description of S-OPT is given in Algorithm~\ref{sopt_alg}. S-OPT has outer and inner iterations. In each inner iteration, the training is almost the same as OPT, except that the orthogonal matrix transforms a subset of the dimensions and the learnable orthogonal matrix has to be re-initialized to an identity matrix. The selection of neuron dimension is randomized in every outer iteration such that all neuron dimensions can be sufficiently covered as the number of outer iterations increases. Therefore, given sufficient number of iterations, S-OPT will perform comparably to OPT, as empirically verified in Section~\ref{sopt_sect}. As a parallel direction to improve the scalability, we further propose a parameter-efficient OPT in Appendix~\ref{appendix_peopt}. This OPT variant explores structure priors in $\bm{R}$ to improve parameter efficiency.

\vspace{-1.5mm}
\section{Intriguing Insights and Discussions}
\vspace{-1.1mm}
\subsection{Local Landscape}\label{main_landscape}
\vspace{-0.8mm}

\setlength{\columnsep}{8pt}
\begin{wrapfigure}{r}{0.251\textwidth}
\begin{center}
\advance\leftskip+1mm
    \renewcommand{\captionlabelfont}{\footnotesize}
    \vspace{-0.54in}  
    \includegraphics[width=0.25\textwidth]{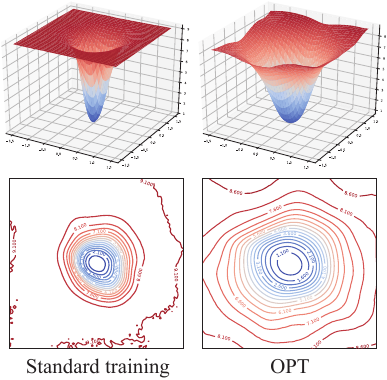}
    \vspace{-0.275in} 
    \caption{\footnotesize Training loss landscapes. }\label{loss_landscape}
\vspace{-0.22in} 
\end{center}
\end{wrapfigure}

We follow \cite{li2018visualizing} to visualize the loss landscapes of both standard training and OPT in Fig.~\ref{loss_landscape}. For standard training, we perturb the parameter space of all the neurons (\ie, filters). For OPT, we perturb the parameter space of all the trainable matrices (\ie, $\bm{P}$ in Fig.~\ref{ortho}), because OPT does not directly learn neuron weights. The general idea is to use two random vectors (\eg, normal distribution) to perturb the parameter space and obtain the loss value with the perturbed network parameters. Details and full results about the visualization are given in Appendix~\ref{loss_vis_normal}. The loss landscape of standard training has extremely sharp minima. The red region is very flat, leading to small gradients. In contrast, the loss landscape of OPT is much more smooth and convex with flatter minima, well matching the finding that flat minimizers generalize well \cite{hochreiter1997flat,chaudhari2017entropy,izmailov2018averaging}. Additional loss landscape visualization results in Appendix~\ref{loss_vis} (with uniform perturbation distributions) also support the same argument.

\setlength{\columnsep}{8pt}
\begin{wrapfigure}{r}{0.251\textwidth}
\begin{center}
\advance\leftskip+1mm
    \renewcommand{\captionlabelfont}{\footnotesize}
    \vspace{-0.19in}  
    \includegraphics[width=0.25\textwidth]{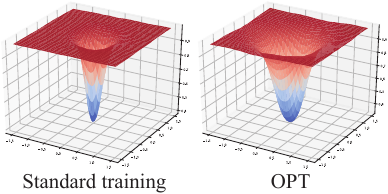}
    \vspace{-0.27in} 
    \caption{\footnotesize Testing error landscapes. }\label{acc_landscape}
\vspace{-0.22in} 
\end{center}
\end{wrapfigure}

We also show the landscape of testing error on CIFAR-100 in Fig.~\ref{acc_landscape}. Full results and details are in Appendix~\ref{loss_vis_normal}. Compared to standard training, the testing error of OPT increases more slowly and smoothly while the network parameters move away from the minima, which indicates that the parameter space of OPT yields better robustness to perturbations.

\vspace{-0.8mm}
\subsection{Optimization and Generalization}
\vspace{-0.8mm}

We discuss why OPT may improve optimization and generalization. On one hand, \cite{xie2016diverse} proves that once the neurons are hyperspherically diverse enough in a one-hidden-layer network, the training loss is on the order of the square norm of the gradient and the generalization error will have an additional term $\thickmuskip=2mu \medmuskip=2mu \tilde{\mathcal{O}}(1/\sqrt{m})$ where $m$ is the number of samples. This suggests that SGD-optimized networks with minimum hyperspherical energy~(MHE) attained have no spurious local minima. Since OPT is guaranteed to achieve MHE in expectation, OPT-trained networks enjoy the inductive bias induced by MHE. On the other hand, \cite{kawaguchi2016deep,allen2018convergence,du2017gradient,li2018algorithmic,gunasekar2017implicit} shows that over-parameterization in neural networks improves the first-order optimization, leads to better generalization, and imposes implicit regularizations. In the light of this, OPT also introduces over-parameterization to each neuron, which shares similar spirits with \cite{lin2013network}. Specifically, one $d$-dimensional neuron has $\thickmuskip=2mu \medmuskip=2mu d^2+d$ parameters in OPT (with $d^2$ being layer-shared), compared to $d$ parameters in a standard neuron. Although OPT uses more parameters for a neuron in training, the equivalent number of parameters for a neuron stays unchanged and it will not affect testing speed. 

\vspace{-1mm}
\subsection{Discussions}\label{main_discussion}
\vspace{-1mm}

\textbf{Over-parameterization}. We delve deeper into the over-parameterization in the context of OPT. Its definition varies in different cases. OPT is over-parameterized in terms of training in the following sense. Although OPT-trained networks have the same effective number of parameters as the standard networks in testing, the OPT neuron is decomposed into two sets of parameters in training: orthogonal matrix and neuron weights. It means that the same set of parameters in a neural network can be represented by different sets of training parameters in OPT (\ie, different combinations of orthogonal matrices and neuron weights can lead to the same neural network). OPT is still over-parameterized even if we only count the number of learnable parameters. For a layer of $n$ $d$-dimensional neurons, the number of learnable parameters in vanilla OPT is $\frac{d(d-1)}{2}$ in contrast to $nd$ in standard training. In prevailing architectures (\eg, ResNet~\cite{he2016deep}), the neuron dimension is far larger than the number of neurons.

\textbf{Coordinate system and relative position}. OPT shows that learning the coordinate system yields better generalization than learning neuron weights directly. This implies that the coordinate system is crucial to generalization. However, the relative position does not matter only when the hyperspherical energy is sufficiently low, indicating that the neurons need to be diverse enough on the unit hypersphere.

\textbf{The effects of neuron norm}. Because we will normalize the neuron norm when computing the hyperspherical energy, the effects of neuron norm will not be taken into consideration. Moreover, simply learning the orthogonal matrices will not change the neuron norm. Therefore, the neuron norm may affect the training. We use an extreme example to demonstrate the effects. Assume that one of the neurons has norm $1000$ and the other neurons have norm $0.01$. Then no matter what orthogonal matrices we have learned, the final performance will be bad. In this case, the hyperspherical energy can still be minimized to a very low value, but it can not capture the norm distribution. Fortunately, such an extreme case is unlikely to happen, because we are using zero-mean Gaussian distribution to initialize the neuron and every neuron also has the same expected value for the norm. To eliminate the effects of norms, we can normalize the neuron weights in training, as proposed in Section~\ref{refine_neuron}.

\vspace{-1.1mm}
\section{Applications and Experimental Results}\label{exp_results}
\vspace{-1.15mm}
We put all the experimental settings and many additional results in Appendix~\ref{exp_set} and Appendix~\ref{appendix_peopt},\ref{appendix_energy}, respectively.

\vspace{-0.85mm}
\subsection{Ablation Study and Exploratory Experiments}
\vspace{-0.85mm}

\setlength{\columnsep}{5pt}
\begin{wraptable}{r}[0cm]{0pt}
    \centering
    \scriptsize
    \newcommand{\tabincell}[2]{\begin{tabular}{@{}#1@{}}#2\end{tabular}}
    \setlength{\tabcolsep}{1.8pt}
\renewcommand{\captionlabelfont}{\footnotesize}
\begin{tabular}{c |c c c c} 
\specialrule{0em}{-10pt}{0pt}
  \hline
Method & FN & LR & CNN-6 & CNN-9\\\hline
Baseline & - & - & 37.59 & 33.55\\
UPT & \ding{55} & U & 48.47 & 46.72\\
UPT & \ding{51} & U & 42.61 & 39.38\\
OPT & \ding{55} & GS & 37.24 & 32.95\\
OPT & \ding{51} & GS & \textbf{33.02} & \textbf{31.03}\\
 \hline
  \specialrule{0em}{0pt}{-9pt}
\end{tabular}
\caption{\footnotesize Error (\%) on C-100.}
\label{upt}
\vspace{-2.6mm}
\end{wraptable}

\textbf{Orthogonality}. We evaluate whether orthogonality in OPT is necessary. We use 6-layer and 9-layer CNN (Appendix~\ref{exp_set}) on CIFAR-100. Then we compare OPT with unconstrained over-parameterized training~(UPT) which learns an unconstrained matrix $\bm{R}$ (with weight decay) using the same network. In Table~\ref{upt}, ``FN'' denotes whether the randomly initialized neuron weights are fixed in training. ``LR'' denotes whether the learnable matrix $\bm{R}$ is unconstrained (``U'') or orthogonal (``GS'' for Gram-Schmidt process). Table~\ref{upt} shows that without orthogonality, UPT performs much worse than OPT.

\textbf{Fixed or learnable weights}. From Table~\ref{upt}, we can see that using fixed neuron weights is consistently better than learnable neuron weights in both UPT and OPT. It indicates that fixing the neuron weights can well maintain low hyperspherical energy and is beneficial to empirical generalization.

\vspace{0.25mm}

\setlength{\columnsep}{5pt}
\begin{wraptable}{r}[0cm]{0pt}
    \centering
    \scriptsize
    \newcommand{\tabincell}[2]{\begin{tabular}{@{}#1@{}}#2\end{tabular}}
    \setlength{\tabcolsep}{1.5pt}
\renewcommand{\captionlabelfont}{\footnotesize}
\begin{tabular}{c|c c c c} 
\specialrule{0em}{-10pt}{0pt}
  \hline
Method  & Original & MHE & HS-MHE & CoMHE\\\hline
OPT (GS) & 33.02 & 32.99 & 32.78& \textbf{32.69}\\
OPT (LS) & 34.48& 34.43 & 34.37 & \textbf{34.15}\\
OPT (CP) & 33.53 & 33.50 & 33.42& \textbf{33.27}\\
Energy & 3.5109 & 3.5003 & 3.4976& \textbf{3.4954}\\
 \hline
  \specialrule{0em}{0pt}{-9.5pt}
\end{tabular}
\caption{\footnotesize Refining initialized energy.}
\label{rhe}
\vspace{-4mm}
\end{wraptable}
\textbf{Refining initialization}. We evaluate two refinement methods in Section~\ref{refine_neuron} for neuron initialization. First, we consider the hyperspherical energy minimization as a preprocessing for the neuron weights. Our experiment uses CNN-6 on CIFAR-100. Specifically, we run gradient descent for 5k iterations to minimize the objective of MHE/HS-MHE~\cite{liu2018learning} or CoMHE~\cite{LinCoMHE20} before the training starts. Table~\ref{rhe} shows the hyperspherical energy before and after the preprocessing. All methods start with the same random initialization, so all hyperspherical energies start at 3.5109. Testing errors (\%) in Table~\ref{rhe} show that the refinement well improves OPT. Although using advanced regularizations such as CoMHE as pre-processing can improve the performance significantly, we do not use them in the other experiments in order to keep our comparison fair and clean. More different ways to minimize the hyperspherical energy can also be considered~\cite{Liu2021SphereUni}.
\par

\setlength{\columnsep}{5pt}
\begin{wraptable}{r}[0cm]{0pt}
    \centering
    \scriptsize
    \newcommand{\tabincell}[2]{\begin{tabular}{@{}#1@{}}#2\end{tabular}}
    \setlength{\tabcolsep}{1.3pt}
\renewcommand{\captionlabelfont}{\footnotesize}
\begin{tabular}{c |c c} 
\specialrule{0em}{-10.5pt}{0pt}
  \hline
Method & w/o Norm & w/ Norm\\\hline
Baseline & 37.59 & \textbf{36.05}\\
OPT (GS) & 33.02 & \textbf{32.54}\\
OPT (HR) & 35.67 & \textbf{35.30}\\
OPT (LS) & 34.48 & \textbf{32.11}\\
OPT (CP) & 33.53 & \textbf{32.49}\\
OPT (OGD) & 33.37 & \textbf{32.70}\\
OPT (OR) & 34.70 & \textbf{33.27}\\
 \hline
  \specialrule{0em}{0pt}{-9.5pt}
\end{tabular}
\caption{\footnotesize Normalization (\%).}
\label{wnorm}
\vspace{-3.8mm}
\end{wraptable}

We evaluate the second refinement strategy, \ie, neuron weight normalization. Section~\ref{main_discussion} has explained why normalizing the neuron weights may be useful. After initialization, we normalize all the neuron weights to $1$. Since OPT does not change the neuron norm, the neuron will keep the norm as 1. More importantly, the hyperspherical energy will not be affected by the neuron normalization. We conduct classification with CNN-6 on CIFAR-100. Testing errors in Table~\ref{wnorm} show that normalizing the neurons greatly improves OPT, validating our previous analysis. Note that, these two refinements are not used by default in other experiments.

\setlength{\columnsep}{5pt}
\begin{wraptable}{r}[0cm]{0pt}
    \centering
    \scriptsize
    \newcommand{\tabincell}[2]{\begin{tabular}{@{}#1@{}}#2\end{tabular}}
    \setlength{\tabcolsep}{1.8pt}
\renewcommand{\captionlabelfont}{\footnotesize}
\begin{tabular}{c c c} 
\specialrule{0em}{-10.5pt}{0pt}
  \hline
 Mean & Energy & Error (\%)\\\hline
0 & \textbf{3.5109} & \textbf{32.49}\\
1e-3 & 3.5117 & 33.11\\
1e-2 & 3.5160 & 39.51\\
2e-2 & 3.5531 & 53.89\\
3e-2 & 3.6761 & N/C\\
 \hline
  \specialrule{0em}{0pt}{-9pt}
\end{tabular}
\caption{\footnotesize Initial energy.}
\label{highhe}
\vspace{-4mm}
\end{wraptable}

\textbf{High vs. low energy}. We validate that high hyperspherical energy corresponds to inferior empirical generalization. To initialize high energy neurons, we use \cite{he2015delving} and set the mean as 1e-3, 1e-2, 2e-2, and 3e-2. We experiment on CIFAR-100 with CNN-6. Table~\ref{highhe} (``N/C'' denotes not converged) show that higher energy generalizes worse and also leads to difficulty in convergence. We see that a small change in energy can lead to a dramatic generalization gap.

\setlength{\columnsep}{5pt}
\begin{wraptable}{r}[0cm]{0pt}
    \centering
    \scriptsize
    \newcommand{\tabincell}[2]{\begin{tabular}{@{}#1@{}}#2\end{tabular}}
    \setlength{\tabcolsep}{1.4pt}
\renewcommand{\captionlabelfont}{\footnotesize}
\begin{tabular}{c c} 
\specialrule{0em}{-10pt}{0pt}
  \hline
Method & Error (\%)\\\hline
Baseline & 38.95\\
HS-MHE & 36.90\\
OPT (GS) & 35.61 \\
OPT (HR) & 37.51 \\
OPT (LS) & 35.83 \\
OPT (CP) & \textbf{34.88} \\
OPT (OGD) & 35.38\\
 \hline
  \specialrule{0em}{0pt}{-9.5pt}
\end{tabular}
\caption{\footnotesize No BN.}
\label{bn}
\vspace{-4mm}
\end{wraptable}

\textbf{No BatchNorm}. We evaluate how OPT performs without BatchNorm (BN) \cite{ioffe2015batch}. We perform classification on CIFAR-100 with CNN-6. In Table~\ref{bn}, we see that all OPT variants consistently outperform both the baseline and HS-MHE~\cite{liu2018learning} by a significant margin, validating that OPT can work well without BN. CP achieves the best error with more than $4\%$ lower than standard training.

\vspace{-0.75mm}
\subsection{Empirical Evaluation on OPT}\label{CNN_exp}
\vspace{-0.9mm}

\textbf{Multi-layer perceptrons}. We evaluate OPT on MNIST with a 3-layer MLP. Appendix~\ref{exp_set} gives specific settings. Table~\ref{cnn_opt} shows the testing error with normal initialization (MLP-N) or Xavier initialization~\cite{glorot2010understanding} (MLP-X). GS/HR/LS denote different orthogonalization unrolling. CP denotes Cayley parameterization. OGD denotes orthogonal-preserving gradient descent. OR denotes relaxed orthogonal regularization. All OPT variants outperform the others by a large margin.

\begin{table}[t]
    \centering
    \scriptsize
    \newcommand{\tabincell}[2]{\begin{tabular}{@{}#1@{}}#2\end{tabular}}
    \setlength{\tabcolsep}{3pt}
\renewcommand{\captionlabelfont}{\footnotesize}
\begin{tabular}{c |c c| c c c c} 
\specialrule{0em}{2.8pt}{0pt}
  \hline
\multirow{2}{*}[0pt]{Method} &\multicolumn{2}{c|}{MNIST} &\multicolumn{4}{c}{CIFAR-100}\\
 & MLP-N & MLP-X & CNN-6 & CNN-9 & ResNet-20 & ResNet-32\\\hline
Baseline & 6.05 & 2.14 & 37.59 & 33.55 & 31.11 & 30.16\\
Orthogonal~\cite{brock2016neural} & 5.78 & 1.93 & 36.32 & 33.24 & 31.06 & 30.05\\
SRIP~\cite{bansal2018can} & - & - & 34.82 & 32.72 & 30.89 & 29.70 \\
HS-MHE~\cite{liu2018learning} & 5.57 & 1.88 &  34.97 & 32.87 & 30.98 & 29.76\\\hline
OPT (GS) & \textbf{5.11} & \textbf{1.45} & \textbf{33.02} & \textbf{31.03} & 30.49 & 29.34\\
OPT (HR) & 5.31 & 1.60 & 35.67 & 32.75 & 30.73 & 29.56\\
OPT (LS) & 5.32 & 1.54 & 34.48 & 31.22 & 30.51 & 29.42\\
OPT (CP) & 5.14 & 1.49 & 33.53 & 31.28 &\textbf{30.47} & \textbf{29.31}\\
OPT (OGD) & 5.38 & 1.56 & 33.33 & 31.47 & 30.50 & 29.39\\
OPT (OR) & 5.41 & 1.78 & 34.70 & 32.63 & 30.66 & 29.47 \\
 \hline
  \specialrule{0em}{0pt}{-8pt}
\end{tabular}
\caption{\footnotesize Testing error (\%) of OPT for MLPs and CNNs.}
\label{cnn_opt}
\vspace{-4.925mm}
\end{table}


\textbf{Convolutional networks}. We evaluate OPT with 6/9-layer plain CNNs and ResNet-20/32~\cite{he2016deep} on CIFAR-100. Detailed settings are in Appendix~\ref{exp_set}. All neurons (\ie, convolution kernels) are initialized by \cite{he2015delving}. BatchNorm is used by default. Table~\ref{cnn_opt} shows that all OPT variants outperform both baseline and HS-MHE by a large margin. HS-MHE puts the  hyperspherical energy into the loss function and naively minimizes it along with the CNN. We observe that OPT (HR) performs the worse among all OPT variants partially because of its intensive unrolling computation. OPT (GS) achieves the best testing error on CNN-6/9, while OPT (CP) achieves the best testing error on ResNet-20/34, implying that different OPT encodes different inductive bias.

\begin{figure}[!h]
\renewcommand{\captionlabelfont}{\footnotesize}
  \centering
    \vspace{-0.0105in}  
    \includegraphics[width=3.24in]{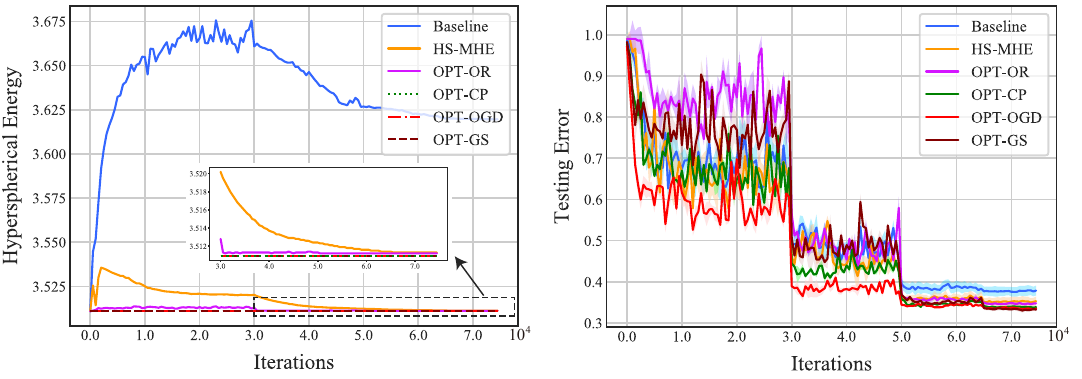}
  \vspace{-3.6mm}
  \caption{\footnotesize Training dynamics on CIFAR-100. Left: Hyperspherical energy vs. iteration. Right: Testing error vs. iteration.}\label{training}
    \vspace{-0.098in}
\end{figure}

\textbf{Training dynamics}. We look into how hyperspherical energy and testing error changes in OPT. Fig.~\ref{training} shows that the energy of the baseline will increase dramatically at the beginning and then gradually go down, but it still stays in a high value in the end. HS-MHE well reduces the energy at the end of the training. In contrast, OPT variants always maintain very small energy in training. OPT with GS, CP and OGD keep exactly the same energy as the random initialization, while OPT (OR) slightly increases the energy due to relaxation. All OPT variants converge efficiently and stably.

\setlength{\columnsep}{8pt}
\begin{wraptable}{r}[0cm]{0pt}
    \centering
    \scriptsize
    \newcommand{\tabincell}[2]{\begin{tabular}{@{}#1@{}}#2\end{tabular}}
    \setlength{\tabcolsep}{2.5pt}
\renewcommand{\captionlabelfont}{\footnotesize}
\begin{tabular}{c |c c} 
\specialrule{0em}{-9.5pt}{0pt}
  \hline
Method & Top-1& Top-5\\\hline
Baseline & 44.32 & 21.13\\
Orthogonal~\cite{brock2016neural} & 44.13 & 20.97\\
HS-MHE~\cite{liu2018learning} & 43.92 & 20.85\\
OPT (OGD) & 43.81 & 20.49\\
OPT (CP) & \textbf{43.67} & \textbf{20.26}\\
 \hline
  \specialrule{0em}{0pt}{-8pt}
\end{tabular}
\caption{\footnotesize ImageNet (\%).}
\label{imagenet}
\vspace{-4mm}
\end{wraptable}

\textbf{Large-scale learning}. To see how OPT performs in large-scale settings, we evaluate OPT on the large-scale ImageNet-2012~\cite{russakovsky2015imagenet}. Specifically, we use OPT with OGD and CP to train a plain 10-layer CNN (Appendix~\ref{exp_set}) on ImageNet. Note that, our purpose is to validate the superiority of OPT over the corresponding baseline rather than achieving state-of-the-art results. Table~\ref{imagenet} shows that OPT (CP) reduces top-1 and top-5 error for the baseline by $\thickmuskip=2mu \medmuskip=2mu \sim0.7\%$ and $\thickmuskip=2mu \medmuskip=2mu \sim0.9\%$, respectively.

\setlength{\columnsep}{5pt}
\begin{wraptable}{r}[0cm]{0pt}
    \centering
    \scriptsize
    \newcommand{\tabincell}[2]{\begin{tabular}{@{}#1@{}}#2\end{tabular}}
    \setlength{\tabcolsep}{1.5pt}
\renewcommand{\captionlabelfont}{\footnotesize}
\begin{tabular}{c |c } 
\specialrule{0em}{0.5pt}{0pt}
  \hline
Method & 5-shot Acc. (\%)\\\hline
MAML~\cite{finn2017model} & 62.71 $\pm$ 0.71\\
MatchingNet~\cite{vinyals2016matching} & 63.48 $\pm$ 0.66\\
ProtoNet~{\cite{snell2017prototypical}} & 64.24 $\pm$ 0.72\\\hline
Baseline~{\cite{chen2019closer}} &  62.53 $\pm$ 0.69\\
Baseline w/ OPT & \textbf{63.27 $\pm$ 0.68}\\
Baseline++~{\cite{chen2019closer}} & 66.43 $\pm$0.63\\
Baseline++ w/ OPT & \textbf{66.82 $\pm$ 0.62}\\
 \hline 
  \specialrule{0em}{0pt}{-9.3pt}
\end{tabular}
\caption{\footnotesize Few-shot learning.}
\label{fewshot}
\vspace{-3.7mm}
\end{wraptable}

\textbf{Few-shot recognition}. For evaluating OPT on cross-task generalization, we perform the few-shot recognition on Mini-ImageNet, following the same setup as \cite{chen2019closer}. Appendix~\ref{exp_set} gives more detailed settings. We apply OPT with CP to train the baseline and baseline++ in \cite{chen2019closer}, and immediately obtain improvements. Therefore, OPT-trained networks generalize well in this challenging scenario. 

\vspace{0.3mm}

\setlength{\columnsep}{5pt}
\begin{wraptable}{r}[0cm]{0pt}
    \centering
    \scriptsize
    \newcommand{\tabincell}[2]{\begin{tabular}{@{}#1@{}}#2\end{tabular}}
    \setlength{\tabcolsep}{1.7pt}
\renewcommand{\captionlabelfont}{\footnotesize}
\begin{tabular}{c| c c |c} 
\specialrule{0em}{-9.8pt}{0pt}
  \hline
\multirow{2}{*}[0pt]{Method}& \multicolumn{2}{c|}{GCN} & PointNet\\
 & Cora & Pubmed & MN-40 \\\hline
Baseline & 81.3 & 79.0 & 87.1\\
OPT (GS) & 81.9 & 79.4 & 87.23\\
OPT (CP) & 82.0 & 79.4 & 87.81\\
OPT (OGD) & \textbf{82.3} & \textbf{79.5} & \textbf{87.86}\\
 \hline 
  \specialrule{0em}{0pt}{-9pt}
\end{tabular}
\caption{\footnotesize Geometric networks.}
\label{geometric}
\vspace{-3mm}
\end{wraptable}

\textbf{Geometric learning}. We apply OPT to graph convolution network (GCN)~\cite{kipf2016semi} and point cloud network (PointNet)~\cite{qi2017pointnet} for graph node and point cloud classification, respectively. The training of GCN and PointNet is conceptually similar to MLP, and the detailed training procedures are given in Appendix~\ref{exp_set}. For GCN, we evaluate OPT on Cora and Pubmed datsets~\cite{sen2008collective}. For PointNet, we conduct experiments on ModelNet-40 dataset~\cite{wu20153d}. Table~\ref{geometric} shows that OPT effectively improves both GCN and PointNet.

\vspace{-0.8mm}
\subsection{Empirical Evaluation on S-OPT}\label{sopt_sect}
\vspace{-0.7mm}

\begin{table}[h]
    \centering
    \scriptsize
    \newcommand{\tabincell}[2]{\begin{tabular}{@{}#1@{}}#2\end{tabular}}
    \setlength{\tabcolsep}{3.1pt}
\renewcommand{\captionlabelfont}{\footnotesize}
\begin{tabular}{c| c c c c| c c } 
\specialrule{0em}{-8pt}{0pt}
  \hline
\multirow{2}{*}[0pt]{Method}& \multicolumn{4}{c|}{CIFAR-100} & \multicolumn{2}{c}{ImageNet}\\
& CNN-6 & Params & Wide CNN-9 & Params & ResNet-18 & Params \\[0.2mm]\hline
Baseline & 37.59 & 258K & 28.03 & 2.99M & 32.95 & 11.7M  \\
HS-MHE~\cite{liu2018learning} & 34.97 & 258K & 25.96 &  2.99M & 32.50 &11.7M \\
OPT (GS) & \textbf{33.02} & 1.36M & OOM & 16.2M & OOM & 46.5M \\
S-OPT (GS) & 33.70 & \textbf{90.9K} & \textbf{25.59} & \textbf{1.04M} & \textbf{32.26} & \textbf{3.39M} \\
 \hline 
  \specialrule{0em}{0pt}{-9pt}
\end{tabular}
\caption{\footnotesize OPT vs. S-OPT on CIFAR-100 \& ImageNet.}
\label{sopt_exp}
\vspace{-2.8mm}
\end{table}

\textbf{Convolutional networks}. S-OPT is a scalable OPT variant, and we evaluate its performance in terms of number of \emph{trainable parameters} and testing error. Training parameters are learnable variables in training, and are different from model parameters in testing. In testing, all methods have the same number of model parameters. We perform classification on CIFAR-100 with CNN-6 and wide CNN-9. We also evaluate S-OPT with standard ResNet-18 on ImageNet. Detailed settings are in Appendix~\ref{exp_set}. For S-OPT, we set the sampling dimension as 25\% of the original neuron dimension in each layer. Table~\ref{sopt_exp} shows that S-OPT achieves a good trade-off between accuracy and scalability. More importantly, S-OPT can be applied to large neural networks, making OPT more useful in practice. Additionally, Appendix~\ref{appendix_peopt} discusses an efficient parameter sharing for OPT.

\vspace{0.5mm}

\setlength{\columnsep}{5pt}
\begin{wraptable}{r}[0cm]{0pt}
    \centering
    \scriptsize
    \newcommand{\tabincell}[2]{\begin{tabular}{@{}#1@{}}#2\end{tabular}}
    \setlength{\tabcolsep}{2pt}
\renewcommand{\captionlabelfont}{\footnotesize}
\begin{tabular}{c |c c } 
\specialrule{0em}{-10.8pt}{0pt}
  \hline
$p$ $=$ & Error (\%) &  Params\\\hline
$d$ & OOM & 16.2M\\
$d$/4 & \textbf{25.59} & 1.04M\\
$d$/8 & 28.61 & 278K\\
$d$/16 & 32.52 & 88.7K\\
16 & 33.03 & 27.0K\\
3 & 45.22 & 26.0K\\
0 & 60.64 & \textbf{25.6K}\\\specialrule{0em}{-0.7pt}{0pt}
 \hline
  \specialrule{0em}{0pt}{-9.7pt}
\end{tabular}
\caption{\footnotesize Sampling dim.}
\label{sopt_dim}
\vspace{-4mm}
\end{wraptable}
\textbf{Sampling dimensions}. We study how the sampling dimension $p$ affect the performance by performing classification with wide CNN-9 on CIFAR-100. In Table~\ref{sopt_dim}, $p\thickmuskip=2mu \medmuskip=2mu=d/4$ means that we randomly sample $\thickmuskip=2mu \medmuskip=2mu 1/4$ of the original neuron dimension in each layer, so $p$ may vary in different layer. $\thickmuskip=2mu \medmuskip=2mu p=16$ means that we sample 16 dimensions in each layer. Note that there are 25.6K parameters used for the final classification layer, which can not be saved in S-OPT. Table~\ref{sopt_dim} shows that S-OPT can achieve highly competitive accuracy with a reasonably large $p$.

\vspace{-0.8mm}
\subsection{Large Categorical Training}
\vspace{-0.7mm}

Previously, OPT is not applied to the final classification layer, since it makes little sense to fix random classifiers and learn an orthogonal matrix to transform them. However, learning the classification layer can be costly with large number of classes. The number of trainable parameters of the classification layer grows linearly with the number of classes. To address this, OPT can be used to learn the classification layer, because its number of trainable parameters only depends on the classifier dimension. To be fair, we \emph{only} learn the last classification layer with OPT and the other layers are normally learned (CLS-OPT). The oracle learns the entire network normally. Experimental details are in Appendix~\ref{exp_set}.

\setlength{\columnsep}{5pt}
\begin{wrapfigure}{r}{0.235\textwidth}
\begin{center}
\advance\leftskip+1mm
    \renewcommand{\captionlabelfont}{\footnotesize}
    \vspace{-0.295in}  
    \includegraphics[width=0.232\textwidth]{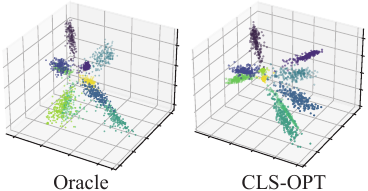}
    \vspace{-0.27in} 
    \caption{\footnotesize Feature visualization. }\label{cls_opt_vis}
\vspace{-0.235in} 
\end{center}
\end{wrapfigure}

We intuitively compare the oracle and CLS-OPT by visualizing the deep MNIST features following \cite{liu2016lsoftmax}. The features are the direct outputs of CNN by setting the output dimension as 3. Fig.~\ref{cls_opt_vis} show that even if CLS-OPT fixes randomly initialized classifiers, it can still learn discriminative and separable deep features.

\setlength{\columnsep}{5pt}
\begin{wraptable}{r}[0cm]{0pt}
    \centering
    \scriptsize
    \newcommand{\tabincell}[2]{\begin{tabular}{@{}#1@{}}#2\end{tabular}}
    \setlength{\tabcolsep}{1.7pt}
\renewcommand{\captionlabelfont}{\footnotesize}
\begin{tabular}{c | c c |c c} 
\specialrule{0em}{-8.5pt}{0pt}
\hline  
\multirow{2}{*}[0pt]{Method}& \multicolumn{2}{c|}{ResNet-18A} &  \multicolumn{2}{c}{ResNet-18B}\\
 & Error & Params & Error & Params \\\hline
Oracle & \textbf{18.08} & 64.0K & 12.12 & 512K \\
CLS-OPT & 21.12 & \textbf{8.13K} & \textbf{12.05} & \textbf{131K}\\
 \hline 
  \specialrule{0em}{0pt}{-8pt}
\end{tabular}
\caption{\footnotesize CLS-OPT on ImageNet.}
\label{large_class_imagenet}
\vspace{-3mm}
\end{wraptable}

We evaluate its performance on ImageNet with 1K classes. We use ResNet-18 with different output dimensions (A:128, B:512). Table~\ref{large_class_imagenet} gives the top-5 test error (\%) and ``Params'' denotes the number of trainable parameters in the classification layer. CLS-OPT performs well with far less trainable parameters.

\setlength{\columnsep}{4pt}
\begin{wraptable}{r}[0cm]{0pt}
    \centering
    \scriptsize
    \newcommand{\tabincell}[2]{\begin{tabular}{@{}#1@{}}#2\end{tabular}}
    \setlength{\tabcolsep}{1.7pt}
\renewcommand{\captionlabelfont}{\footnotesize}
\begin{tabular}{c | c c |c c} 
\specialrule{0em}{-8.5pt}{0pt}
  \hline
\multirow{2}{*}[0pt]{Method}& \multicolumn{2}{c|}{512 Dim.} &  \multicolumn{2}{c}{1024 Dim.}\\
 & Acc. & Params & Acc. & Params \\\hline
Oracle & \textbf{95.7} & 5.41M & \textbf{96.4} & 10.83M \\
CLS-OPT & 94.9 & \textbf{131K} & 95.8 & \textbf{524K}\\
 \hline 
  \specialrule{0em}{0pt}{-8pt}
\end{tabular}
\caption{\footnotesize Verification (\%) on LFW.}
\label{face_recog}
\vspace{-4mm}
\end{wraptable}

Since face datasets usually contain large number of identities~\cite{guo2016ms}, it is natural to apply CLS-OPT to learn face embeddings. We train on CASIA~\cite{yi2014learning} which has 0.5M face images of 10,572 identities, and test on LFW~\cite{huang2008labeled}. Since the training and testing sets do not overlap, the task well evaluates the generalizability of learned features. All methods use CNN-20~\cite{Liu2017CVPR} and standard softmax loss. We set the output feature dimension as 512 or 1024. Table~\ref{face_recog} validates CLS-OPT's effectiveness.

\vspace{-1.65mm}
\section{Concluding Remarks}
\vspace{-1.2mm}

We propose a novel training framework for neural networks. By parameterizing neurons with weights and a shared orthogonal matrix, OPT can provably achieve small hyperspherical energy and yield superior generalizability. 

\vspace{1.85mm}
{
\begin{spacing}{0.44}
{\scriptsize \noindent\textbf{Acknowledgements.} Weiyang Liu and Adrian Weller are supported by DeepMind and the Leverhulme Trust via CFI. Adrian Weller acknowledges support from the Alan Turing Institute under EPSRC grant EP/N510129/1 and U/B/000074. Rongmei Lin and Li Xiong are supported by NSF under CNS-1952192, IIS-1838200.}
\end{spacing}
}

\newpage

{
    \footnotesize
    \bibliographystyle{ieee_fullname}
    \bibliography{OPT_ref}
}

\newpage
\onecolumn
\begin{appendix}
\begin{center}

{\LARGE \textbf{Appendix}}
\end{center}

{
  \hypersetup{hidelinks}
  \tableofcontents
  \noindent\hrulefill
}
\addtocontents{toc}{\protect\setcounter{tocdepth}{2}} 

\clearpage
\newpage

\section{Details of Unrolled Orthogoanlization Algorithms}\label{supp_ortho}

\subsection{Gram-Schmidt Process}

\textbf{Gram-Schmidt Process}. GS process is a method for orthonormalizing a set of vectors in an inner product space, \ie, the Euclidean space $\mathbb{R}^n$ equipped with the standard inner product. Specifically, GS process performs the following operations to orthogonalize a set of vectors $\{ \bm{u}_1,\cdots,\bm{u}_n \}\in\mathbb{R}^{n\times n}$:
\begin{equation}\label{sup_gs}
\begin{aligned}
\textnormal{Step 1:}\ &\tilde{\bm{e}}_1 = \bm{u}_1,\ \ \ \bm{e}_1 = \frac{\tilde{\bm{e}}_1}{\norm{\tilde{\bm{e}}_1}}\\
\textnormal{Step 2:}\ & \tilde{\bm{e}}_2 = \bm{u}_2 - \textnormal{Proj}_{\tilde{e}_1}(\bm{u}_2),\ \ \ \bm{e}_2 = \frac{\tilde{\bm{e}}_2}{\norm{\tilde{\bm{e}}_2}}\\
\textnormal{Step 3:}\ & \tilde{\bm{e}}_3 = \bm{u}_3 - \textnormal{Proj}_{\tilde{\bm{e}}_1}(\bm{u}_3)-\textnormal{Proj}_{\tilde{\bm{e}}_2}(\bm{u}_3),\ \ \ \bm{e}_3 = \frac{\tilde{\bm{e}}_3}{\norm{\tilde{\bm{e}}_3}}\\
\textnormal{Step 4:}\ & \tilde{\bm{e}}_4 = \bm{u}_4 - \textnormal{Proj}_{\tilde{\bm{e}}_1}(\bm{u}_4)-\textnormal{Proj}_{\tilde{\bm{e}}_2}(\bm{u}_4)-\textnormal{Proj}_{\tilde{\bm{e}}_3}(\bm{u}_4),\ \ \ \bm{e}_4 = \frac{\tilde{\bm{e}}_4}{\norm{\tilde{\bm{e}}_4}}\\
\vdots&\\
\textnormal{Step n:}\ & \tilde{\bm{e}}_n = \bm{u}_n - \textnormal{Proj}_{\tilde{\bm{e}}_1}(\bm{u}_n)-\textnormal{Proj}_{\tilde{\bm{e}}_2}(\bm{u}_n)-\textnormal{Proj}_{\tilde{\bm{e}}_3}(\bm{u}_n)-\cdots-\textnormal{Proj}_{\tilde{\bm{e}}_{n-1}}(\bm{u}_n),\ \ \ \bm{e}_n = \frac{\tilde{\bm{e}}_n}{\norm{\tilde{\bm{e}}_n}}
\end{aligned}
\end{equation}
where $\textnormal{Proj}_{\bm{a}}(\bm{b})=\frac{\langle \bm{a},\bm{b} \rangle}{\langle \bm{a},\bm{a} \rangle}\bm{a}$ denotes the projection of the vector $\bm{b}$ onto the vector $\bm{a}$. The set $\{\bm{e}_1,\bm{e}_2,\cdots,\bm{e}_n\}$ denotes the output orthonormal set. The algorithm flowchart can be described as follows:
\vspace{-4mm}

\begin{center}
\begin{minipage}{.5\linewidth}
\begin{algorithm}[H]
\caption{Gram-Schmidt Process}
\KwIn{$\bm{U}=\{\bm{u}_1,\bm{u}_2,\cdots,\bm{u}_n\}\in\mathbb{R}^{n\times n}$}
\KwOut{$\bm{R}=\{\bm{e}_1,\bm{e}_2,\cdots,\bm{e}_n\}\in\mathbb{R}^{n\times n}$}
$\bm{R}= \bm{0} $\;

\For{$j=1,2,\cdots,n$}{
    $\bm{q}_j= \bm{R}^\top\bm{u}_j$\;
    
    $\bm{t}=\bm{u}_j-\bm{R}\bm{q}_j$\;
    
    $\bm{q}_{jj}=\norm{\bm{t}}_2$\;
    
    $\bm{e}_j=\frac{\bm{t}}{\bm{q}_{jj}}$\;
}
\end{algorithm}
\end{minipage}
\end{center}

The vectors $\bm{q}_j,\forall j$ in the algorithm above are used to compute the QR factorization, which is not useful in orthogonalization and therefore does not need to be stored. When the GS process is implemented on a finite-precision computer, the vectors $\bm{e}_j,\forall j$ are often not quite orthogonal, because of rounding errors. Besides the standard GS process, there is a modified Gram-Schmidt (MGS) algorithm which enjoys better numerical stability. This approach gives the same result as the original formula in exact arithmetic and introduces smaller errors in finite-precision arithmetic. Specifically, GS computes the following formula:
\begin{equation}\label{supp_ej}
\begin{aligned}
    \tilde{\bm{e}}_j &= \bm{u}_j-\sum_{k=1}^{j-1}\textnormal{Proj}_{\tilde{\bm{e}}_k}(\bm{u}_j)\\
    \bm{e}_j &= \frac{\tilde{\bm{e}}_j}{\norm{\tilde{\bm{e}}_j}}
\end{aligned}
\end{equation}
Instead of computing the vector $\bm{e}_j$ as in Eq.~\ref{supp_ej}, MGS computes the orthogonal basis differently. MGS does not subtract the projections of the original vector set, and instead remove the projection of the previously constructed orthogonal basis. Specifically, MGS computes the following series of formulas:
\begin{equation}
\begin{aligned}
    \tilde{\bm{e}}_j^{(1)} &= \bm{u}_j-\textnormal{Proj}_{\tilde{\bm{e}}_1}(\bm{u}_j)\\
    \tilde{\bm{e}}_j^{(2)} &= \tilde{\bm{e}}_j^{(1)}-\textnormal{Proj}_{\tilde{\bm{e}}_2}(\tilde{\bm{e}}_j^{(1)})\\
    \vdots&\\
    \tilde{\bm{e}}_j^{(j-2)} &= \tilde{\bm{e}}_j^{(j-3)}-\textnormal{Proj}_{\tilde{\bm{e}}_2}(\tilde{\bm{e}}_j^{(j-3)})\\
    \tilde{\bm{e}}_j^{(j-1)} &= \tilde{\bm{e}}_j^{(j-2)}-\textnormal{Proj}_{\tilde{\bm{e}}_2}(\tilde{\bm{e}}_j^{(j-2)})\\
    \bm{e}_j &=\frac{\tilde{\bm{e}}_j^{(j-1)}}{\norm{\tilde{\bm{e}}_j^{(j-1)}}}
\end{aligned}
\end{equation}
where each step finds a vector $\tilde{\bm{e}}_j^{(i)}$ that is orthogonal to $\tilde{\bm{e}}_j^{(i-1)}$. Therefore, $\tilde{\bm{e}}_j^{(i)}$ is also orthogonalized against any errors brought by the computation of $\tilde{\bm{e}}_j^{(i-1)}$. In practice, although MGS enjoys better numerical stability, we find the empirical performance of GS and MGS is almost the same in OPT. However, MGS takes longer time to complete since the computation of each orthogonal basis is an iterative process. Therefore, we usually stick to classic GS for OPT.

\textbf{Iterative Gram-Schmidt Process}. Iterative Gram-Schmidt~(IGS) process is an iterative version of the GS process. It is shown in \cite{hoffmann1989iterative} that GS process can be carried out iteratively to obtain a basis matrix that is orthogonal in almost full working precision. The IGS algorithm is given as follows:
\vspace{-4mm}

\begin{center}
\begin{minipage}{.5\linewidth}
\begin{algorithm}[H]
\caption{Iterative Gram-Schmidt Process}
\KwIn{$\bm{U}=\{\bm{u}_1,\bm{u}_2,\cdots,\bm{u}_n\}\in\mathbb{R}^{n\times n}$}
\KwOut{$\bm{R}=\{\bm{e}_1,\bm{e}_2,\cdots,\bm{e}_n\}\in\mathbb{R}^{n\times n}$}
$\bm{R}= \bm{0} $\;

\For{$j=1,2,\cdots,n$}{
    $\bm{q}_j= \bm{0}$\;
    
    $\bm{t}=\bm{u}_j$\;
    
    \While{$\bm{t} \perp \textnormal{span}(\bm{e}_1,\cdots,\bm{e}_{j-1})$ is False}{
    $\bm{p}=\bm{t}$\;
    
    $\bm{s} = \bm{R}^\top \bm{p}$\;
    
    $\bm{v}=\bm{R}\bm{s}$\;
    
    $\bm{t}=\bm{p}-\bm{v}$\;
    
    $\bm{q}_j\leftarrow\bm{q}_j+\bm{s}$
    }
    
    $\bm{q}_{jj}=\norm{\bm{t}}_2$\;
    
    $\bm{e}_j=\frac{\bm{t}}{\bm{q}_{jj}}$
}
\end{algorithm}
\end{minipage}
\end{center}

The vectors $\bm{q}_j,\forall j$ in the algorithm above are used to compute the QR factorization, which is not useful in orthogonalization and therefore does not need to be explicitly computed. The while loop in IGS is an iterative procedure. In practice, we can unroll a fixed number of steps for the while loop in order to improve the orthogonality. The resulting $\bm{q}_j$ in the $j$-th step corresponds to the solution of the equation $\tilde{\bm{R}}^\top\tilde{\bm{R}}\bm{q}_j=\tilde{\bm{R}}^\top\bm{u}_j$ where $\tilde{\bm{R}}=\{\bm{e}_1,\cdots,\bm{e}_{j-1}\}$. The IGS process corresponds to the Gauss-Jacobi iteration for solving this equation.

Both GS and IGS are easy to be embedded in the neural networks, since they are both differentiable. In our experiments, we find that the performance gain of unrolling multiple steps in IGS over GS is not very obvious (partially because GS has already achieved nearly perfect orthogonality), but IGS costs longer training time. Therefore, we unroll the classic GS process by default.

\subsection{Householder Reflection}
Let $\bm{v}\in\mathbb{R}^n$ be a non-zero vector. A matrix $\bm{H}\in\mathbb{R}^{n\times n}$ of the form
\begin{equation}
    \bm{H}=\bm{I}-\frac{2\bm{v}\bm{v}^\top}{\bm{v}^\top\bm{v}}
\end{equation}
is a Householder reflection. The vector $\bm{v}$ is the Householder vector. If a vector $\bm{x}$ is multiplied by the matrix $\bm{H}$, then it will be reflected in the hyperplane $\textnormal{span}(\bm{v})^\perp$. Householder matrices are symmetric and orthogonal.

For a vector $\bm{x}\in\mathbb{R}^n$, we let $\bm{v}=\bm{x}\pm\norm{\bm{x}}_2\bm{e}_1$ where $\bm{e}_1$ is a vector of $\{1,0,\cdots,0\}$ (the first element is $1$ and the remaining elements are $0$). Then we construct the Householder reflection matrix with $\bm{v}$ and multiply it to $\bm{x}$:
\begin{equation}
    \bm{H}\bm{x}=\bigg( \bm{I}-2\frac{\bm{v}\bm{v}^\top}{\bm{v}^\top\bm{v}} \bigg)\bm{x}=\mp\norm{\bm{x}}_2\bm{e}_1
\end{equation}
which indicates that we can make any non-zero vector become $\alpha\bm{e}_1$ where $\alpha$ is some constant by using Householder reflection. By left-multiplying a reflection we can turn a dense vector $\bm{x}$ into a vector with the same length and with only a single nonzero entry. Repeating this $n$ times gives us the Householder QR factorization, which also orthogonalizes the original input matrix. Householder reflection orthogonalizes a matrix $\bm{U}=\{\bm{u}_1,\cdots,\bm{u}_n\}$ by triangularizing it:
\begin{equation}
    \bm{U}=\bm{H}_1\bm{H}_2\cdots\bm{H}_n\bm{R}
\end{equation}
where $\bm{R}$ is a upper-triangular matrix in the QR factorization. $\bm{H}_j,j\geq 2$ is constructed by $\textnormal{Diag}(\bm{I}_{j-1},\tilde{\bm{H}}_{n-j+1})$ where $\tilde{\bm{H}}_{n-j+1}\in\mathbb{R}^{(n-j+1)\times (n-j+1)}$ is the Householder reflection that is performed on the vector $\bm{U}_{(j:n,j)}$. The algorithm flowchart is given as follows:

\vspace{-3mm}
\begin{center}
\begin{minipage}{.65\linewidth}
\begin{algorithm}[H]

\caption{Householder Reflection Orthogonalization}
\KwIn{$\bm{U}=\{\bm{u}_1,\bm{u}_2,\cdots,\bm{u}_n\}\in\mathbb{R}^{n\times n}$}
\KwOut{$\bm{U}=\bm{Q}\bm{R}$, where $\bm{Q}=\{\bm{e}_1,\bm{e}_2,\cdots,\bm{e}_n\}\in\mathbb{R}^{n\times n}$ is the orthogonal matrix and $\bm{R}\in\mathbb{R}^{n\times n}$ is a upper triangular matrix}

\For{$j=1,2,\cdots,n-1$}{
    $\{\bm{v},\beta\}=\textnormal{\texttt{Householder}}(\bm{U}_{j:n,j})$\;
    
    $\bm{U}_{j:n,j:n}\leftarrow \bm{U}_{j:n,j:n}-\beta\bm{v}(\bm{v}^\top\bm{U}_{j:n,j:n})$\;
    
    $\bm{U}_{j+1:n,j}\leftarrow\bm{v}_{(2:\textnormal{end})}$
    }
    
    \SetKwProg{myproc}{function}{}{}
    \myproc{$\{\bm{v},\beta\}=\textnormal{\texttt{Householder}}(\bm{\bm{x}})$}{
    
    $\sigma^2=\norm{\bm{x}_{2:\textnormal{end}}}^2_2$\;
    
    $\bm{v}\leftarrow \begin{bmatrix}
    1 \\
    \bm{x}_{2:\textnormal{end}}
    \end{bmatrix}$\;
    
    \eIf{$\sigma^2=0$}{$\beta=0$}{
    \eIf{$\bm{x}_1\leq 0$}{$\bm{v}_1=\bm{x}_1-\sqrt{\bm{x}_1^2+\sigma^2}$}{
    $\bm{v}_1=-\dfrac{\sigma^2}{\bm{x}_1+\sqrt{\bm{x}_1^2+\sigma^2}}$
    }
    $\beta=\frac{2\bm{v}_1^2}{\sigma^2+\bm{v}_1^2}$\;
    
    $\bm{v}\leftarrow\frac{\bm{v}}{\bm{v}_1}$
    
    }

    }
    
    \SetKwProg{myproc}{end function}{}{}
    \myproc{}{
    
    }

\end{algorithm}
\end{minipage}
\end{center}

The algorithm follows the Matlab notation where $\bm{U}_{j:n,j:n}$ denotes the submatrix of $\bm{U}$ from the $j$-th column to the $n$-th column and from the $j$-th row to the $n$-th row. Note that, there are a number of variants for the Householder reflection orthogonalization, such as the implicit variant where we do not store each reflection $\bm{H}_j$ explicitly. Here $\bm{Q}$ is the final orthogonal matrix we need.

\subsection{Löwdin’s Symmetric Orthogonalization}

Let $\bm{U}=\{\bm{u}_1,\bm{u}_2,\cdots,\bm{u}_n\}$ be a set of linearly independent vectors in a $n$-dimensional space. We define a general non-singular linear transformation $\bm{A}$ that can transform the basis $\bm{U}$ to a new basis $\bm{R}$:
\begin{equation}
    \bm{R}=\bm{U}\bm{A}
\end{equation}
where the basis $\bm{R}$ will be orthonormal if (the transpose will become conjugate transpose in complex space)
\begin{equation}
    \bm{R}^\top\bm{R}=(\bm{U}\bm{A})^\top(\bm{U}\bm{A})=\bm{A}^\top\bm{U}^\top\bm{U}\bm{A}=\bm{A}^\top\bm{M}\bm{A}=\bm{I}
\end{equation}
where $\bm{M}=\bm{U}^\top\bm{U}$ is the gram matrix of the given basis $\bm{U}$.

A general solution to this orthogonalization problem can be obtained via the substitution:
\begin{equation}
    \bm{A}=\bm{M}^{-1}\bm{B}
\end{equation}
in which $\bm{B}$ is an arbitrary orthogonal (or unitary) matrix. When $\bm{B}=\bm{I}$, we will have the symmetric orthogonalization, namely
\begin{equation}
    \bm{R}:=\bm{\Phi}=\bm{U}\bm{M}^{-\frac{1}{2}}
\end{equation}
When $\bm{B}=\bm{V}$ in which $\bm{V}$ diagonalizes $\bm{M}$, then we have the canonical orthogonalization, namely
\begin{equation}
    \bm{\Lambda}=\bm{U}\bm{V}\bm{d}^{-\frac{1}{2}}.
\end{equation}
Because $\bm{V}$ diagonalizes $\bm{M}$, we have that $\bm{M} = \bm{V}\bm{d}\bm{V}^\top$. Therefore, we have the $\bm{M}^{-\frac{1}{2}}$ transformation as $\bm{M}^{-\frac{1}{2}}=\bm{V}\bm{d}^{-\frac{1}{2}}\bm{V}^\top$. This is essentially an eigenvalue decomposition of the symmetric matrix $\bm{M}=\bm{U}^\top\bm{U}$. 

In order to compute the Löwdin’s symmetric orthogonalized basis sets, we can use singular value decomposition. Specifically, SVD of the original basis set $\bm{U}$ is given by
\begin{equation}
    \bm{U}=\bm{W}\bm{\Sigma}\bm{V}^\top
\end{equation}
where both $\bm{W}\in\mathbb{R}^{n\times n}$ and $\bm{U}\in\mathbb{R}^{n\times n}$ are orthogonal matrices. $\bm{\Sigma}$ is the diagonal matrix of singular values. Therefore, we have that
\begin{equation}
\begin{aligned}
    \bm{R}&=\bm{U}\bm{M}^{-\frac{1}{2}}\\
    &=\bm{W}\bm{\Sigma}\bm{V}^\top\bm{V}\bm{d}^{-\frac{1}{2}}\bm{V}^\top\\
    &=\bm{W}\bm{\Sigma}\bm{d}^{-\frac{1}{2}}\bm{V}^\top
\end{aligned}
\end{equation}
where we have $\bm{\Sigma}=\bm{d}^{\frac{1}{2}}$ due to the connections between eigenvalue decomposition and SVD. Therefore, we end up with
\begin{equation}
    \bm{R}=\bm{W}\bm{V}^\top
\end{equation}
which is the output orthogonal matrix for Löwdin’s symmetric orthogonalization.

An interesting feature of the symmetric orthogonalization is to ensure that
\begin{equation}
    \bm{R} = \arg\min_{\bm{P}\in \textnormal{orth}(\bm{U})} \sum_i\norm{\bm{P}_i-\bm{U}_i} 
\end{equation}
where $\bm{P}_i$ and $\bm{U}_i$ are the $i$-th column vectors of $\bm{P}\in\mathbb{R}^{n\times n}$ and $\bm{U}$, respectively. $\textnormal{orth}(\bm{U})$ denotes the set of all possible orthonormal sets in the range of $\bm{U}$. This means that the symmetric orthogonalization functions $\bm{R}_i$ (or $\bm{\Phi}_i$) are the least distant in the Hilbert space from the original functions $\bm{U}_i$. Therefore, symmetric orthogonalization indicates the gentlest pushing of the directions of the vectors in order to make them orthogonal.

More interestingly, the symmetric orthogonalized basis sets has unique geometric properties~\cite{srivastava2000unified,annavarapu2013singular} if we consider the Schweinler-Wigner matrix in terms of the sum of squared projections.

\newpage
\section{Proof of Theorem~\ref{unifrom}}

To be more specific, neurons with each element initialized by a zero-mean Gaussian distribution are uniformly distributed on a hypersphere. We show this argument with the following theorem.
\begin{theorem}\label{sphereuniform}
The normalized vector of Gaussian variables is uniformly distributed on the sphere. Formally, let $x_1,x_2,\cdots,x_n\sim \mathcal{N}(0,1)$ and be independent. Then the vector
\begin{equation}
    \bm{x}=\bigg{[} \frac{x_1}{z},\frac{x_2}{z},\cdots,\frac{x_n}{z} \bigg{]}
\end{equation}
follows the uniform distribution on $\mathbb{S}^{n-1}$, where $z=\sqrt{x_1^2+x_2^2+\cdots+x_n^2}$ is a normalization factor.
\end{theorem}
\begin{proof}
A random variable has distribution $\mathcal{N}(0,1)$ if it has the density function
\begin{equation}
    f(x)=\frac{1}{\sqrt{2\pi}}e^{-\frac{1}{2}x^2}.
\end{equation}
A $n$-dimensional random vector $\bm{x}$ has distribution $\mathcal{N}(0,1)$ if the components are independent and have distribution $\mathcal{N}(0,1)$ each. Then the density of $\bm{x}$ is given by
\begin{equation}
    f(x)=\frac{1}{(\sqrt{2\pi})^n}e^{-\frac{1}{2}\langle x,x\rangle}.
\end{equation}
Then we introduce the following lemma (Lemma~\ref{lemma_sphereuniform}) about the orthogonal-invariance of the normal distribution.
\begin{lemma}\label{lemma_sphereuniform}
Let $\bm{x}$ be a $n$-dimensional random vector with distribution $\mathcal{N}(0,1)$ and $\bm{U}\in\mathbb{R}^{n\times n}$ be an orthogonal matrix ($\bm{U}\bm{U}^\top =\bm{U}^\top\bm{U}=\bm{I} $). Then $\bm{Y}=\bm{U}\bm{x}$ also has the distribution of $\mathcal{N}(0,1)$.
\end{lemma}
\begin{proof}
For any measurable set $A\subset\mathbb{R}^n$, we have that
\begin{equation}
\begin{aligned}
    P(Y\in A)&= P(X\in U^\top A)\\
    &=\int_{U^\top A}\frac{1}{(\sqrt{2\pi})^n} e^{-\frac{1}{2}\langle x,x \rangle}\\
    &=\int_A\frac{1}{(\sqrt{2\pi})^n}e^{-\frac{1}{2}\langle Ux, Ux \rangle}\\
    &=\int_A\frac{1}{(\sqrt{2\pi})^n}e^{-\frac{1}{2}\langle x, x \rangle}
\end{aligned}
\end{equation}
because of orthogonality of $U$. Therefore the lemma holds.
\end{proof}

Because any rotation is just a multiplication with some orthogonal matrix, we know that normally distributed random vectors are invariant to rotation. As a result, generating $\bm{x}\in\mathbb{R}^n$ with distribution $\mathbb{N}(0,1)$ and then projecting it onto the hypersphere $\mathbb{S}^{n-1}$ produces random vectors $U=\frac{\bm{x}}{\|\bm{x}\|}$ that are uniformly distributed on the hypersphere. Therefore the theorem holds.
\end{proof}

Then we show the normalized vector $\bm{y}$ where each element follows a zero-mean Gaussian distribution with some constant variance $\sigma^2$:
\begin{equation}
    \bm{y}=\bigg{[} \frac{y_1}{r},\frac{y_2}{r},\cdots,\frac{y_n}{r} \bigg{]}
\end{equation}
where $r=\sqrt{y_1^2+y_2^2+\cdots+y_n^2}$. Because we have that $\frac{y_i}{\sigma}\sim\mathcal{N}(0,1)$, we can rewrite $\bm{y}$ as the following random vector:
\begin{equation}
    \bm{y}=\bigg{[} \frac{y_1/\sigma}{r/\sigma},\frac{y_2/\sigma}{r/\sigma},\cdots,\frac{y_n/\sigma}{r/\sigma} \bigg{]}
\end{equation}
where $r/\sigma=\sqrt{(y_1/\sigma)^2+(y_2/\sigma)^2+\cdots+(y_n/\sigma)^2}$. Therefore, we directly can apply Theorem~\ref{sphereuniform} and conclude that $\bm{y}$ also follows the uniform distribution on $\mathbb{S}^{n-1}$. Now we obtain that any random vector with each element following a zero-mean Gaussian distribution with some constant variance follows the uniform distribution on $\mathbb{S}^{n-1}$.

Then we show that the minimum hyperspherical energy asymptotically corresponds to the uniform distribution over the unit hypersphere. We first write down the hyperspherical energy of $N$ neurons $\{\bm{w}_1,\cdots,\bm{w}_N\in\mathbb{R}^{d+1}\}$ (we also define that $\hat{\bm{w}}_i=\frac{\bm{w}_i}{\|\bm{w}_i\|}\in\mathbb{S}^d$):
\begin{equation}
    \bm{E}_{s,d}(\hat{\bm{w}}_i|_{i=1}^N) = \sum_{i=1}^{N}\sum_{j=1,j\neq i}^{N} f_s\big(\norm{\hat{\bm{w}}_i-\hat{\bm{w}}_j}\big)=\left\{
{\begin{array}{*{20}{l}}
{\sum_{i\neq j} \norm{\hat{\bm{w}}_i-\hat{\bm{w}}_j}^{-s},\ \ s>0}\\
{\sum_{i\neq j} \log\big(\norm{\hat{\bm{w}}_i-\hat{\bm{w}}_j}^{-1}\big),\ \ s=0}
\end{array}} \right.
\end{equation}
where $s$ is a hyperparameter that controls the behavior of hyperspherical energy. We then define a $N$-point minimal hyperspherical $s$-energy over $\bm{A}$ with
\begin{equation}\label{mini_s_energy}
    \bm{\varepsilon}_{s,d}(\bm{A},\hat{\bm{W}}_N)  :=\inf_{\hat{\bm{W}}_N\subset \bm{A}}\bm{E}_{s,d}(\hat{\bm{w}}_i|_{i=1}^N)
\end{equation}
where we denote that $\hat{\bm{W}}_N=\{\hat{\bm{w}}_1,\cdots,\hat{\bm{w}}_N\}$. Typically, we will assume that $\bm{A}$ is compact. Based on \cite{hardin2005minimal}, we discuss the asymptotic behavior (as $N\rightarrow \infty$) of $\bm{\varepsilon}_{s,d}(\bm{A},\hat{\bm{W}}_N)$ in three different scenarios: (1) $0<s<d$; (2) $s=d$; and $s>d$. The reason behind is the behavior of the following energy integral:
\begin{equation}\label{energy_int}
I_s(\mu)=\iint_{\mathbb{S}^d\times\mathbb{S}^d}\|\bm{u}-\bm{v}\|^{-s}d\mu(\bm{u})d\mu(\bm{v}),
\end{equation}
is quite different under these three scenarios. In scenario (1), Eq.~\ref{energy_int} that is taken over all probability measures $\mu$ supported on $\mathcal{S}^d$ will be minimal for normalized Lebesgue measure $\frac{\mathcal{H}_d(\cdot)|_{\mathbb{S}^d}}{\mathcal{H}_d(\mathbb{S}^d)}$ on $\mathbb{S}^d$. In the case of $s\geq d$, we will have that $I_s(\mu)$ is positive infinity for all such measures $\mu$. Therefore, the behaviour of the minimum hyperspherical energy is different in these three cases. In general, as the parameter $s$ increases, there is a transition from the global effects to the more local influences (from nearest neighbors). The transition happens when $s=d$. However, we typically have $0<s<d$ in the neural networks. Therefore, we will mostly study the case of $0<s<d$ and the theoretical asymptotic behavior is quite standard results from the potential theory~\cite{kazarinoff1975foundations}. From the classic potential theory, we have the following known lemma:

\vspace{2mm}
\begin{lemma}
If $0<s<d$, we have that
\begin{equation}
    \lim_{N\rightarrow\infty}\frac{\bm{\varepsilon}_{s,d}(\bm{S}^d,\hat{\bm{W}}_N) }{N^2}=I_s\bigg(\frac{\mathcal{H}_d(\cdot)|_{\mathbb{S}^d}}{\mathcal{H}_d(\mathbb{S}^d)}\bigg)
\end{equation}
Moreover, any sequence of $s$-energy configuration of minimal hyperspherical energy ($(\hat{\bm{W}}_N^*)_2^\infty\subset\mathbb{S}^d$) is asymptotically uniformly distributed in the sense that for the weak-star topology of measures,
\begin{equation}\label{mhe_uniform}
    \frac{1}{N}\sum_{\bm{v}\in\hat{\bm{W}}_N^*}\delta_{\bm{v}}\rightarrow \frac{\mathcal{H}_d(\cdot)|_{\mathbb{S}^d}}{\mathcal{H}_d(\mathbb{S}^d)}~~~~\textnormal{as}~N\rightarrow\infty
\end{equation}
where $\delta_{\bm{v}}$ denotes the unit point mass at $\bm{v}$.
\end{lemma}

The lemma above concludes that the neuron configuration with minimal hyperspherical energy asymptotically corresponds to the uniform distribution on $\mathbb{S}^d$ when $0<s<d$. From \cite{hardin2005minimal}, we also have the following lemma that shows the same conclusion holds for the the case of $s=d$ and $s>d$:

\vspace{2mm}
\begin{lemma}
Let $\mathcal{B}^d:=\bar{B}(0,1)$ denote the closed unit ball in $\mathbb{R}^d$. For the case of $s=d$, we have that
\begin{equation}
    \lim_{N\rightarrow\infty}\frac{\bm{\varepsilon}_{s,d}(\bm{S}^d,\hat{\bm{W}}_N) }{N^2\log N}= \frac{\mathcal{H}_d(\mathcal{B}^d)}{\mathcal{H}_d(\mathbb{S}^d)}=\frac{1}{d}\frac{\Gamma(\frac{d+1}{2})}{\sqrt{\pi}\Gamma(\frac{d}{2})}
\end{equation}
and any sequence $(\hat{\bm{W}}_N^*)\subset\mathbb{S}^d$ of minimal $s$-energy configurations satisfies Eq.~\ref{mhe_uniform}.
\end{lemma}

The lemma above shows that the same conclusion holds for $s=d$. For the case of $s>d$, the theoretical analysis is more involved, but the conclusion that the neuron configuration with minimal hyperspherical energy asymptotically corresponds to the uniform distribution on $\mathbb{S}^d$ still holds. Note that, we usually will not have the case of $s>d$ in our applications.

\newpage
\section{Proof of Theorem~\ref{sopt_thm}}
We consider a set of $n$ $d$-dimensional neurons $\bm{W}=\{\bm{w}_1,\cdots,\bm{w}_n\}\in\mathbb{R}^{d\times n}$. The hyperspherical energy of the original set of neurons can be written as:
\begin{equation}
    \bm{E}(\bm{w}_i|_{i=1}^n) = \sum_{i=1}^{n}\sum_{j=1,j\neq i}^{n} \norm{\frac{\bm{w}_i}{\|\bm{w}_i\|_2}-\frac{\bm{w}_j}{\|\bm{w}_j\|}}^{-1}
\end{equation}
which means that if the pairwise angle between any two neurons stays unchanged, then the hyperspherical energy will also stay unchanged. Now we consider the cosine value of the angle $\theta_{(\bm{w}_i,\bm{w}_j)}$ between any two neuron $\bm{w}_i$ and $\bm{w}_j$:
\begin{equation}
    \cos(\theta_{(\bm{w}_i,\bm{w}_j)})=\frac{\bm{w}_i^\top\bm{w}_j}{\|\bm{w}_i\|\cdot\|\bm{w}_j\|}=\frac{\sum_{k=1}^d w_{ik}\cdot w_{jk}}{\|\bm{w}_i\|\cdot\|\bm{w}_j\|}
\end{equation}
where $w_{ik}$ is the $k$-th element of the neuron $\bm{w}_{i}$. From the equation above, we can observe that permuting the order of the elements in the neurons together will not change the angle. For example, switching the $i$-th and $j$-th element in all the neurons will not change the hyperspherical energy. Assume that we randomly select $p$ dimensions from the $d$ dimensions and denote the set of $p$ dimension as $\bm{s}=\{s_1,\cdots, s_p\}\in\mathbb{R}^p$. Therefore we can construct a new set of neurons $\tilde{\bm{W}}=\{\tilde{\bm{w}}_1,\cdots,\tilde{\bm{w}}_n\}$ by permuting the $p$ dimensions in $\bm{s}$ to become the first $p$ elements for all the neurons. Essentially, we use permutation to make $\tilde{\bm{w}}_i=\bm{w}_{s_i}$ for $i\in[1,p]$. Therefore, we can have the following equation:
\begin{equation}
    \bm{E}(\bm{w}_i|_{i=1}^n) =\bm{E}(\tilde{\bm{w}}_i|_{i=1}^n) 
\end{equation}
Then we consider an orthogonal matrix $\bm{R}_p\in\mathbb{R}^{d\times d}$ that is used to transform the $p$ dimension in the neurons. The equivalent orthogonal transformation for the $d$-dimensional neurons $\tilde{\bm{W}}$ is
\begin{equation}
    \tilde{\bm{R}}=\begin{bmatrix}
    \bm{R}_p & \bm{0}\\
    \bm{0} & \bm{I}_{n-p}
    \end{bmatrix}=\begin{bmatrix}
    \bm{R}_p & 0 & \cdots & 0\\
    0 & 1 &\ddots &0\\
    \vdots& \ddots &\ddots&0 \\
    0 & \cdots & 0& 1\\
    \end{bmatrix}
\end{equation}
where $\bm{I}_{n-p}$ is an identity matrix of size $(n-p)\times (n-p)$. It is easy to verify that $\tilde{\bm{R}}$ is also an orthogonal matrix: $\tilde{\bm{R}}^\top\tilde{\bm{R}}=\bm{I}_n$. Then we permute the order of $\tilde{\bm{W}}$ back to the original neuron set $\bm{W}$ and obtain a new set of neurons $\bm{W}^t=\{\bm{w}_1^t,\cdots,\bm{w}_n^t\}$. $\bm{W}^t$ is in fact the result of directly performing orthogonal transformation to the $p$ dimensions in $\bm{W}$. Because any order permutation of elements in neurons does not change the hyperspherical energy, we have the following equation
\begin{equation}
    \bm{E}(\bm{w}_i|_{i=1}^n) =\bm{E}(\tilde{\bm{w}}_i|_{i=1}^n) =\bm{E}(\tilde{\bm{R}}\tilde{\bm{w}}_i|_{i=1}^n) =\bm{E}(\tilde{\bm{R}}\bm{w}^t_i|_{i=1}^n)
\end{equation}
which concludes our proof.

\newpage
\section{Experimental Settings}\label{exp_set}

\begin{table*}[h]
    \renewcommand{\captionlabelfont}{\footnotesize}
    \newcommand{\tabincell}[2]{\begin{tabular}{@{}#1@{}}#2\end{tabular}}
    \centering
    \setlength{\abovecaptionskip}{4pt}
    \setlength{\belowcaptionskip}{-5pt}
    \scriptsize
    \begin{tabular}{|c|c|c|c|c|c|c|}
        \hline
        Layer & CNN-6 (CIFAR-100) & CNN-9 (CIFAR-100) & CNN-10 (ImageNet-2012) \\
        \hline
        Conv1.x  & [3$\times$3, 64]$\times$2 & [3$\times$3, 64]$\times$3  & \tabincell{c}{[7$\times$7, 64], Stride 2\\3$\times$3, Max Pooling, Stride 2\\{[3$\times$3, 64]$\times$3} }\\\hline
        Pool1&\multicolumn{3}{c|}{2$\times$2 Max Pooling, Stride 2}\\\hline
        Conv2.x  & [3$\times$3, 64]$\times$2 & [3$\times$3, 64]$\times$3 & [3$\times$3, 128]$\times$3\\\hline
        Pool2 & \multicolumn{3}{c|}{2$\times$2 Max Pooling, Stride 2}\\\hline
        Conv3.x & [3$\times$3, 64]$\times$2 & [3$\times$3, 64]$\times$3 & [3$\times$3, 256]$\times$3\\\hline
        Pool3 & \multicolumn{3}{c|}{2$\times$2 Max Pooling, Stride 2}\\\hline
        Fully Connected & 64 & 64 & 256\\\hline
    \end{tabular}
    \caption{\footnotesize Our plain CNN architectures with different convolutional layers. Conv1.x, Conv2.x and Conv3.x denote convolution units that may contain multiple convolution layers. E.g., [3$\times$3, 64]$\times$3 denotes 3 cascaded convolution layers with 64 filters of size 3$\times$3.}\label{netarch}
\end{table*}

\begin{table*}[h]
    \renewcommand{\captionlabelfont}{\footnotesize}
    \newcommand{\tabincell}[2]{\begin{tabular}{@{}#1@{}}#2\end{tabular}}
    \centering
    \setlength{\abovecaptionskip}{4pt}
    \setlength{\belowcaptionskip}{5pt}
    \scriptsize
    \begin{tabular}{|c|c|c|c|c|}
        \hline
        Layer & ResNet-20 (CIFAR-100) & ResNet-32 (CIFAR-100) \\
        \hline
        Conv1.x & \tabincell{c}{[3$\times$3, 16]$\times$1\\$\left[\begin{aligned} &3\times 3, 16\\&3\times3, 16\end{aligned}\right]\times 3$} & \tabincell{c}{[3$\times$3, 16]$\times$1\\$\left[\begin{aligned} &3\times 3, 16\\&3\times3, 16\end{aligned}\right]\times 5$} \\ \hline
        Conv2.x  & \tabincell{c}{$\left[\begin{aligned} &3\times3, 32\\&3\times3, 32\end{aligned}\right]\times 3$}  & $\left[\begin{aligned} &3\times 3, 32\\&3\times3, 32\end{aligned}\right]\times 5$ \\\hline
        Conv3.x  & \tabincell{c}{$\left[\begin{aligned} &3\times3, 64\\&3\times3, 64\end{aligned}\right]\times 3$} & $\left[\begin{aligned} &3\times 3, 64\\&3\times3, 64\end{aligned}\right]\times 5$  \\\hline
        & \multicolumn{2}{c|}{Average Pooling}  \\\hline
    \end{tabular}
    \caption{\footnotesize Our ResNet architectures with different convolutional layers. Conv0.x, Conv1.x, Conv2.x, Conv3.x and Conv4.x denote convolution units that may contain multiple convolutional layers, and residual units are shown in double-column brackets. Conv1.x, Conv2.x and Conv3.x usually operate on different size feature maps. These networks are essentially the same as \cite{he2016deep}, but some may have a different number of filters in each layer. The downsampling is performed by convolutions with a stride of 2. E.g., [3$\times$3, 64]$\times$4 denotes 4 cascaded convolution layers with 64 filters of size 3$\times$3, S2 denotes stride 2. }\label{netarch2}
\end{table*}

\textbf{Reported Results}. For all the experiments on MLPs and CNNs (except CNNs in the few-shot learning), we report testing error rates. For the few-shot learning experiment, we report testing accuracy. For all the experiments on both GCNs and PointNets, we report testing accuracy. All results are averaged over 10 runs of the model.

\textbf{Multilayer perceptron}. We conduct digit classification task on MNIST with a three-layer multilayer perceptron following this repository\footnote{\url{https://github.com/hwalsuklee/tensorflow-mnist-MLP-batch_normalization-weight_initializers}} . The input dimension of each MNIST digit is $28\times 28$, which is 784 dimensions after flattened. Our two hidden layers have 256 output dimensions, \ie, 256 neurons. The output layer will output 10 logits for classification. We use a cross-entropy loss with softmax function. For the optimization,  we use a momentum SGD with learning rate 0.01, momentum 0.9 and batch size 100. The training stops at 100 epochs.

\textbf{Convolutional neural networks.} The network architectures used in the paper are elaborated in Table~\ref{netarch} and Table~\ref{netarch2}. For CIFAR-100, we use 128 as the mini-batch size. We use momentum SGD with momentum 0.9 and the learning rate starts with 0.1, divided by 10 when the performance is saturated. For ImageNet-2012, we use batch size 128 and start with learning rate 0.1. The learning rate is divided by 10 when the performance is saturated, and the training is terminated at 700k iterations. For ResNet-20 and ResNet-32 on CIFAR-100, we use exactly the same architecture used on CIFAR-10 as \cite{he2016deep}. The rotation matrix is initialized with random normal distribution (mean is 0 and variance is 1). Note that, for all the compared methods, we always use the best possible hyperparameters to make sure that the comparison is fair. The baseline has exactly the same architecture and training settings as the one that OPT uses. If not otherwise specified, standard $\ell_2$ weight decay ($\thickmuskip=2mu \medmuskip=2mu 5e-4$) is applied to all the neural network including baselines and the networks that use OPT training. 

\par
\textbf{Few-shot learning}. The network architecture (Table~\ref{netarch_fewshot}) we used for few-shot learning experiments is the same as that used in \cite{chen2019closer}. In our experiments, we show comparison of our OPT training with standard training on `baseline' and `baseline++' settings in \cite{chen2019closer}. In `baseline' setting, a standard CNN model is pretrained on the whole meta-train dataset (standard non-MAML supervised training) and later only the classifier layer is finetuned on few-shot dataset. `baseline++' differs from `baseline' on the classifier: in `baseline', each output dimension of the classifier is computed as the inner product between weight $w$ and input $x$, i.e. $w \cdot x$; while in `baseline++' it becomes the scaled cosine distance $c \frac{w \cdot x}{\norm{w} \norm{x}}$ where $c$ is a positive scalar. Following \cite{chen2019closer}, we set $c=2$.

During pretraining, the model is trained for 200 epochs on the meta-train set of mini-ImageNet with an Adam optimzer (learning rate $1e-3$, weight decay $5e-4$) and the classifier is discarded after pretraining. The model is later finetuned, with a new classifier, on the few-shot samples (5 way, support size 5) with a momentum SGD optimizer (learning rate $1e-2$, momentum $0.9$, dampening $0.9$, weight decay $1e-3$, batch size 4) for 100 epochs. We re-initialize the classifier for each few-shot sample. 

\begin{table*}[t]
    \renewcommand{\captionlabelfont}{\footnotesize}
    \newcommand{\tabincell}[2]{\begin{tabular}{@{}#1@{}}#2\end{tabular}}
    \centering
    \setlength{\abovecaptionskip}{4pt}
    \setlength{\belowcaptionskip}{0pt}
    \scriptsize
    \begin{tabular}{|c|c|c|c|c|c|c|}
        \hline
        Layer & CNN-4 \\
        \hline
        Conv1  & 3$\times$3, 64 \\ \hline
        Pool1 & 2$\times$2 Max Pooling, Stride 2\\ \hline
        Conv2  & 3$\times$3, 64 \\ \hline
        Pool2 & 2$\times$2 Max Pooling, Stride 2\\ \hline
        Conv3  & 3$\times$3, 64 \\ \hline
        Pool3 & 2$\times$2 Max Pooling, Stride 2\\ \hline
        Conv4  & 3$\times$3, 64 \\ \hline
        Pool4 & 2$\times$2 Max Pooling, Stride 2\\ \hline
        Linear Classifier & number of classes \\\hline
    \end{tabular}
    \caption{\footnotesize Architecture for few-shot learning. The number of classes is different for pretraining and finetuning.}\label{netarch_fewshot}
\end{table*}

\textbf{Graph neural networks.}
We implement the OPT training for GCN in the official repository\footnote{\url{https://github.com/tkipf/gcn}}. The experimental settings also follow the official repository to ensure a fair comparison. For OPT (CP) method, we use the original hyperparameters and experimental setup except the added rotation matrix. For OPT (OGD) method, we use our own OGD optimizer in Tensorflow to train the rotation matrix in order to maintain orthogonality and use the original optimizer to train the other variables.

Training a GCN with OPT is not that straightforward. Specifically, the forward model of GCN is $\thickmuskip=2mu \medmuskip=2mu\bm{Z}=\textnormal{Softmax}\big(\hat{\bm{A}}\cdot\textnormal{ReLU}(\hat{\bm{A}}\cdot\bm{X}\cdot\bm{W}_0)\cdot\bm{W}_1\big)$ where $\thickmuskip=2mu \medmuskip=2mu \hat{\bm{A}}=\tilde{\bm{D}}^{\frac{1}{2}}\tilde{\bm{A}}\tilde{\bm{D}}^{\frac{1}{2}}$. We note that $\bm{A}$ is the adjacency matrix of the graph,  $\thickmuskip=2mu \medmuskip=2mu \tilde{\bm{A}}=\bm{A}+\bm{I}$ ($\bm{I}$ is an identity matrix), and $\thickmuskip=2mu \medmuskip=2mu \tilde{\bm{D}}=\sum_j\tilde{\bm{A}}_{ij}$. $\bm{X}\in\mathbb{R}^{n\times d}$ is the feature matrix of $n$ nodes in the graph (feature dimension is $d$). $\bm{W}_1$ is the weights of the classifiers. $\bm{W}_1$ is the weights of the classifiers. $\bm{W}_0$ is the weight matrix of size $d\times h$ where $h$ is the dimension of the hidden space. We treat each column vector of $\bm{W}_0$ as a neuron, so there are $h$ neurons in total. Then we apply OPT to train these $h$ neurons of dimension $d$ in GCN. We conduct experiments on Cora and Pubmed datsets~\cite{sen2008collective}. We aim to verify the effectiveness of OPT on GCN instead of achieving state-of-the-art performance on this task.

\begin{table}[t]
    \renewcommand{\captionlabelfont}{\footnotesize}
    \newcommand{\tabincell}[2]{\begin{tabular}{@{}#1@{}}#2\end{tabular}}
    \centering
    \setlength{\abovecaptionskip}{7pt}
    \setlength{\belowcaptionskip}{-5pt}
    \scriptsize
    \begin{tabular}{|c|c|c|}
        \hline
        Layer &  Wide CNN-9 (CIFAR-100) & ResNet-18 (ImageNet-2012) \\
        \hline
        Conv0.x   & N/A & \tabincell{c}{[7$\times$7, 64], Stride 2\\3$\times$3, Max Pooling, Stride 2 } \\\hline
        Conv1.x   & \tabincell{c}{[3$\times$3, 64]$\times$3\\2$\times$2 Max Pooling, Stride 2} & $\left[\begin{aligned}&3\times 3, 64\\&3\times3, 64\end{aligned}\right]\times 2$ \\\hline
        Conv2.x   & \tabincell{c}{[3$\times$3, 128]$\times$3\\2$\times$2 Max Pooling, Stride 2} & $\left[\begin{aligned}&3\times 3, 128\\&3\times3, 128\end{aligned}\right]\times 2$ \\\hline
        Conv3.x  & \tabincell{c}{[3$\times$3, 256]$\times$3\\2$\times$2 Max Pooling, Stride 2} & $\left[\begin{aligned}&3\times 3, 256\\&3\times3, 256\end{aligned}\right]\times 2$\\\hline
        Conv4.x   & N/A & $\left[\begin{aligned}&3\times 3, 512\\&3\times3, 512\end{aligned}\right]\times 2$  \\\hline
        Final & 256-dim Fully Connected  & Average Pooling  \\\hline
    \end{tabular}
    \caption{\footnotesize Our wide CNN-9 and wide ResNet-18 architectures with different convolutional layers.}\label{sopt_arch}
\end{table}

\par
\textbf{Point cloud recognition.} To simplify the comparison and remove all the bells and whistles, we use a vanilla PointNet (without T-Net) as our backbone network. We apply OPT to train the MLPs in PointNet. We follow the same experimental settings as \cite{qi2017pointnet} and evaluate on the ModelNet-40 dataset~\cite{wu20153d}. We exactly follow the same setting in the original paper~\cite{qi2017pointnet} and the official repositories\footnote{\url{https://github.com/charlesq34/pointnet}}. Specifically, we multiply the rotation matrix to the original fixed neurons in all the $\thickmuskip=2mu \medmuskip=2mu 1\times1$ convolution layers and the fully connected layer except the final classifier. All the rotation matrix is initialized with random normal distribution. For the experiments, we use point number 1024, batch size 32 and Adam optimizer with initial learning rate 0.001. The learning rate will decay by 0.7 every 200k iterations, and the training is terminated at 250 epochs.
\par
\textbf{Experimental settings for S-OPT}. For the experiment of S-OPT, the architecture of wide CNN-9 and wide ResNet-18 is given in Table~\ref{sopt_arch}. CNN-6 is the same as the one in Table~\ref{netarch}. We use standard data augmentation for CIFAR-100, following \cite{liu2017deep}. For ImageNet-2012, we use the same data augmentation in \cite{krizhevsky2012imagenet,liu2017deep}. This data augmentation does not contain as many transformation as the one in \cite{he2016deep}, so the final performance may be worse than \cite{he2016deep}. However, all the compared methods use the same data augmentation in our experiments, so the experiment is still a fair comparison. For CIFAR-100, we use $N_{\textnormal{out}}=300$ and $N_{\textnormal{in}}=750$. For ImageNet, we use $N_{\textnormal{out}}=700$ and $N_{\textnormal{in}}=1000$. For S-OPT, we directly use the original OPT for the first layer, as its neuron dimension is typically very small. We decrease the learning rate by a factor of 10 when the performance is saturated in the outer iteration.

\begin{table}[t]
    \renewcommand{\captionlabelfont}{\footnotesize}
    \newcommand{\tabincell}[2]{\begin{tabular}{@{}#1@{}}#2\end{tabular}}
    \centering
    \setlength{\abovecaptionskip}{7pt}
    \setlength{\belowcaptionskip}{-5pt}
    \scriptsize
    \begin{tabular}{|c|c|c|c|c|}
        \hline
        Layer & MNIST Visualization &ResNet-18A & ResNet-18B & CNN-20 for Face Embeddings \\
        \hline
        Conv0.x  & N/A  & \tabincell{c}{[7$\times$7, 64], Stride 2\\3$\times$3, Max Pooling, Stride 2 } & \tabincell{c}{[7$\times$7, 64], Stride 2\\3$\times$3, Max Pooling, Stride 2 } & N/A \\\hline
        Conv1.x  & \tabincell{c}{[3$\times$3, 32]$\times$ 2\\3$\times$3, Max Pooling, Stride 2 } & $\left[\begin{aligned}&3\times 3, 64\\&3\times3, 64\end{aligned}\right]\times 2$ & $\left[\begin{aligned}&3\times 3, 64\\&3\times3, 64\end{aligned}\right]\times 2$ & \tabincell{c}{[3$\times$3, 64]$\times$1, Stride 2\\$\left[\begin{aligned}&3\times3, 64\\&3\times3, 64\end{aligned}\right]\times 1$}\\\hline
        Conv2.x  & \tabincell{c}{[3$\times$3, 64]$\times$ 2\\3$\times$3, Max Pooling, Stride 2 } & $\left[\begin{aligned}&3\times 3, 128\\&3\times3, 128\end{aligned}\right]\times 2$ & $\left[\begin{aligned}&3\times 3, 128\\&3\times3, 128\end{aligned}\right]\times 2$ & \tabincell{c}{[3$\times$3, 128]$\times$1, Stride 2\\$\left[\begin{aligned}&3\times3, 128\\&3\times3, 128\end{aligned}\right]\times 2$} \\\hline
        Conv3.x & \tabincell{c}{[3$\times$3, 128]$\times$ 2\\3$\times$3, Max Pooling, Stride 2 } & $\left[\begin{aligned}&3\times 3, 128\\&3\times3, 128\end{aligned}\right]\times 2$ & $\left[\begin{aligned}&3\times 3, 256\\&3\times3, 256\end{aligned}\right]\times 2$ & \tabincell{c}{[3$\times$3, 256]$\times$1, Stride 2\\$\left[\begin{aligned}&3\times3, 256\\&3\times3, 256\end{aligned}\right]\times 4$}\\\hline
        Conv4.x  & N/A & $\left[\begin{aligned}&3\times 3, 128\\&3\times3, 128\end{aligned}\right]\times 2$ & $\left[\begin{aligned}&3\times 3, 512\\&3\times3, 512\end{aligned}\right]\times 2$  & \tabincell{c}{[3$\times$3, 512]$\times$1, Stride 2\\$\left[\begin{aligned}&3\times3, 512\\&3\times3, 512\end{aligned}\right]\times 1$}\\\hline
        Final & Fully Connected (3-dim) & Average Pooling (128-dim) & Average Pooling (512-dim) & Fully Connected (512-dim) \\\hline
    \end{tabular}
    \caption{\footnotesize Our network architectures for large categorical training.}\label{cls_opt_arch}
\end{table}

\textbf{Experimental settings for large categorical training}. All the network architectures used in the large categorical training are specified in Table~\ref{cls_opt_arch}. For the visualization on MNIST, we simply set the output dimension as 3 and directly plot the 3-dimensional features. In Fig.~\ref{cls_opt_vis}, each color denotes a class of digits, and each dot point denotes 3-dimensional features for a digit image. The experiments on ImageNet follows the same setting as the previous section. For the open-set face recognition experiments, we generally follow the same training configuration as SphereFace~\cite{Liu2017CVPR}. For all the methods used in face recognition, we use the 20-layer residual network as described in Table~\ref{cls_opt_arch}. Since OPT is originally implemented in TensorFlow, we re-implement the CNN-20 for deep face recognition in TensorFlow, which yields an accuracy gap compared to \cite{Liu2017CVPR}. This is due to some mis-match in data augmentation and optimizations. However, since both the baseline and our CLS-OPT use the same network implementation in TensorFlow and achieving state-of-the-art results is not our major focus, it is still a fair and valid comparison. We expect CLS-OPT can also be generally useful for large categorical training of deep face recognition.

\newpage
\section{Loss Landscape Visualization (Normal Distribution Perturbation)}\label{loss_vis_normal}

\subsection{Visualization Procedure}
We generally follow the visualization procedure in \cite{li2018visualizing}. However, since OPT has a different training process, we use a modified visualization method but still make it comparable to the baseline.

Specifically, if we want to plot the loss landscape of a neural neural with loss $\mathcal{L}(\bm{\theta})$ where $\bm{\theta}$ is the learnable model parameters, we need to first choose pretrained model parameters $\bm{\theta}^*$ as a center point. Then we choose two random direction vectors $\bm{\delta}$ and $\bm{\eta}$. The 2D plot $f(\alpha,\beta)$ is defined as
\begin{equation}
    f(\alpha,\beta)=\mathcal{L}(\bm{\theta}^*+\alpha\bm{\delta}+\beta\bm{\eta})
\end{equation}
which can be used as a 2D surface visualization. Note that, after we randomly initialize the direction vectors $\bm{\delta}$ and $\bm{\eta}$ (with normal distribution), we need to perform the filter normalization~\cite{li2018visualizing}. Specifically, we normalize each filter in $\bm{\delta}$ and $\bm{\eta}$ to have the same norm as the corresponding filter in $\bm{\theta}^*$. The loss landscape of our baseline is plotted using this visualization approach.
\par
In contrast, the learnable parameters in OPT are no longer the weights of neurons. Instead, the learnable parameters are the orthogonal matrices. More precisely, the trainable matrices are used to perform orthogonalization in the neural networks (\ie, $\bm{P}$ in Fig.~\ref{ortho}). We denote the combination of all the trainable matrices as $\tilde{\bm{R}}$, and the corresponding pretrained matrices as $\tilde{\bm{R}}^*$. Then the 2D visualization of OPT is 
\begin{equation}
    f(\alpha,\beta)=\mathcal{L}(\tilde{\bm{R}}^*+\alpha\bm{\gamma}+\beta\bm{\kappa})
\end{equation}
where $\bm{\gamma}$ and $\bm{\kappa}$ are two random direction vectors (which follow the normal distribution) to perturb $\tilde{\bm{R}}^*$. The visualization procedures of baseline and OPT are essentially the same except that the trainable variables are different. Therefore, their loss landscapes are comparable.

\subsection{Experimental Details}\label{exp_setting_fig3_normal}
In Fig.~\ref{loss_landscape}, we vary $\alpha$ and $\beta$ from $-1.5$ to $1.5$ for both baseline and OPT, and then plot the surface of 2D function $f$. We use the CNN-6 (as specified in Appendix~\ref{exp_set}) on CIFAR-100. We use the same data augmentation as \cite{liu2018learning}. We train the network with SGD with momentum $0.9$ and batch size $128$. We start with learning rate $0.1$, divide it by 10 at 30k, 50k and 64k iterations, and terminate training at 75k iterations. The training details basically follows \cite{liu2018learning}. We mostly use CP for OPT due to efficiency. Note that, the other orthogonalization methods in OPT yields similar loss landscapes in general. The pretrained model for standard training yields $37.59\%$ testing error on CIFAR-100, while the pretrained model for OPT yields $33.53\%$ error. This is also reported in Section~\ref{exp_results}. 

\subsection{Full Visualization Results for the Main Paper}
\begin{figure}[h]
  \centering
  \renewcommand{\captionlabelfont}{\footnotesize}
  \includegraphics[width=5in]{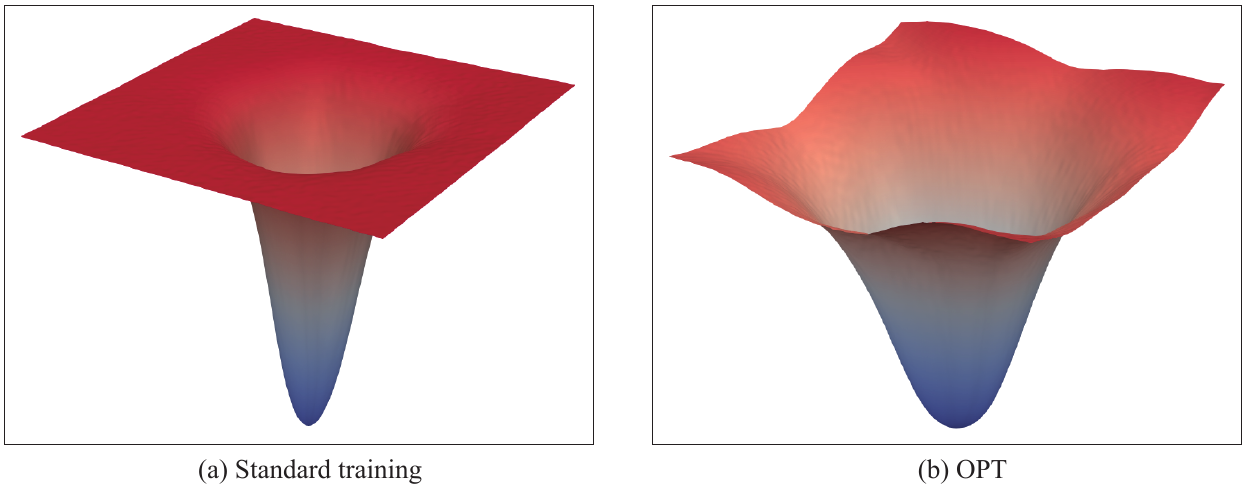}
  \caption{\footnotesize High-quality rendered loss landscapes of standard training and OPT.}\label{hq_loss}
\end{figure}

Following the same experimental settings in Appendix~\ref{exp_setting_fig3_normal}, we render the 3D loss landscapes with some color and lighting effects for Fig.~\ref{hq_loss}. The visualization data is exactly the same as Fig.~\ref{loss_landscape}, and we simply use ParaView to plot the figure. The rendered loss landsacpe better reflects that OPT yields a much more smooth loss geometry.

We also give the large and full version of Fig.~\ref{loss_landscape}(b) (in the main paper) in the following figure. Fig.~\ref{loss_landscape_full_normal} is identical to Fig.~\ref{loss_landscape}(b) in the main paper except that Fig.~\ref{loss_landscape_full_normal} has larger size. From Fig.~\ref{loss_landscape_full_normal}, we can better observe the dramatically different loss landscape between standard training and OPT. From the contour plots, we can better see that the red region of standard training is extremely flat and is highly non-convex around the edge. In comparison to standard training, OPT has very smooth and relatively convex contour shape.

\begin{figure}[h]
  \centering
  \renewcommand{\captionlabelfont}{\footnotesize}
  \includegraphics[width=5in]{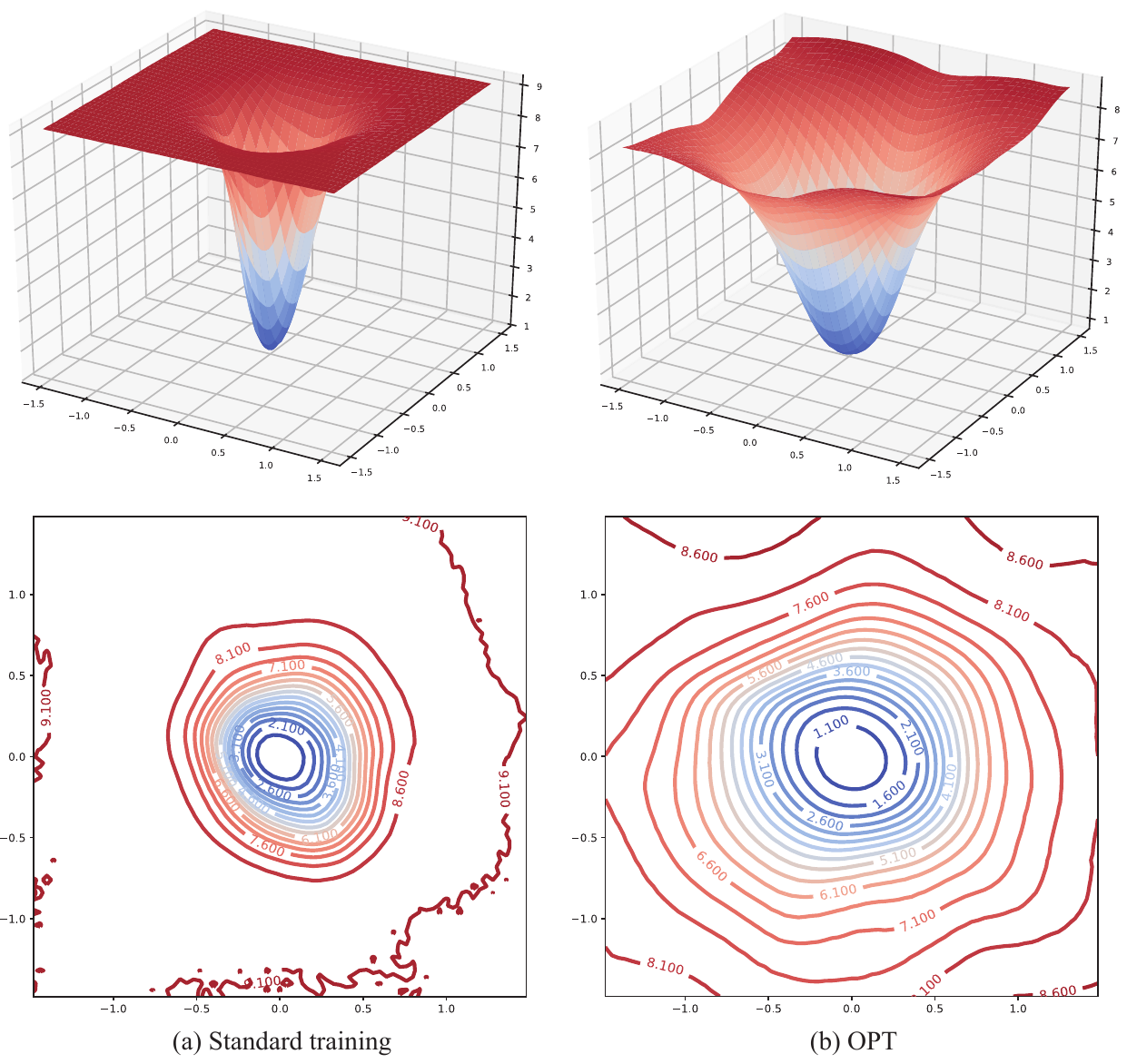}
  \caption{\footnotesize Comparison of loss landscapes between standard training and OPT (full results of Fig.~\ref{loss_landscape}(b) in the main paper). Top row: loss landscape visualization with Cartesian coordinate system; Bottom row: loss contour visualization.}\label{loss_landscape_full_normal}
  \vspace{3mm}
\end{figure}

Then we visualize the testing error landscape. Additionally, we include a high-quality visualization using ParaView. We render the plot with lighting and color effects in order to better demonstrate the difference between OPT and standard training. The visualization results are given in Fig.~\ref{cnn6_test_err_normal}. The comparison of testing error landscape shows that the parameter space of OPT is more robust than standard training, because the testing error of OPT is increased in a slower speed while the model is perturbed away from the pretrained parameters. In other words, the parameter space of OPT is more robust to the random perturbation than standard training. Combining the visualization of loss landscape, it is well justified that OPT can significantly alleviate training difficulty and improve generalization.

\ 

\ 

\ 

\ 

\

\ 

\ 

\ 

\ 

\ 

\ 

\ 

\ 

\

\ 

\ 

\ 

\ 

\begin{figure}[t]
  \centering
  \renewcommand{\captionlabelfont}{\footnotesize}
  \includegraphics[width=5in]{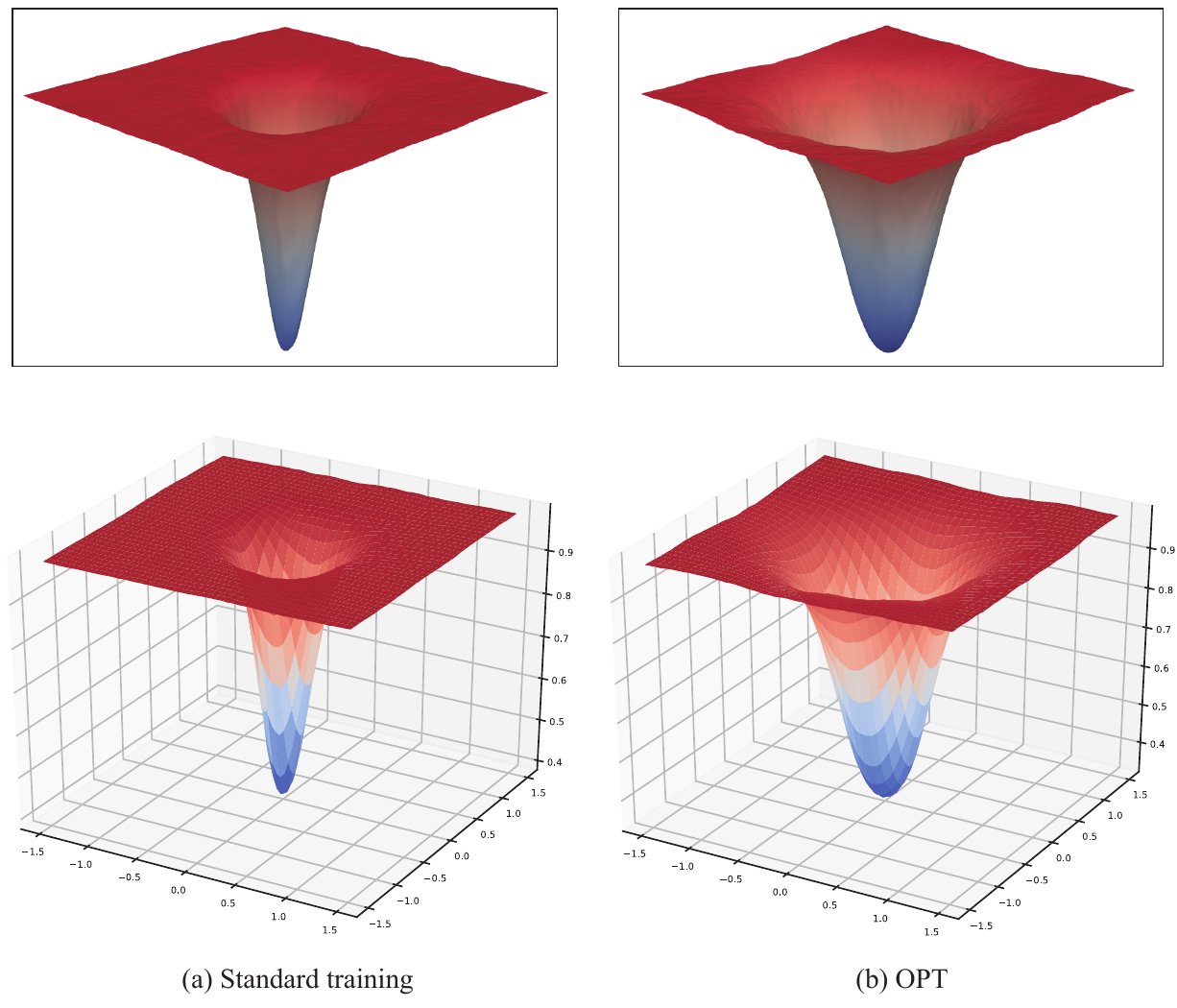}
  \caption{\footnotesize Comparison of testing error landscapes between standard training and OPT. Top row: high-quality rendered testing error landscape visualization with lighting effects; Bottom row: testing error landscape visualization with Cartesian coordinate system.}\label{cnn6_test_err_normal}.
\end{figure}

\newpage

\subsection{Loss Landscape Visualization of Different Neural Network}
In order to show that the loss landscape of OPT is quite general and consistent across different neural network architectures, we also visualize the loss landscape using a deep network network (CNN-9 as specified in Appendix~\ref{exp_set}). The experiments are conducted on CIFAR-100. The results are given in Fig.~\ref{hq_loss_cnn9_normal}. We can see that the loss landscape of OPT is much more smooth than standard training, similar to the observation for CNN-6. Therefore, the loss landscape difference between OPT and standard training is consistent across different network architectures, and OPT consistently shows better optimization landscape.

\begin{figure}[h]
  \centering
  \renewcommand{\captionlabelfont}{\footnotesize}
  \includegraphics[width=5in]{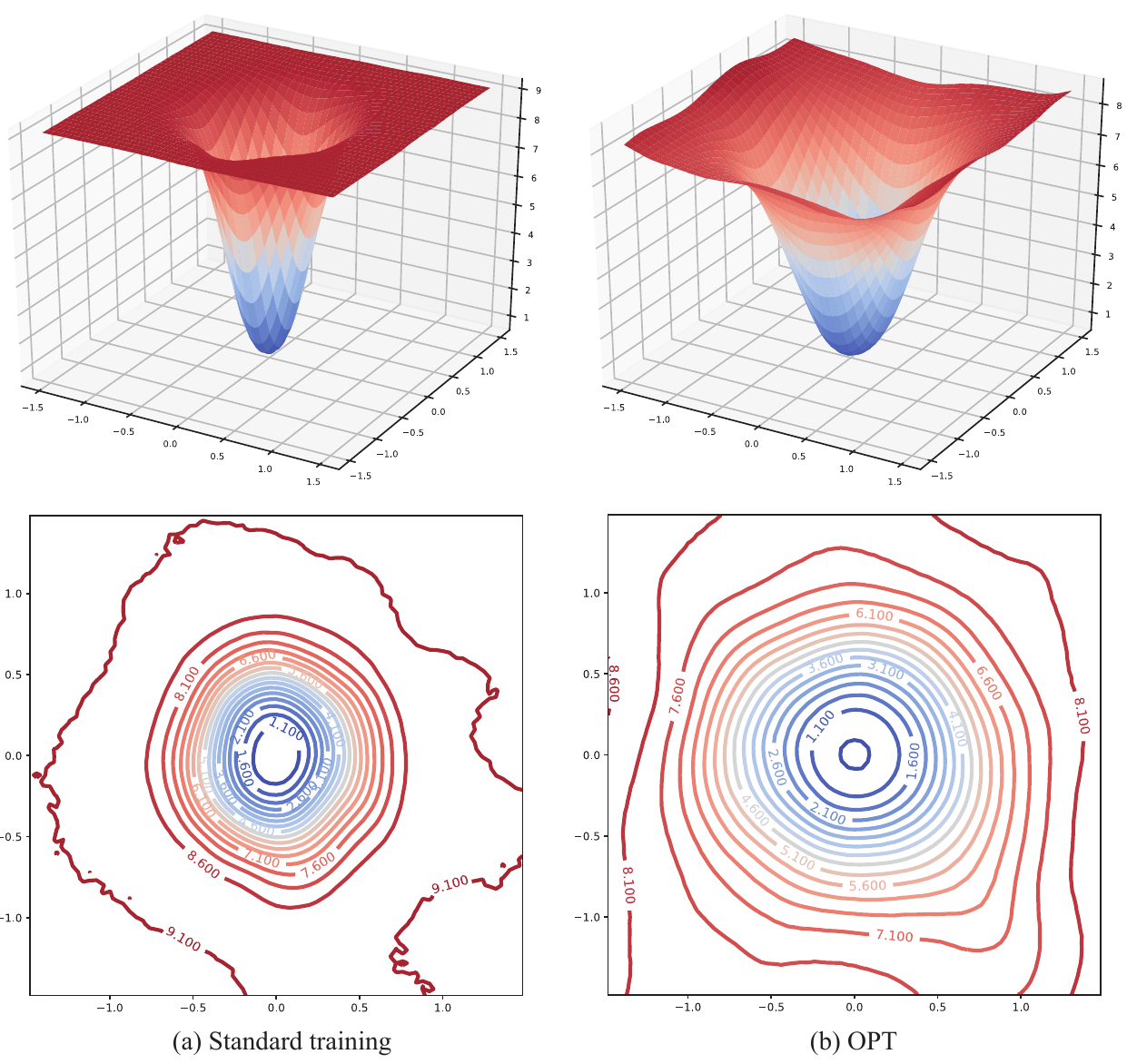}
  \caption{\footnotesize Comparison of loss landscapes between standard training and OPT on CIFAR-100 (CNN-9). Top row: loss landscape visualization with Cartesian coordinate system; Bottom row: loss contour visualization.}\label{hq_loss_cnn9_normal}
  \vspace{4mm}
\end{figure}

Then we show the landscape of testing error for CNN-9 on CIFAR-100. The results are given in Fig.~\ref{cnn9_test_err_normal}. Similar to CNN-6, the testing error landscape of OPT is more smooth and convex than standard training. Moreover, OPT has a more flat local minima of testing error, while standard training has a sharp local minima. The testing error landscape in Fig.~\ref{cnn9_test_err_normal} generally follows the same pattern as the loss landscape in Fig.~\ref{hq_loss_cnn9_normal}. The visualization further verifies the superiority of OPT is very consistent across different network architectures.

\begin{figure}[t]
  \centering
  \renewcommand{\captionlabelfont}{\footnotesize}
  \includegraphics[width=5in]{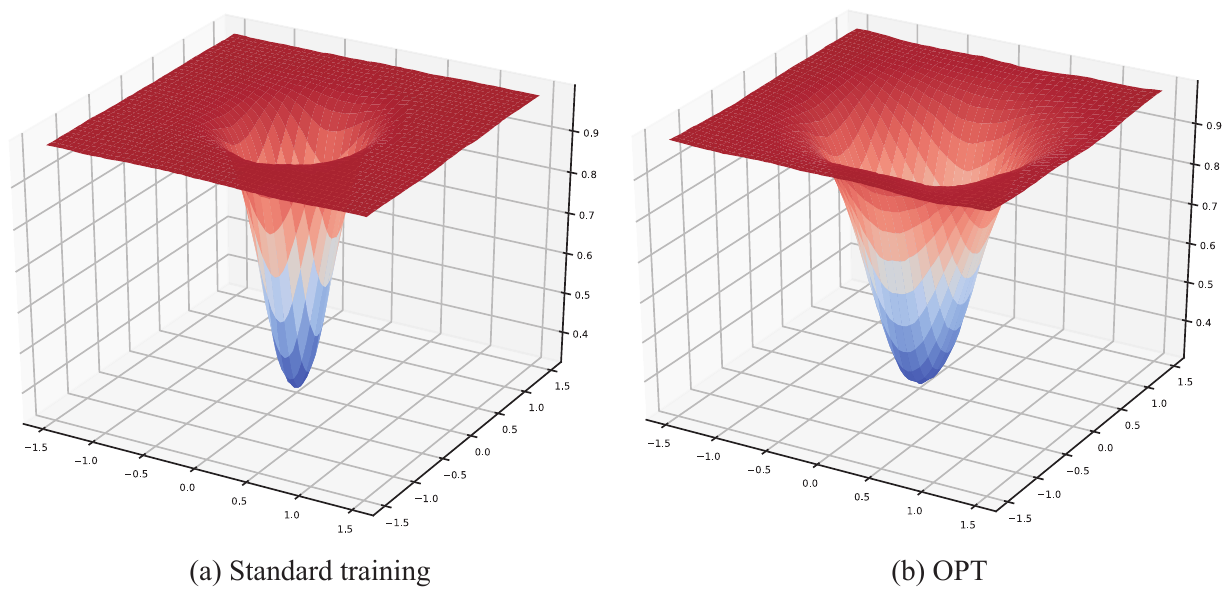}
  \caption{\footnotesize Comparison of testing error landscapes between standard training and OPT on CIFAR-100 (CNN-9).}\label{cnn9_test_err_normal}
\end{figure}

\ 

\ 

\ 

\ 

\ 

\ 

\ 

\ 

\ 

\ 

\ 

\ 

\newpage

\subsection{Loss Landscape Visualization of Different Dataset}

Similar to Appendix~\ref{uniform_cifar10}, we also visualize the loss landscape of OPT and standard training on a different dataset (CIFAR-10) with CNN-6. The loss landscape visualization is given in Fig.~\ref{cifar10_loss_normal}. Although the loss landscape on CIFAR-10 is quite different from the one on CIFAR-100, we can still observe that the loss landscape of OPT has a very flat local minima and the loss values are increasing smoothly and slowly. In contrast, the loss landscape of standard training has a sharp local minima and the loss values quickly increase to a large value. The red region of standard training will lead to very small gradient, potentially affecting the training. From the contour plots, the comparison apparently shows that the loss landscape of OPT is much more smooth than standard training.

\begin{figure}[h]
  \centering
  \renewcommand{\captionlabelfont}{\footnotesize}
  \vspace{1mm}
  \includegraphics[width=5in]{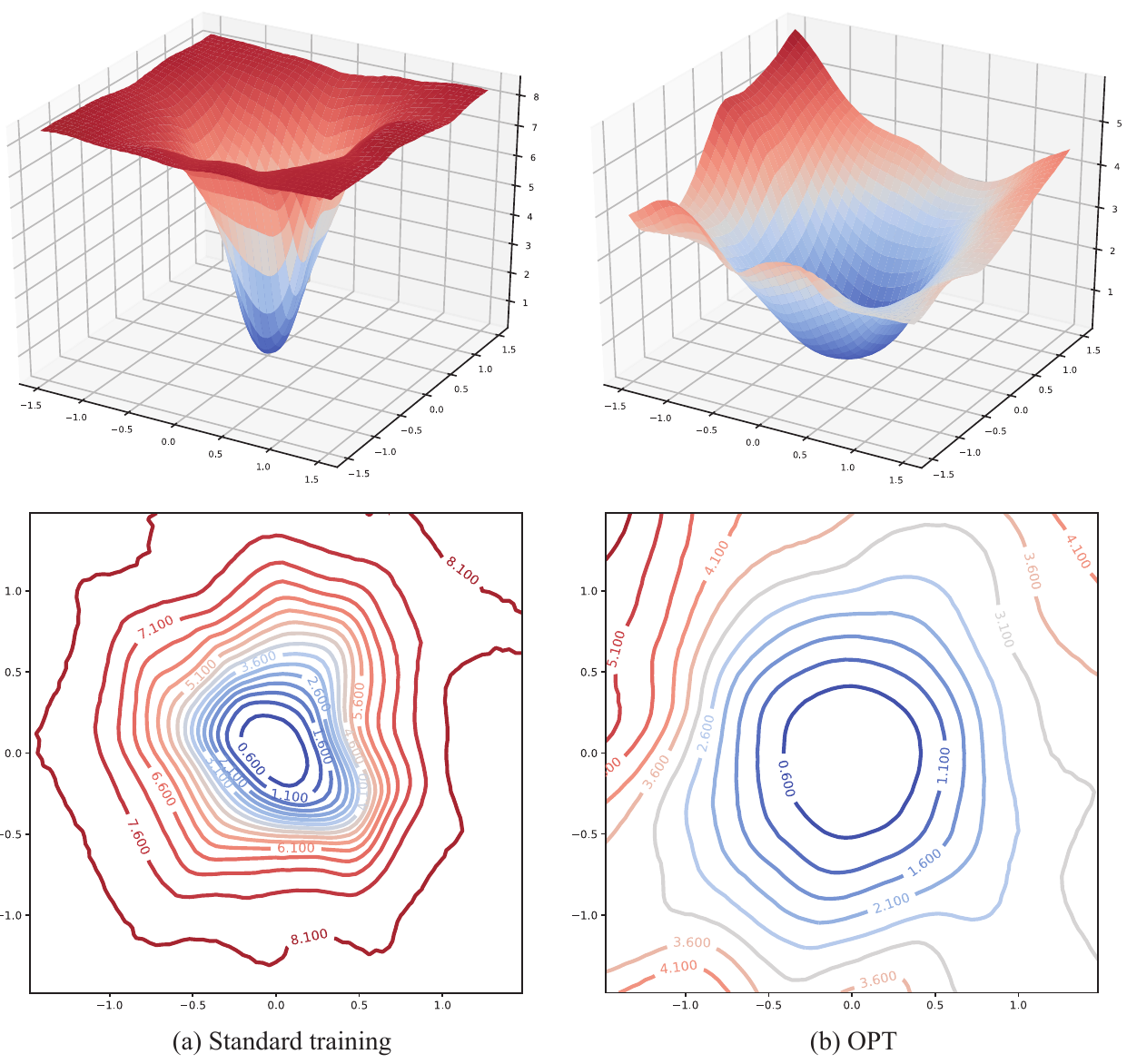}
  \caption{\footnotesize Comparison of loss landscapes between standard training and OPT on CIFAR-10 (CNN-6). Top row: loss landscape visualization with Cartesian coordinate system; Bottom row: loss contour visualization.}\label{cifar10_loss_normal}
  \vspace{3mm}
\end{figure}

We also visualize the landscape of testing error in Fig.~\ref{cifar10_test_err_normal}. The testing error landscape generally follows the pattern in the loss landscape. One can easily observe that the parameter space of standard training is very sensitive to perturbations. A small perturbation can make the model parameters completely fail (\ie, the testing error dramatically increase). Differently, the parameter space of OPT is much more robust to perturbations. The model parameter can still work well with a small perturbation. Both Fig.~\ref{cifar10_loss_normal} and Fig.~\ref{cifar10_test_err_normal} validate that the superiority of OPT is consistent across different training datasets.

\ 

\ 

\ 

\ 

\ 

\ 

\ 

\ 

\ 

\ 

\begin{figure}[t]
  \centering
  \renewcommand{\captionlabelfont}{\footnotesize}
  \includegraphics[width=5in]{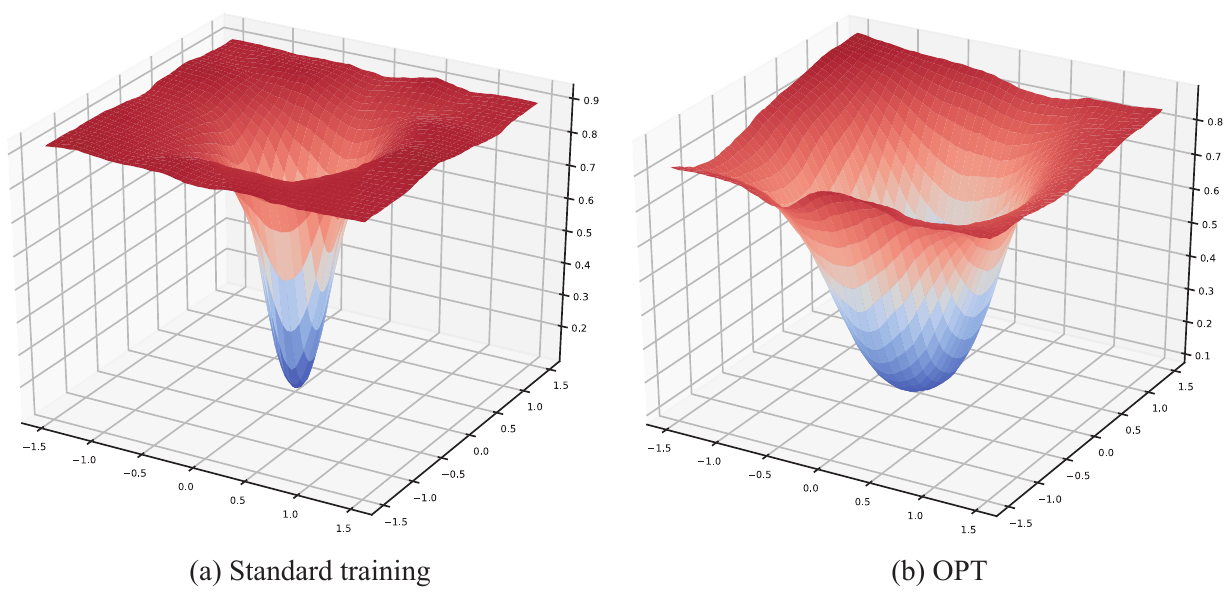}
  \caption{\footnotesize Comparison of testing error landscapes between standard training and OPT on CIFAR-10 (CNN-6).}\label{cifar10_test_err_normal}
\end{figure}

\newpage
\section{Loss Landscape Visualization (Uniform Distribution Perturbation)}\label{loss_vis}

\subsection{Visualization Procedure}
Different from Appendix~\ref{loss_vis_normal}, we choose two random direction vectors $\bm{\delta}$ and $\bm{\eta}$ based on $[0,1]$ uniform distribution to further justify the effectiveness of OPT. The 2D plot $f(\alpha,\beta)$ is defined as
\begin{equation}
    f(\alpha,\beta)=\mathcal{L}(\bm{\theta}^*+\alpha\bm{\delta}+\beta\bm{\eta})
\end{equation}
which can be used as a 2D surface visualization. Note that, after we randomly initialize the direction vectors $\bm{\delta}$ and $\bm{\eta}$ with $[0,1]$ normal distribution, we need to perform the filter normalization~\cite{li2018visualizing}. Specifically, we normalize each filter in $\bm{\delta}$ and $\bm{\eta}$ to have the same norm as the corresponding filter in $\bm{\theta}^*$. The loss landscape of our baseline is plotted using this visualization approach.
\par
In contrast, the learnable parameters in OPT are no longer the weights of neurons. Instead, the learnable parameters are the orthogonal matrices. More precisely, the trainable matrices are used to perform orthogonalization in the neural networks (\ie, $\bm{P}$ in Fig.~\ref{ortho}). We denote the combination of all the trainable matrices as $\tilde{\bm{R}}$, and the corresponding pretrained matrices as $\tilde{\bm{R}}^*$. Then the 2D visualization of OPT is 
\begin{equation}
    f(\alpha,\beta)=\mathcal{L}(\tilde{\bm{R}}^*+\alpha\bm{\gamma}+\beta\bm{\kappa})
\end{equation}
where $\bm{\gamma}$ and $\bm{\kappa}$ are two random direction vectors (which follow the $[0,1]$ uniform distribution) to perturb $\tilde{\bm{R}}^*$. The visualization procedures of baseline and OPT are essentially the same except that the trainable variables are different. Therefore, their loss landscapes are comparable.

\subsection{Experimental Details}\label{exp_setting_fig3}
In Fig.~\ref{loss_landscape}, we vary $\alpha$ and $\beta$ from $-1$ to $1$ for both baseline and OPT, and then plot the surface of 2D function $f$. We use the CNN-6 (as specified in Appendix~\ref{exp_set}) on CIFAR-100. We use the same data augmentation as \cite{liu2018learning}. We train the network with SGD with momentum $0.9$ and batch size $128$. We start with learning rate $0.1$, divide it by 10 at 30k, 50k and 64k iterations, and terminate training at 75k iterations. The training details basically follows \cite{liu2018learning}. We mostly use CP for OPT due to efficiency. Note that, the other orthogonalization methods in OPT yields similar loss landscapes in general. The pretrained model for standard training yields $37.59\%$ testing error on CIFAR-100, while the pretrained model for OPT yields $33.53\%$ error. This is also reported in Section~\ref{exp_results}.

\subsection{Full Visualization Results}
\begin{figure}[h]
  \centering
  \renewcommand{\captionlabelfont}{\footnotesize}
  \includegraphics[width=5in]{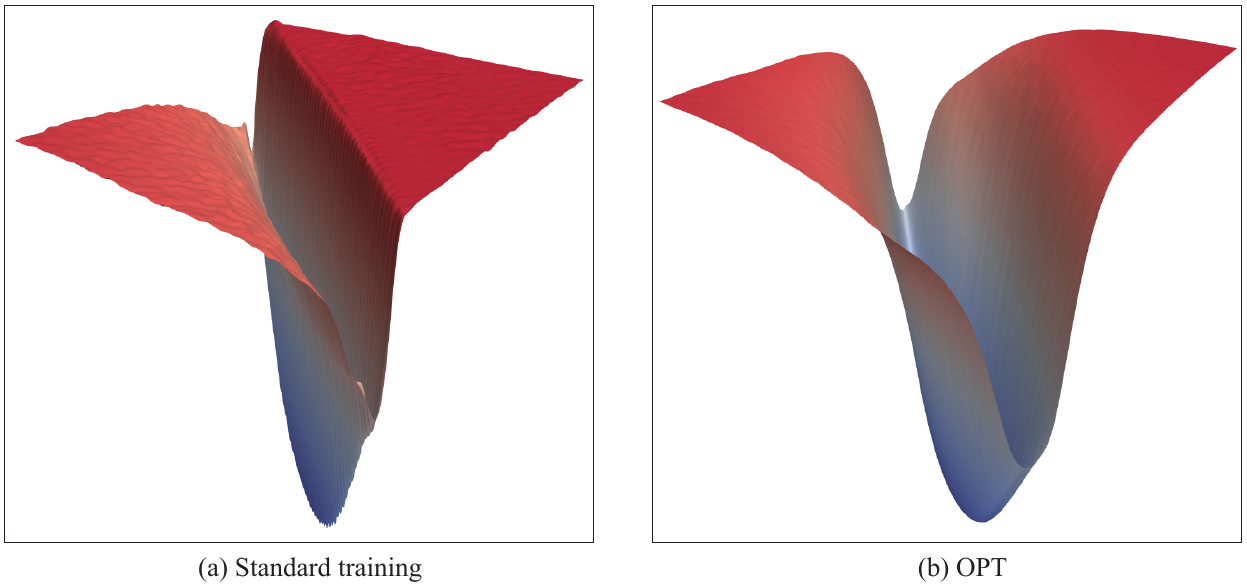}
  \caption{\footnotesize High-quality rendered loss landscapes of standard training and OPT.}\label{hq_loss_uniform}
\end{figure}
Following the same experimental settings in Appendix~\ref{exp_setting_fig3}, we render the 3D loss landscapes with some color and lighting effects for Fig.~\ref{hq_loss_uniform}. We first use ParaView to plot a high-qualify loss landscape comparison between standard training and OPT. As expected, the loss landscape of OPT is much more smooth than standard training. Note that, for the flat red region in standard training, we can still observe numerous small local minima, while the red region of OPT is very smooth. Fig.~\ref{hq_loss_uniform} better validates our analysis and discussion in Section~\ref{main_landscape}, and also shows the superiority of OPT. 

We provide the visualization results in the rest of the subsection. From Fig.~\ref{loss_landscape_full}, we can better observe the dramatically different loss landscape between standard training and OPT.

\begin{figure}[h]
  \centering
  \renewcommand{\captionlabelfont}{\footnotesize}
  \includegraphics[width=5in]{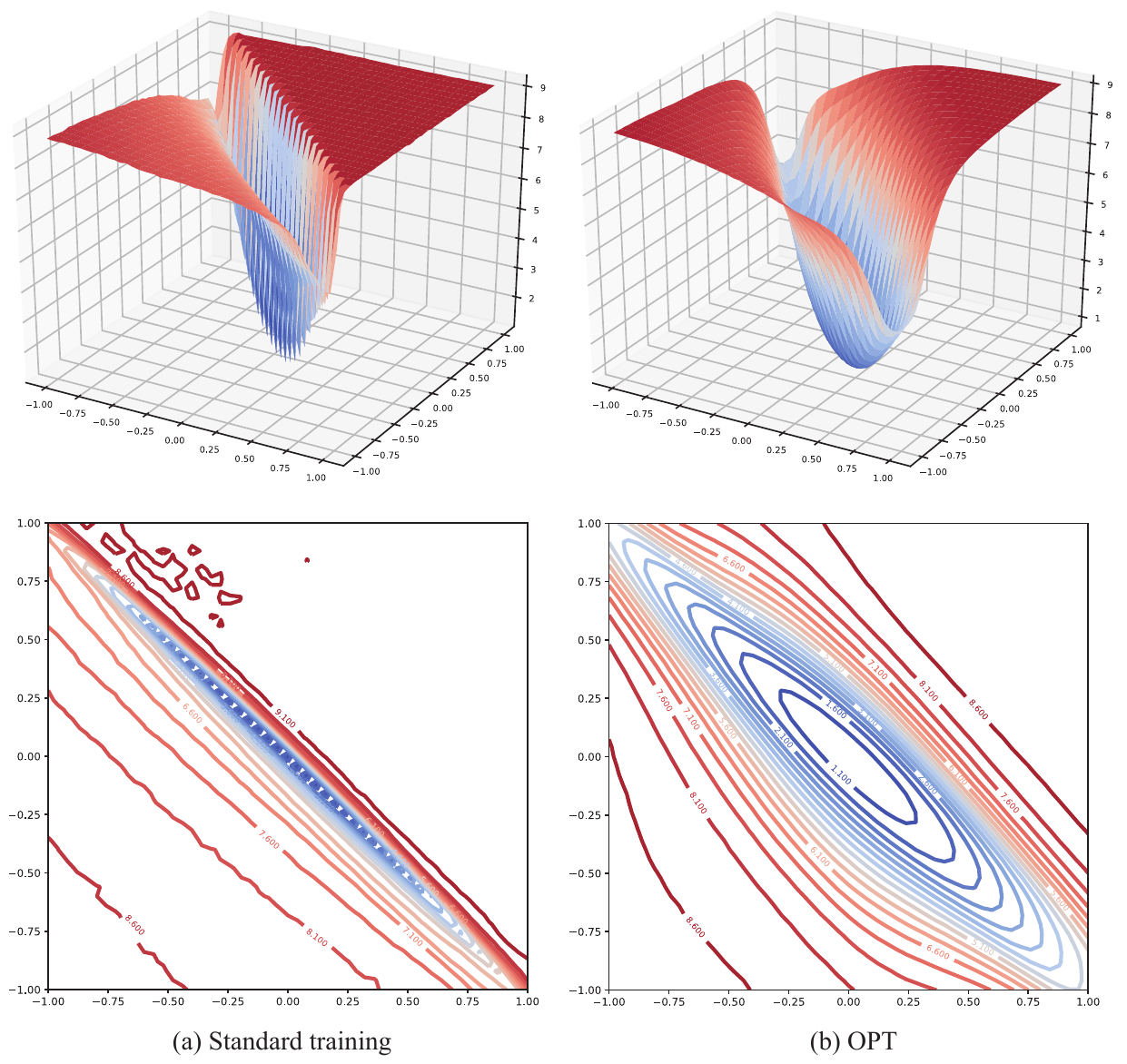}
  \caption{\footnotesize Comparison of loss landscapes between standard training and OPT (full results of Fig.~\ref{loss_landscape}(a) in the main paper). Top row: loss landscape visualization with Cartesian coordinate system; Bottom row: loss contour visualization.}\label{loss_landscape_full}
\end{figure}

To better understand the difference of the training dynamics between standard training and OPT, we also plot the testing error landscapes in Fig.~\ref{cnn6_test_err} for both methods. The testing error is computed on the testing set of CIFAR-100 with the perturbed pretrained model ($\alpha$ and $\beta$ are the perturbation parameters). From the testing error landscape comparison in Fig.~\ref{cnn6_test_err}, we can see that once the baseline pretrained model is slightly perturbed, the testing error will immediately increase to $99.99\%$ which is random selection-level testing error (because we have 100 balanced classes in total, randomly picking a class leads to 0.01\% accuracy). In contrast, the testing error landscape of OPT is much more smooth. Even if we perturb the OPT pretrained model, we still end up with a reasonably low testing error, show that the parameter space of OPT is more smooth and continuous. All these evidences suggest that OPT is a better training framework for neural networks and can significantly alleviate the optimization difficulty. In this following subsections, we aim to show that the loss and testing error landscape difference between standard training and OPT is not a coincidence. We will show that the improvement of OPT on the loss and testing error landscape is both dataset-agnostic and architecture-agnostic. 

\begin{figure}[h]
  \centering
  \renewcommand{\captionlabelfont}{\footnotesize}
  \includegraphics[width=5in]{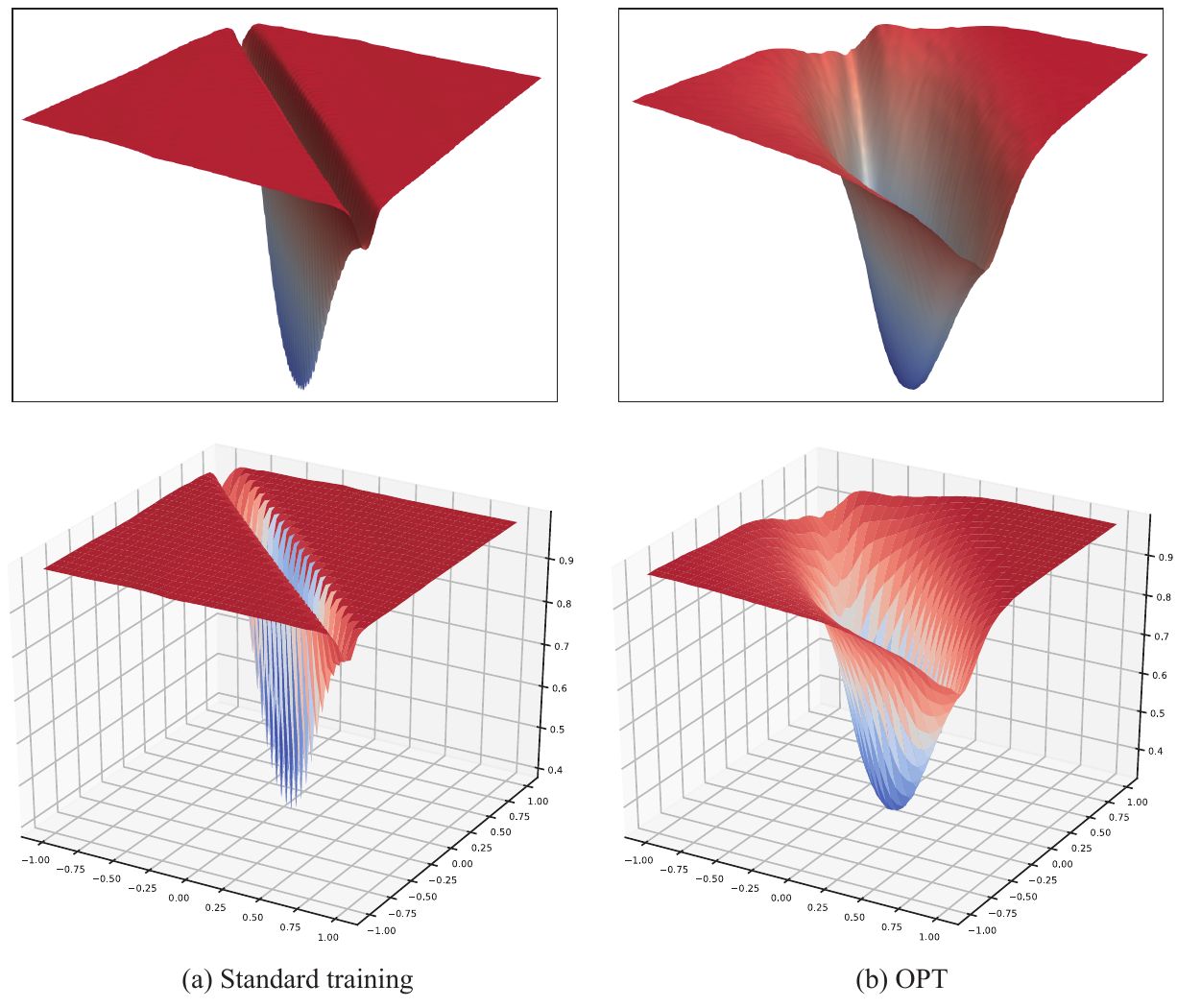}
  \caption{\footnotesize Comparison of testing error landscapes between standard training and OPT. Top row: high-quality rendered testing error landscape visualization with lighting effects; Bottom row: testing error landscape visualization with Cartesian coordinate system.}\label{cnn6_test_err}
\end{figure}

\ 

\ 

\ 

\ 
\ 

\ 

\ 

\ 

\ 

\ 

\ 

\ 

\

\ 

\ 

\ 

\ 

\newpage

\subsection{Loss Landscape Visualization of Different Neural Network}

\begin{figure}[h]
  \centering
  \renewcommand{\captionlabelfont}{\footnotesize}
  \vspace{-3mm}
  \includegraphics[width=5in]{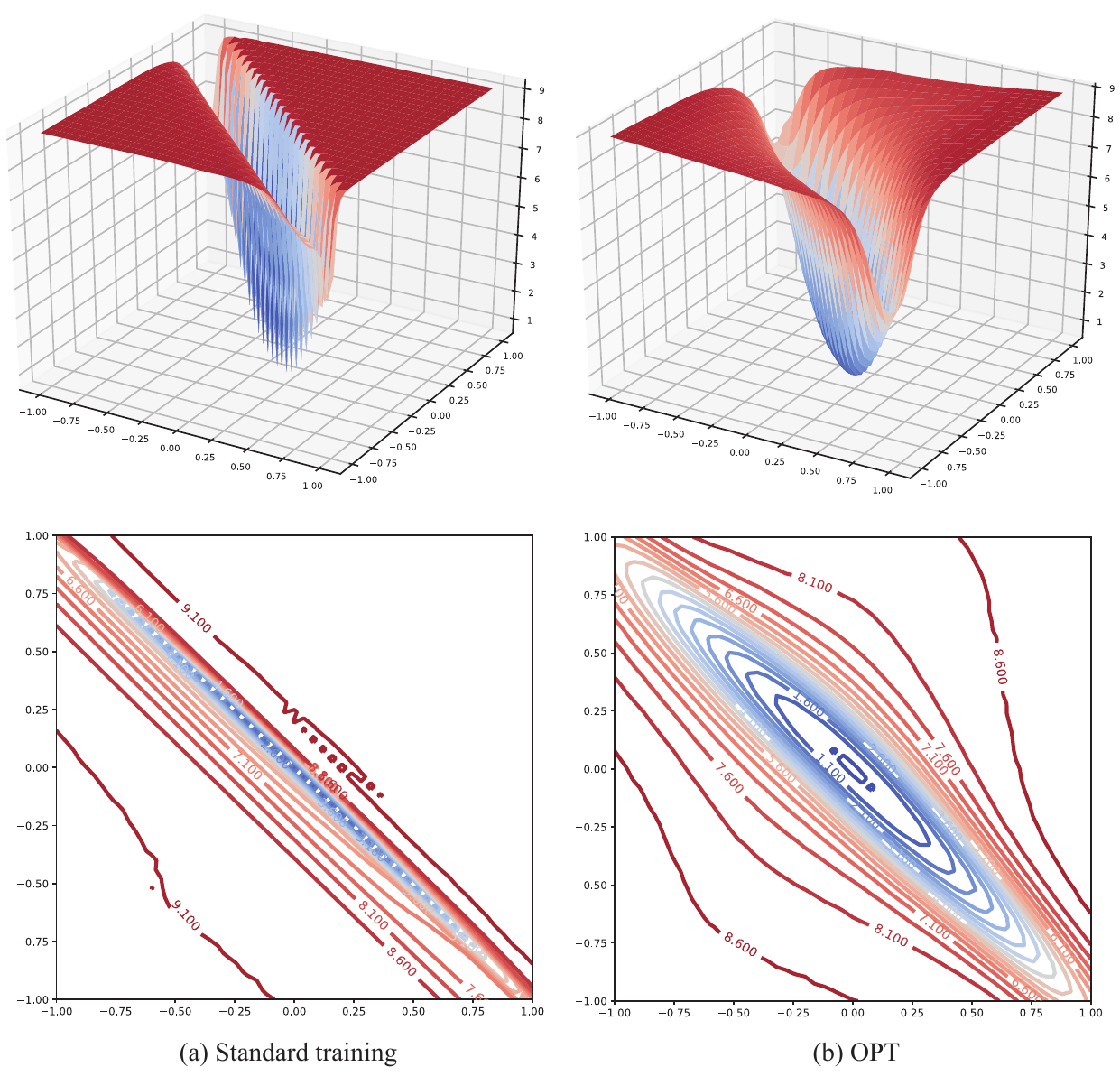}
  \vspace{-1.2mm}
  \caption{\footnotesize Comparison of loss landscapes between standard training and OPT on CIFAR-100 (CNN-9). Top row: loss landscape visualization with Cartesian coordinate system; Bottom row: loss contour visualization.}\label{hq_loss_cnn9}
  \vspace{-2.9mm}
\end{figure}

\begin{figure}[h]
  \centering
  \renewcommand{\captionlabelfont}{\footnotesize}
  \vspace{-1mm}
  \includegraphics[width=5in]{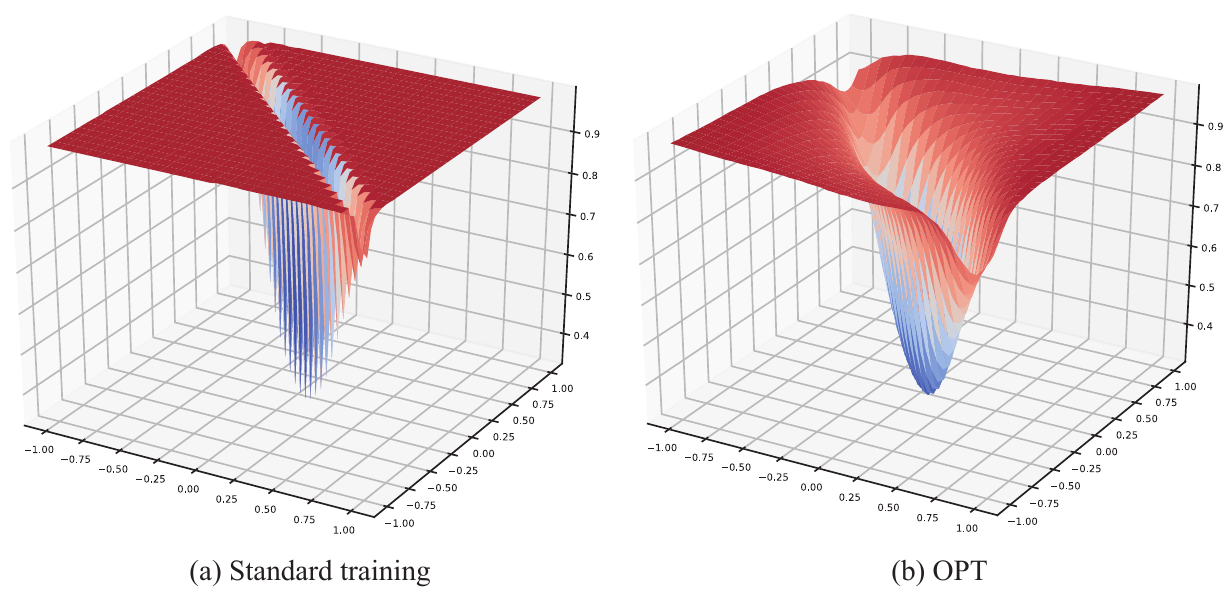}
  \vspace{-1.5mm}
  \caption{\footnotesize Comparison of testing error landscapes between standard training and OPT on CIFAR-100 (CNN-9).}\label{cnn9_test_err}
  \vspace{-1mm}
\end{figure}

To show that the difference of the loss landscape between standard training and OPT is consistent across different neural networks. We use a deeper CNN-9 (as specified in Appendix~\ref{exp_set}) to visualize the loss landscape. The experimental settings generally follow Appendix~\ref{exp_setting_fig3}. We use CNN-9 as our backbone architecture and train it on CIFAR-100. The visualization of the loss landscapes is given in Fig.~\ref{hq_loss_cnn9}. We observe that the loss landscapes of OPT with CNN-9 and CNN-6 are very similar. In general, the conclusion that OPT yields better loss landscape still holds for a deeper neural network, showing that the effectiveness of OPT is not architecture-dependent.

Moreover, we also visualize the landscape of testing error in Fig.~\ref{cnn9_test_err}. We can observe that the testing error landscapes are somewhat similar to the loss landscape in Fig.~\ref{hq_loss_cnn9}. The results further validate the superiority of OPT. We can observe in Fig.~\ref{cnn9_test_err} that OPT has more smooth testing error landscape and can make the training parameter space of the neural network less sensitive to perturbations.

\subsection{Loss Landscape Visualization on Different Dataset}\label{uniform_cifar10}

We also perform the same loss and testing error landscape visualization on CIFAR-10. The training details basically follows Appendix~\ref{exp_setting_fig3}. For CIFAR-10, we use the same data augmentation as in Appendix~\ref{exp_set}. The results are given in Fig.~\ref{cifar10_loss}. From Fig.~\ref{cifar10_loss}, we can observe even more dramatic difference of the loss landscape between standard training and OPT. In standard training, the loss landscape exhibits highly non-convex and non-smooth behavior. There are countless local minima in the loss landscape. Different from the results in Fig.~\ref{hq_loss_cnn9}, the loss landscape of standard training on CIFAR-10 has some huge local minima that are hard to escape from. In contrast, the loss landscape of OPT on CIFAR-10 does not show obvious and huge local minima and is far more convex and smooth than standard training. The contour maps show more significant difference between standard training and OPT. The contour map of OPT shows the shape of a single symmetric and convex valley, while the contour map of standard training presents the shape of multiple highly irregular valleys. The visualization further validate that the improvement of OPT on optimization landscape is very consistent across different training datasets.

\begin{figure}[h]
  \centering
  \renewcommand{\captionlabelfont}{\footnotesize}
  \vspace{1mm}
  \includegraphics[width=5in]{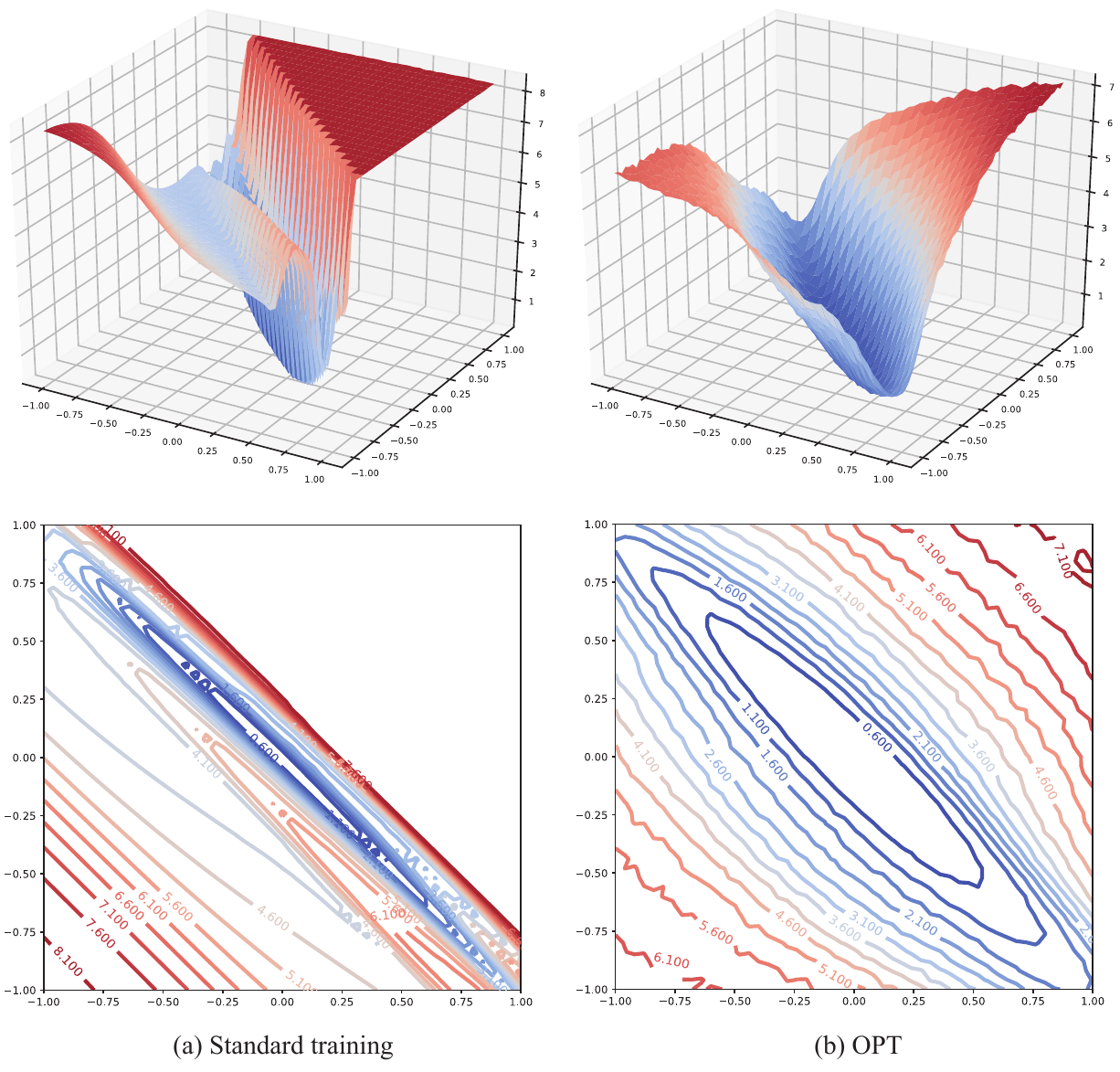}
  \caption{\footnotesize Comparison of loss landscapes between standard training and OPT on CIFAR-10 (CNN-6). Top row: loss landscape visualization with Cartesian coordinate system; Bottom row: loss contour visualization.}\label{cifar10_loss}
  \vspace{2mm}
\end{figure}

Then we also visualize the testing error landscape of standard training and OPT in Fig.~\ref{cifar10_test_err}. As expected, the testing error landscape of standard training shows that changing the pretrained model parameters with a very small perturbation could lead to a dramatic increase in testing error. It indicates that the parameter space of standard training is very sensitive to even a small perturbation. In comparison, the testing error landscape of OPT shows similar shape to the training loss landscape which is the shape of a single regular, smooth and convex valley. We can conclude that OPT has huge advantages over standard training in terms of the optimization landscape. Although the conclusion is drawn from a simple visualization method, it can still shed some light on why OPT yields better training dynamics and generalization ability.

\begin{figure}[h]
  \centering
  \renewcommand{\captionlabelfont}{\footnotesize}
  \includegraphics[width=5in]{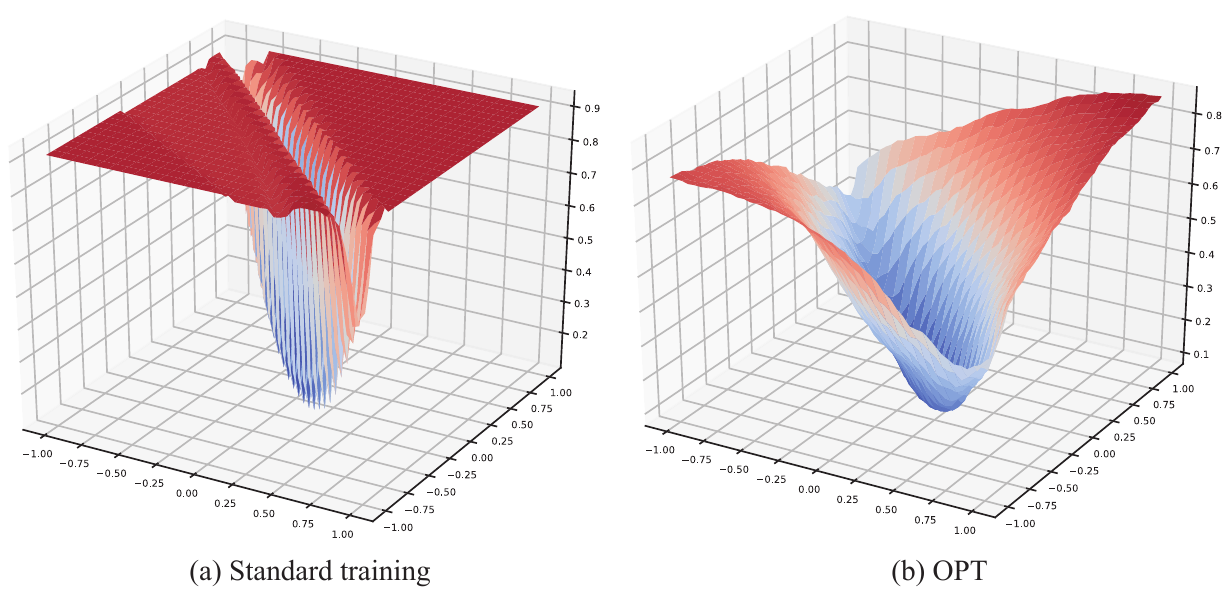}
  \caption{\footnotesize Comparison of testing error landscapes between standard training and OPT on CIFAR-10 (CNN-6).}\label{cifar10_test_err}
\end{figure}

\newpage
\section{Theoretical Discussion on Optimization and Generalization}

The key problem we discuss in this section is why OPT may lead to easier optimization and better generalization. We have already shown that OPT can guarantee the minimum hyperspherical energy~(MHE) in a probabilistic sense. Although empirical evidences~\cite{liu2018learning} have shown significant and consistent performance gain by minimizing hyperspherical energy, why lower hyperspherical energy will lead to better generalization is still unclear. We argue that OPT leads to better generalization from two aspects: how OPT may affect the training and generalization, and why minimum hyperspherical energy serves as a good inductive bias. We note that rigorously proving that OPT generalizes better is out of the scope of this paper and remains our future work. The section serves as a very preliminary discussion for this topic, and hopefully the discussion can inspire more theoretical studies about OPT.

\par
Our goal here is to leverage and apply existing theoretical results~\cite{kawaguchi2016deep,xie2016diverse,soudry2016no,lee2016gradient,du2017gradient,allen2018convergence} to explain the role that MHE plays rather than proving sharp and novel generalization bounds. We emphasize that our paper is \emph{NOT} targeted as a theoretical one that proves novel generalization bounds. 

We simply consider one-hidden-layer networks as the hypothesis class: 
\begin{equation}
    \mathcal{F}=\{f(x)=\sum_{j=1}^n v_j\sigma(\bm{w}_j^\top\bm{x}): v_j\in\{\pm1\},\sum_{j=1}^n\|\bm{w}_j\|\leq C_w\}
\end{equation}
where $\thickmuskip=2mu \medmuskip=2mu \sigma(\cdot)=\max(0,\cdot)$ is ReLU. Since the magnitude of $v_j$ can be scaled into $\bm{w}_j$, we can restrict $v_j$ to be $\pm 1$. Given a set of \emph{i.i.d.} training sample $\{\bm{x}_i,y_i\}_{i=1}^m$ where $\thickmuskip=2mu \medmuskip=2mu x\in\mathbb{R}^d$ is drawn uniformly from the unit hypersphere, we minimize the least square loss $\thickmuskip=2mu \medmuskip=2mu \mathcal{L}=\frac{1}{2m}\sum_{i=1}^m(y_i-f(\bm{x}_i))^2$. The gradient \emph{w.r.t.}  $\bm{w}_i$ is
\begin{equation}
    \frac{\partial \mathcal{L}}{\partial \bm{w}_j}=\frac{1}{m}\sum_{i=1}^m \big(f(\bm{x}_i)-y_i\big)v_j\sigma'(\bm{w}_j^\top\bm{x}_i)\bm{x}_i.
\end{equation}
Let $\thickmuskip=2mu \medmuskip=2mu \bm{W}:=\{\bm{w}_1^\top,\cdots,\bm{w}_n^\top\}^\top$ be the column concatenation of neuron weights. We aim to identify the conditions under which there are no spurious local minima. We rewrite that
\begin{equation}
\frac{\partial \mathcal{L}}{\partial \bm{W}}=\big( \frac{\partial \mathcal{L}}{\partial \bm{w}_1}^\top,\cdots,\frac{\partial \mathcal{L}}{\partial \bm{w}_n}^\top \big)^\top=\bm{D}\bm{r}
\end{equation}
where $\thickmuskip=2mu \medmuskip=2mu\bm{r}\in\mathbb{R}^m$ $\thickmuskip=2mu \medmuskip=2mu \bm{r}_i=\frac{1}{m}f(\bm{x}_i)-y_i$, $\thickmuskip=2mu \medmuskip=2mu\bm{D}\in\mathbb{R}^{n\times m}$, and $\thickmuskip=2mu \medmuskip=2mu\bm{D}_{ij}= v_i\sigma'(\bm{w}_i^\top\bm{x}_j)\bm{x}_j$. Therefore, we can obtain the following inequality: 
\begin{equation}
\norm{\bm{r}}\leq\frac{1}{s_m(\bm{D})}\norm{\frac{\partial\mathcal{L}}{\partial\bm{W}}}
\end{equation}
where $\|\bm{r}\|$ is the training error and $s_m(\bm{D})$ is the minimum singular value of $\bm{D}$. If we need the training error to be small, then we have to lower bound $s_m(\bm{D})$ away from zero. Therefore, the essential problem now becomes the relationship between MHE and the lower bound of $s_m(\bm{D})$. We have the following result from \cite{xie2016diverse}:

\vspace{3mm}
\begin{lemma}\label{smd}
With probability larger than $\thickmuskip=2mu \medmuskip=2mu 
1-m\exp(-m\gamma_m/8)-2m^2\exp(-4\log^2d)-\delta$, we will have that $\thickmuskip=2mu \medmuskip=2mu s_m(\bm{D})^2\geq \frac{1}{2}nm\gamma_m-cn\rho(\bm{W})$ where 
\begin{equation}
\begin{aligned}
    \rho(\bm{W})&\leq\frac{\log d}{\sqrt{d}}\sqrt{2L_2(\bm{W})}m(\frac{4}{m}\log\frac{1}{\delta})^{1/4}\\
    &\ \ \ \ \ \ +\frac{2\log d}{\sqrt{d}}m\sqrt{\frac{4}{3m}\log \frac{1}{\delta}}+\frac{\log d}{\sqrt{d}}mL_2(\bm{W})+2,
\end{aligned}
\end{equation}
and $\thickmuskip=7mu \medmuskip=7mu L_2(\bm{W})=\frac{1}{n^2}\sum_{i,j=1}^n k(\bm{w}_i,\bm{w_j})^2-\mathbb{E}_{\bm{u},\bm{v}}[k(u,v)^2]$. The kernel function $k(\bm{u},\bm{v})$ is $\frac{1}{2}-\frac{1}{2\pi}\arccos(\frac{\langle\bm{u},\bm{v}\rangle}{\|\bm{u}\|\|\bm{v}\|})$.
\end{lemma}
\vspace{2mm}

Once MHE is achieved, the neurons will be uniformly distributed on the unit hypersphere. From Lemma~\ref{smd}, we can see that if the neurons are uniformly distributed on the unit hypersphere, $L_2(\bm{W})$ will be very small and close to zero. Then $\rho(\bm{W})$ will also be small, leading to large lower bound for $s_m(\bm{D})$. Therefore, MHE can result in small training error once the gradient norm $\norm{\frac{\partial\mathcal{L}}{\partial\bm{W}}}$ is small. The result implies no spurious local minima if we use OPT for training.

Furthermore, suppose that $\|\frac{\partial \mathcal{L}}{\partial \bm{W}}\|^2\leq\epsilon$, \cite{xie2016diverse} also proves a training error bound $\tilde{\mathcal{O}}(\epsilon)$ and a generalization bound bound $\tilde{\mathcal{O}}(\epsilon+\frac{1}{\sqrt{m}})$ based on the assumption that $\bm{W}$ belongs to a specific set $\mathcal{G}_{\bm{W}}$ (for the definition of $\mathcal{G}_{\bm{W}}$, please refer to \cite{xie2016diverse}). Therefore, MHE is also connected to the training and generalization error. Note that, the analysis here is highly simplified and the purpose here is to give some justifications rather than rigorously proving any bound.

\par
We further argue that MHE induced by OPT serves as an important inductive bias for neural networks. As the standard regularizer for neural networks,  weight decay controls the norm of the neuron weights, regularizing essentially one dimension of the weight. In contrast, MHE completes an important missing pieces by regularizing the remaining dimensions of the weight. MHE encourages minimum hyperspherical redundancy between neurons. In the linear classifier case, MHE impose a prior of maximal inter-class separability.

\section{More Discussions}\label{more_diss}

\textbf{Semi-randomness}. OPT fixes the randomly initialized neuron weight vectors and simply learns layer-shared orthogonal matrices, so OPT naturally imposes strong randomness to the neurons. OPT well combines the good generalizability from randomness and the strong approximation power from neural networks. Such randomness suggests that the specific configuration of relative position among neurons does not matter that much, and the coordinate system is more crucial for generalization. \cite{kawaguchi2018deep,rahimi2008random,srivastava2014dropout} also show that randomness can be beneficial to generalization.

\textbf{Flexible training}. First, OPT can used in multi-task training~\cite{mallya2018piggyback} where each set of orthogonal matrices represent one task. OPT can learn different set of orthogonal matrices for different tasks with the neuron weights remain the same. Second, we can perform progressive training with OPT. For example, after learning a set of orthogonal matrices on a large coarse-grained dataset (\ie, pretraining), we can multiple the orthogonal matrices back to the neuron weights and construct a new set of neuron weights. Then we can use the new neuron weights as a starting point and apply OPT to train on a small fine-grained dataset (\ie, finetuning).

\textbf{Limitations and open problems} The limitations of OPT include more GPU memory consumption and heavy computation during training, more numerical issues when ensuring orthogonality and weak scalability for ultra wide neural networks. Therefore, there will be plenty of open problems in OPT, such as scalable and efficient training. Most significantly, OPT opens up a new possibility for studying theoretical generalization of deep networks. With the decomposition to hyperspherical energy and coordinate system, OPT provides a new perspective for future research. 

\newpage
\section{On Parameter-Efficient OPT}\label{appendix_peopt}
\subsection{Formulation}
Since OPT over-parameterizes the neurons, it will consume more GPU memory in training (note that, the number of parameters will not increase in testing). For a $d$-dimensional neuron, OPT will learn an orthogonal matrix of size $d\times d$ that applies to the the neuron. Therefore, we will need $d^2$ extra parameters for one layer of neurons, making the training more expensive in terms of the GPU memory. Although the extra training overhead in OPT will not affect the inference speed of the trained neural networks, we still desire to achieve better parameter efficiency in OPT. To this end, we discuss some design possibilities for the \emph{parameter-efficient OPT}~(PE-OPT) in this section.
\par

Original OPT over-parameterize a neuron $\bm{v}\in\mathbb{R}^{n\times n}$ with $\bm{R}\bm{v}$ where $\bm{R}$ is a layer-shared orthogonal matrix of size $d\times d$. We aim to reduce the effective parameters of this $d\times d$ orthogonal matrix. We incorporate a block-diagonal structure to the orthogonal matrix $\bm{R}$. Specifically, we formulate $\bm{R}$ as $\textnormal{Diag}(\bm{R}^{(1)},\bm{R}^{(2)},\cdots,\bm{R}^{(k)})$ where $\bm{R}^{(i)}$ is an orthogonal matrix with size $d_i\times d_i$ (it is easy to see that we need $d=\sum_i d_i$). As an example, we only consider the case where all $\bm{R}^{(i)}$ are of the same size (\ie, $d_1=d_2=\cdots=d_k=\frac{d}{k}$). It is also obvious that as long as each block is an orthogonal matrix, then the overall matrix $\bm{R}$ remains an orthogonal matrix.
\par

\begin{figure}[h]
  \centering
  \renewcommand{\captionlabelfont}{\footnotesize}
  \includegraphics[width=2.8in]{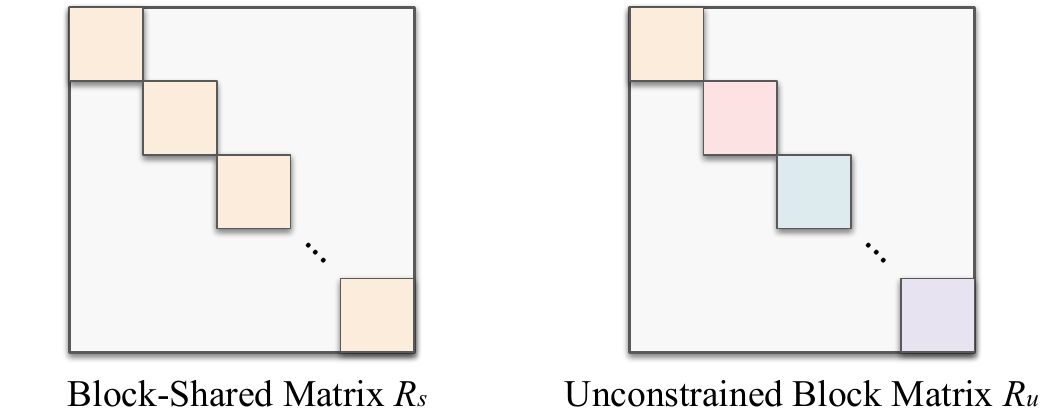}
  \caption{\footnotesize Comparison between the block-shared matrix $\bm{R}_s$ and the unconstrained block matrix $\bm{R}_u$.}\label{peopt}
\end{figure}

First, we consider that all the block matrices on the diagonal of the orthogonal matrix $\bm{R}$ are shared, meaning that $\bm{R}=\textnormal{Diag}(\bm{R}^{(1)},\bm{R}^{(1)},\cdots,\bm{R}^{(1)})$ (\ie, $\bm{R}^{(1)}=\bm{R}^{(2)}=\cdots=\bm{R}^{(k)}$). Therefore, we have a block-diagonal matrix $\bm{R}_{s}$ with shared block $\bm{R}^{(1)}$ as the final orthogonal matrix for the neuron $\bm{v}$:
\begin{equation}
\bm{R}_{s}=\begin{bmatrix}
\bm{R}^{(1)} & 0 & \cdots & 0\\
0 & \bm{R}^{(1)} & \ddots & \vdots\\
\vdots & \ddots & \ddots & 0\\
0 & \cdots & 0 & \bm{R}^{(1)}
\end{bmatrix}
\end{equation}
where $\bm{R}^{(1)}\in\mathbb{R}^{\frac{d}{k}\times\frac{d}{k}}$. The effective number of parameters for the orthogonal matrix $\bm{R}_s$ immediately reduces to $\frac{d^2}{k^2}$. The left figure in Fig.~\ref{peopt} gives an intuitive illustration for the block-shared matrix $\bm{R}_s$. Therefore, PE-OPT only needs to learn $\bm{R}^{(1)}$ in order to construct the orthogonal matrix of size $d\times d$.
\par

Second, we consider that all the diagonal block matrices are independent, indicating that $\bm{R}=\textnormal{Diag}(\bm{R}^{(1)},\bm{R}^{(2)},\cdots,\bm{R}^{(k)})$ where $\bm{R}^{(i)},\forall i$ are different orthogonal matrices in general. We term such matrix $\bm{R}$ as unconstrained block matrix. Therefore, we have the unconstrained block diagonal matrix $\bm{R}_u$ as
\begin{equation}
    \bm{R}_{u}=\begin{bmatrix}
\bm{R}^{(1)} & 0 & \cdots & 0\\
0 & \bm{R}^{(2)} & \ddots & \vdots\\
\vdots & \ddots & \ddots & 0\\
0 & \cdots & 0 & \bm{R}^{(k)}
\end{bmatrix}
\end{equation}
where the orthogonal matrices $\bm{R}^{(i)},\forall i$ will be learned independently. The effective number of parameters for the orthogonal matrix $\bm{R}_u$ is $\frac{d^2}{k}$, making it more flexible than the block-shared matrix $\bm{R}_s$.
\par

Let's consider a convolution neuron (\ie, convolution filter) $\bm{v}\in\mathbb{R}^{c_1\times c_2\times c_3}$ (\eg, a typical convolution neuron is of size $3\times 3 \times 64$) as an example. The orthogonal matrix $\bm{R}$ for the convolution neuron is of size $(c_1c_2c_3)\times (c_1c_2c_3)$. Typically, we will divide the neuron into $k$ sub-neuron along the $c_3$-axis, each with size $c_1\times c_2\times \frac{c_3}{k}$. Then in order to learn a block-shared orthogonal matrix $\bm{R}_s$, we will essentially learn a shared orthogonal matrix of size $(\frac{1}{k}c_1c_2c_3)\times(\frac{1}{k}c_1c_2c_3)$ that applies to each sub-neuron (there are $k$ sub-neurons of size $c_1\times c_2\times \frac{c_3}{k}$ in total). For the case of learning a unconstrained block-diagonal orthogonal matrix $\bm{R}_u$, we simply learn different orthogonal matrices for different sub-neurons.

\subsection{Experiments and Results}
We conduct the image recognition experiments on CIFAR-100 with CNN-6 described in Table~\ref{netarch}. The setting is exactly the same as Section~\ref{CNN_exp}. For the convolution filter, we use the size of $3\times3\times64$, \ie, $c_1=3, c_2=3, c_3=64$. The results are given in Table~\ref{peopt_shared} and Table~\ref{peopt_unshared}. ``\# Parameters'' in both tables denote the number of effective parameters for the orthogonal matrix $\bm{R}$ in a single layer. The baseline with fixed neurons is only to train the final classifiers with the randomly initialized neuron weights staying fixed. It means that this baseline basically removes the learnable orthogonal matrices but still fixes the neuron weights, so it only achieves $73.81\%$ testing error. As expected, as the number of effective parameters goes down, the performance of PE-OPT generally decreases. One can also observe that using separate orthogonal matrices generally yields better performance than shared orthogonal matrices. $k=2$ and $k=4$ seems to be a reasonable trade-off between better accuracy and less parameters.
\par

When $k$ becomes larger (\ie, the number of parameters become less) in the case of block-shared orthogonal matrices, we find that PE-OPT (LS) performs the best among all the variants. When $k$ becomes larger (\ie, the number of parameters become less) in the case of unconstrained block orthogonal matrices, we can see that both PE-OPT (GS) and PE-OPT (LS) performs better than the other variants.

\begin{table*}[h]
    \renewcommand{\captionlabelfont}{\footnotesize}
    \newcommand{\tabincell}[2]{\begin{tabular}{@{}#1@{}}#2\end{tabular}}
    \centering
    \setlength{\abovecaptionskip}{4pt}
    \setlength{\belowcaptionskip}{-5pt}
    \scriptsize
    \begin{tabular}{ccccccc}
        \hline
Method & \# Parameters & PE-OPT (CP) & PE-OPT (GS) & PE-OPT (HR) & PE-OPT (LS) & PE-OPT (OGD)\\\hline
$c_3/k=64$~~~($k=1$)~~(\ie, Original OPT) & 331.7K & 33.53 & \textbf{33.02} & 35.67 & 34.48 & 33.33\\
$c_3/k=32$~~~($k=2$) & 82.9K & 34.93 & \textbf{34.39} & 35.83 & 34.50 & 35.06\\         
$c_3/k=16$~~~($k=4$) & 20.7K & 39.40 & 39.13 & 39.67 & \textbf{37.58} & 39.80\\     
$c_3/k=8$~~~($k=8$) & 5.2K & 47.77 & 46.65 & 46.69 & \textbf{45.62} & 47.43\\
$c_3/k=4$~~~($k=16$) & 1.3K & 56.65 & 55.91 & 55.69 & \textbf{54.75} & 57.15\\
$c_3/k=2$~~~($k=32$) & 0.3K & 63.46 & 62.65 & 62.38 & \textbf{61.60} & 62.46\\
$c_3/k=1$~~~($k=64$) & 0.1K & 67.36 & 67.11 & 67.05 & \textbf{66.61} & 67.23\\\hline
Baseline & - & \multicolumn{5}{c}{37.59}\\
Baseline with fixed random neurons & - & \multicolumn{5}{c}{73.81}\\\hline
    \end{tabular}
    \caption{\footnotesize Testing error (\%) on CIFAR-100 with different settings of PE-OPT (with block-shared orthogonal matrix $\bm{R}_s$).}\label{peopt_shared}
\end{table*}

\begin{table*}[h]
    \renewcommand{\captionlabelfont}{\footnotesize}
    \newcommand{\tabincell}[2]{\begin{tabular}{@{}#1@{}}#2\end{tabular}}
    \centering
    \setlength{\abovecaptionskip}{4pt}
    \setlength{\belowcaptionskip}{-5pt}
    \scriptsize
    \begin{tabular}{ccccccc}
        \hline
Method & \# Parameters & PE-OPT (CP) & PE-OPT (GS) & PE-OPT (HR) & PE-OPT (LS) & PE-OPT (OGD)\\\hline
$c_3/k=64$~~~($k=1$)~~(\ie, Original OPT) & 331.7K & 33.53 & \textbf{33.02} & 35.67 & 34.48 & 33.33\\
$c_3/k=32$~~~($k=2$) & 165.9K & 33.54 & \textbf{33.15} & 35.65 & 34.09 & 34.27\\            
$c_3/k=16$~~~($k=4$) & 82.9K & 34.77 & \textbf{34.50} & 35.71 & 34.96 & 35.97\\     
$c_3/k=8$~~~($k=8$) & 41.5K & 37.25 & 36.43 & 36.40 & \textbf{36.17} & 39.75\\
$c_3/k=4$~~~($k=16$) & 20.7K & 40.74 & \textbf{39.89} & 39.98 & 39.93 & 43.43\\
$c_3/k=2$~~~($k=32$) & 10.4K & 45.36 & 44.77 & 44.83 & \textbf{44.61} & 48.98\\
$c_3/k=1$~~~($k=64$) & 5.2K & 50.94 & \textbf{49.16} & 49.57 & 49.23 & 54.93\\\hline
Baseline & - & \multicolumn{5}{c}{37.59}\\
Baseline with fixed random neurons & - & \multicolumn{5}{c}{73.81}\\\hline
    \end{tabular}
    \caption{\footnotesize Testing error (\%) on CIFAR-100 with different settings of PE-OPT (with unconstrained block orthogonal matrix $\bm{R}_u$).}\label{peopt_unshared}
\end{table*}

\newpage
\section{On Generalizing OPT: Over-Parameterized Training with Constraint}

OPT opens many new possibilities in training neural networks. We consider a simple generalization to OPT in this section to showcase the great potential of OPT. Instead of constraining the over-parameterization matrix $\bm{R}\in\mathbb{R}^{d\times d}$ in Eq.~\ref{opt} to be orthogonal, we can use any meaningful structural constraints for this matrix, and even regularize it in a task-driven way. Furthermore, instead of a linear over-parameterization (\ie, multiplying a matrix $\bm{R}$) to the neuron, we can also consider nonlinear mapping. We come up with the following straightforward generalization to OPT (the settings and notations exactly follow Eq.~\ref{opt}):
\begin{equation}\label{generalized_opt}
\begin{aligned}
\textnormal{Standard:}\ &\min_{\bm{v}_i,u_i,\forall i} \sum_{j=1}^{m}\mathcal{L}\big(y,\sum_{i=1}^{n} u_i\bm{v}_i^\top\bm{x}_j\big)\\
    \textnormal{Original OPT:}\ &\min_{\bm{R},u_i,\forall i} \sum_{j=1}^{m}\mathcal{L}\big(y,\sum_{i=1}^{n} u_i(\bm{R}\bm{v}_i)^\top\bm{x}_j\big)\\
    &\ \ \ \ \ \ \textnormal{s.t.}\ \ \bm{R}^\top\bm{R}=\bm{R}\bm{R}^\top=\bm{I}\\
    \textnormal{Generalized OPT:}\ &\min_{\bm{R},u_i,\forall i} \sum_{j=1}^{m}\mathcal{L}\big(y,\sum_{i=1}^{n} u_i(\mathcal{T}(\bm{v}_i))^\top\bm{x}_j\big)\\
    &\ \ \ \ \ \ \textnormal{s.t.}\ \ \textnormal{Some constraints on }\mathcal{T}(\cdot)
\end{aligned}
\end{equation}
where $\mathcal{T}(\cdot):\mathbb{R}^d\rightarrow\mathbb{R}^{d}$ denotes some transformation (including both linear and nonlinear). Notice that the generalized OPT~(G-OPT) no longer requires orthogonality. Such formulation of G-OPT can immediately inspire a number of 
instances. We will discuss some obvious ones here.
\par

If we consider $\mathcal{T}(\cdot)$ to be a linear mapping, we may constrain $\bm{R}$ to be symmetric positive definite other than orthogonal. A simple way to achieve that is to use Cholesky factorization $\bm{L}\bm{L}^\top$ where $\bm{L}$ is a lower triangular matrix to parameterize the matrix $\bm{R}$. Essentially, we learn a lower triangular matrix $\bm{L}$ and use $\bm{L}\bm{L}^\top$ to replace $\bm{R}$ in OPT. The positive definiteness provides the transformation $\bm{R}$ with some geometric constraint. Specifically, a positive definite $\bm{R}$ only transforms the neuron weight $\bm{v}$ to the direction that has the angle less than $\frac{\pi}{2}$ to $\bm{v}$, because $\bm{v}^\top\bm{R}\bm{v}>0$. Moreover, we can also require the transformation to have structural constraints on $\bm{R}$. For example, $\bm{R}$ can be upper (lower) triangular, banded, symmetric, skew-symmetric, upper (lower) Hessenberg, etc.
\par

We can also consider $\mathcal{T}(\cdot)$ to be a nonlinear mapping. A obvious example is to use a neural network (\eg, MLP, CNN) as $\mathcal{T}(\cdot)$. Then the nonlinear G-OPT will share some similarities with HyperNetworks~\cite{ha2016hypernetworks} and Network-in-Network~\cite{lin2013network}. If we further consider $\mathcal{T}(\cdot)$ to be dependent on the input, then the nonlinear G-OPT will have close connections to dynamic neural networks~\cite{jia2016dynamic,liu2019neural}.
\par

To summarize, OPT provides a novel and effective framework to train neural networks and may inspire many different threads of future research.

\newpage
\section{Hyperspherical Energy Training Dynamics of Individual Layers}\label{appendix_energy}
We also plot the hyperspherical energy ($\thickmuskip=2mu \medmuskip=2mu \bm{E}(\hat{\bm{v}}_i|_{i=1}^n) = \sum_{i=1}^{n}\sum_{j=1,j\neq i}^{n} \norm{\hat{\bm{v}}_i-\hat{\bm{v}}_j}^{-1}$ in which $\thickmuskip=2mu \medmuskip=2mu \hat{\bm{v}}_i=\frac{\bm{v}_i}{\|\bm{v}_i\|}$ is the $i$-th neuron weight projected onto the unit hypersphere.) in every layer of CNN-6 during training to show how these hyperspherical energies are being minimized. From Fig.~\ref{energy_dynamics_supp}, we can observe that OPT can always maintain the minimum hyperspherical energy during the entire training process, while the MHE regularization cannot. Moreover, the hyperspherical energy of the baseline will also decrease as the training proceeds, but it is still much higher than the OPT training.
\begin{figure}[h]
  \centering
  \renewcommand{\captionlabelfont}{\footnotesize}
  \vspace{-0.9mm}
  \includegraphics[width=5.55in]{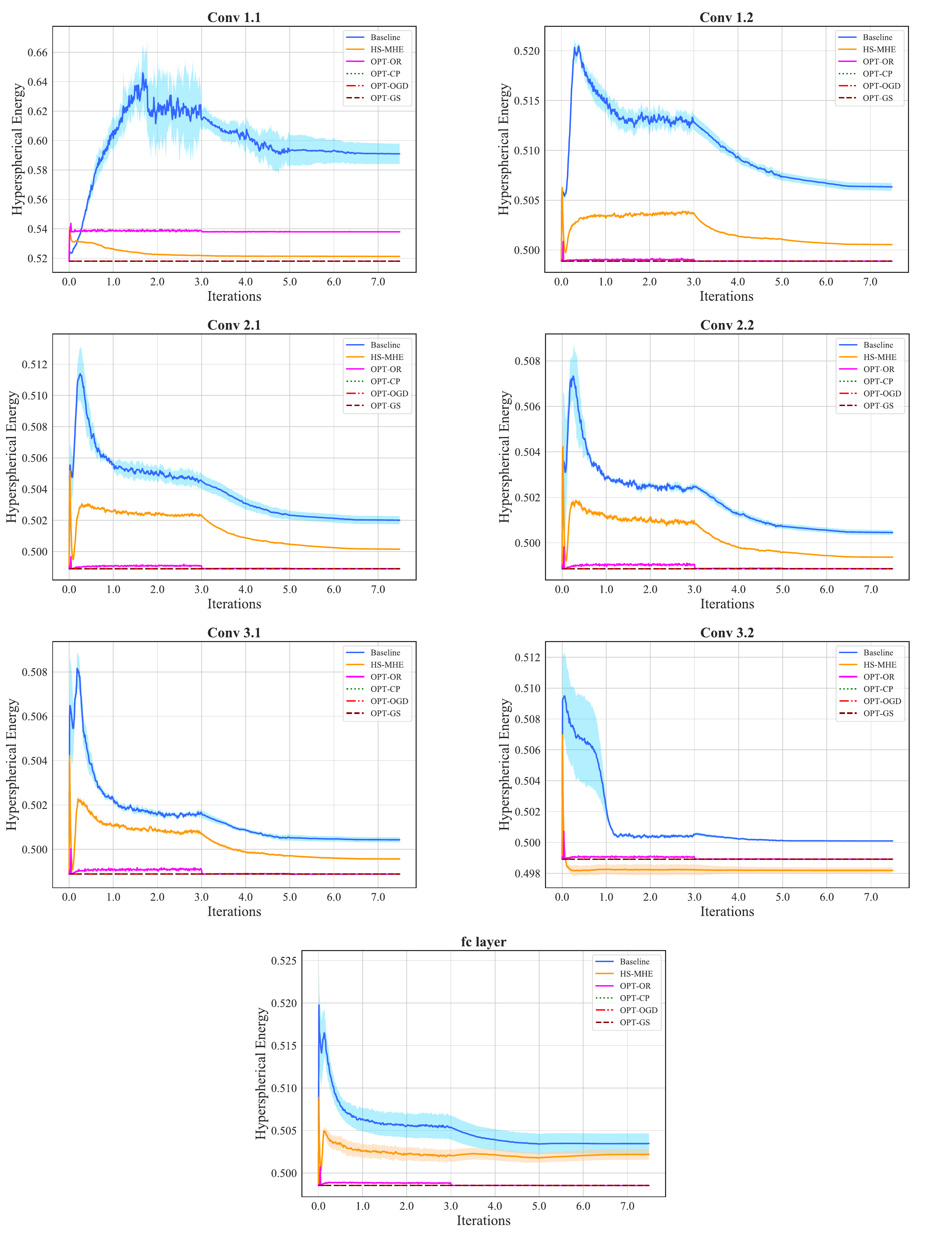}
  \vspace{-3mm}
  \caption{\footnotesize Training dynamics of hyperspherical energy in each layer of CNN-6. We average results with 10 runs.}\label{energy_dynamics_supp}
  \vspace{-1.4mm}
\end{figure}

\newpage

\section{Geometric Properties of Randomly Initialized Neurons}\label{rand_property}
There are many interesting geometric properties~\cite{brauchart2018random,breger2018points} of random points distributed independently and uniformly on the unit hypersphere. We summarize a few of them that make randomly initialized neurons distinct from any deterministic neuron configuration. Note that, there exist many deterministic neuron configurations that can also achieve very low hyperspherical energy, and this section aims to describe a few unique geometric properties of randomly initialized neurons.
\par

There are two widely used geometric properties corresponding to a neuron configuration (\ie, a set of neurons) $\hat{\bm{W}}_N=\{\hat{\bm{w}}_1,\cdots,\hat{\bm{w}}_N\in\mathbb{S}^{d}\}$. In the main paper, we define neurons on $\mathbb{S}^{d-1}$, but without loss of generality we define neurons on $\mathbb{S}^d$ here for convenience. The first one is the \emph{covering radius}:
\begin{equation}
    \alpha(\hat{\bm{W}}_N):=\alpha(\hat{\bm{W}}_N;\mathbb{S}^{d-1}):=\max_{\bm{u}\in\mathbb{S}^d}\min_{1\leq i \leq N}\arccos(\bm{u},\hat{\bm{w}}_i)
\end{equation}
which is the biggest geodesic distance from a neuron in $\mathbb{S}^{d}$ to the nearest point in $\hat{\bm{W}}_N$. The second one is the \emph{separation distance}:
\begin{equation}
    \psi(\hat{\bm{W}}_N):=\min_{1\leq i,j,\leq N, i\neq j}\arccos(\hat{\bm{w}}_i,\hat{\bm{w}}_j)
\end{equation}
which gives the least geodesic distance between arbitrary two points in $\hat{\bm{W}}_N$. Random points (\ie, randomly initialized neurons) typically have poor separation properties, since the separation is sensitive to the specific placement of points. \cite{brauchart2018random} shows an example on $\mathbb{S}^1$ to illustrate this observation.

\par
\cite{brauchart2018random} considers a different but related quantity, \ie, the sume of powers of the ``hole radii''. A set of neurons $\hat{\bm{W}}_N$ on $\mathbb{S}^{d}$ uniquely defines a convex polytope, which can be viewed as the convex hull of the neuron configuration. Each facet of the polytope defines a ``hole''. Such a hole denotes the maximal spherical cap for a facet that contains neurons of $\hat{\bm{W}}_N$ only on the boundary. It is easy to see that the geodesic radius of the largest hole is the covering radius $\alpha(\hat{\bm{W}}_N)$. We assume that for the set of neurons $\hat{\bm{W}}_N$, there are $f_d$ holes (\ie, facets) in total. Therefore, the $i$-th hole radius is defined as $\rho_i =\rho_i(\hat{\bm{W}}_N)$ which is the Euclidean distance in $\mathbb{R}^{d+1}$ between the center of the $i$-th spherical cap and the boundary. The $i$-th spherical cap is located on the sphere corresponding to the $i$-th facet. We have that $\rho_i=2\sin(\frac{\alpha_i}{2})$ where $\alpha_i$ is the geodesic radius of the $i$-th spherical cap. We are interested in the sums of the $p$-th powers of the hole radii, \ie,
\begin{equation}
    \mathcal{P}=\sum_{i=1}^{f_d} (\rho_i)^p
\end{equation}
where $p$ is larger than zero. For large $p$, the largest hole dominates:
\begin{equation}
    \lim_{p\rightarrow\infty}(\mathcal{P})^{\frac{1}{p}}=\lim_{p\rightarrow\infty}\bigg( \sum_{i=1}^{f_d}(\rho_i)^p \bigg)^{\frac{1}{p}}=\max_{1\leq i\leq f_d} \rho_i=2\sin(\frac{\alpha(\hat{\bm{W}}_N)}{2})
\end{equation}
where $\rho(\hat{\bm{W}}_N):=\max_{1\leq i\leq f_d}\rho_i$. Then we introduce some useful notations to state the geometric properties. Let $\psi_d$ be the surface area of $\mathbb{S}^d$, and we have that
\begin{equation}
    \psi_d=\frac{2\pi^{\frac{d+1}{2}}}{\Gamma(\frac{d+1}{2})},
\end{equation}
and we also define the following quantities (with $\psi_0=2$):
\begin{equation}
\begin{aligned}
    \kappa_d:&=\frac{1}{d}\frac{\psi_{d-1}}{\psi_d}=\frac{1}{d}\frac{\Gamma(\frac{d+1}{2})}{\sqrt{\pi}\Gamma(\frac{d}{2})}\\
    B_d:&=\frac{2}{d+1}\frac{\kappa_{d^2}}{(\kappa_d)^d}
\end{aligned}
\end{equation}
where $\kappa_d$ can be alternatively defined with the recursion: $\kappa_1=\frac{1}{\pi}$ and $\kappa_d=\frac{1}{2\pi d\kappa_{d-1}}$. \cite{reitzner2004stochastical} gives the expected number of facets constructed from $N$ random neurons that are independently and uniformly distributed on the unit hypersphere $\mathbb{S}^d$:
\begin{equation}
    \mathbb{E}[f_d]=B_dN\big( 1+o(1) \big)
\end{equation}
where $N\rightarrow\infty$. Then we introduce the main results of \cite{brauchart2018random} (asymptotics for the expected moments of the hole radii) in the following lemma:

\vspace{2mm}
\begin{lemma}
If $p\geq 0$ and $\hat{\bm{w}}_1,\cdots\hat{\bm{w}}_N$ are $N$ neurons on $\mathbb{S}^d$ that are independently and randomly distributed with respect to the normalized surface area measure $\sigma_d$ on $\mathbb{S}^d$, then we have that
\begin{equation}
\begin{aligned}
    \mathbb{E}[ \mathcal{P} ] &=B_d(\kappa_d)^{-\frac{p}{d}}\frac{\Gamma(d+\frac{p}{d})\Gamma(N+1)}{\Gamma(d)\Gamma(N+\frac{p}{d})}\big( 1+\mathcal{O}(N^{-\frac{2}{d}}) \big)\\
    &=c_{d,p} N^{1-\frac{p}{d}}\big( 1+\mathcal{O}(N^{-\frac{2}{d}}) \big)
\end{aligned}
\end{equation}
as $N\rightarrow\infty$, where $\rho_i=\rho_{i,N}$ is the Euclidean hole radius associated with the $i$-th facet of the convex hull of $\hat{\bm{W}}_N$, $c_{d,p}:=B_dB_{d,p}$, and $B_{d,p}:=\frac{\Gamma(d+\frac{p}{d})}{\Gamma(d)}(\kappa_d)^{-\frac{p}{d}}$. The $\mathcal{O}$-terms above depend on $d$ and $p$.
\end{lemma}
\vspace{2mm}

As we mentioned, there are many deterministic point (\ie, neuron) configurations such as minimizing hyperspherical energy (\ie, Riesz $s$-energy)~\cite{liu2018learning} (as $s\rightarrow\infty$, the minimal $s$-energy points approach the best separation), maximizing the determinant for polynomial interpolation~\cite{sloan2004extremal}, Fibonacci points, spherical $t$-designs, minimizing covering radius (\ie, best covering problem), maximizing the separation (\ie, best packing problem) and maximizing the $s$-polarization, etc. We note that randomly initialized neurons are quite different from these deterministic neuron configurations and have unique geometric properties.

\end{appendix}

\end{document}